\documentclass{article} 
\usepackage{iclr2024_conference,times}

\usepackage{amsmath,amsfonts,bm}

\def\eqref#1{equation~\ref{#1}}

\def\1{\bm{1}}

\DeclareMathAlphabet{\mathsfit}{\encodingdefault}{\sfdefault}{m}{sl}
\SetMathAlphabet{\mathsfit}{bold}{\encodingdefault}{\sfdefault}{bx}{n}

\DeclareMathOperator*{\argmax}{arg\,max}
\DeclareMathOperator*{\argmin}{arg\,min}

\usepackage{url}

\usepackage{graphicx}

\usepackage{wrapfig}
\usepackage[breakable, skins]{tcolorbox}
\definecolor{greyC}{RGB}{180,180,180}
\definecolor{greyL}{RGB}{235,235,235}
\usepackage{arydshln}
\usepackage{amsmath}
\usepackage{amsmath,amssymb,amsfonts}
\usepackage{amsthm}
\usepackage{mathrsfs}

\usepackage{amssymb}
\usepackage{enumitem}
\usepackage{subcaption}	
\usepackage{booktabs}
\usepackage{multirow}
\usepackage{xspace}
\usepackage{color, colortbl}
\definecolor{Gray}{gray}{0.9}

\definecolor{mydarkred}{rgb}{0.6,0,0}
\definecolor{myblue}{HTML}{268BD2}
\definecolor{mygreen}{HTML}{658354}
\definecolor{orangeinplot}{HTML}{e29c7a}
\definecolor{purpleinplot}{HTML}{7676a4}
\definecolor{greeninplot}{HTML}{288308}
\usepackage[colorlinks,
linkcolor=mydarkred,
citecolor=myblue]{hyperref}

\def\*#1{\mathbf{#1}}

\usepackage{algorithm}
\usepackage[ruled,vlined,algo2e]{algorithm2e}
\newcommand{\model}{\textsc{SAL}\xspace}
\usepackage{dsfont}

\newtheorem{theorem}{\textbf{Theorem}}
\newtheorem{theorem*}{Theorem}

\newtheorem{lemma}{Lemma}

\newtheorem{Definition}{Definition}
\newtheorem{Proposition}{Proposition}

\newtheorem{assumption}{Assumption}
\newtheorem{Remark}{Remark}

\title{How Does Unlabeled Data Provably Help \\Out-of-Distribution Detection?}

\iclrfinalcopy

\author{Xuefeng Du$^{1,}$\thanks{Equal contributions} , Zhen Fang$^{2,*}$, Ilias Diakonikolas$^{1}$, Yixuan Li$^{1}$ \\
$^{1}$Department of Computer Sciences, University of Wisconsin-Madison\\
$^{2}$Australian Artificial Intelligence Institute, University of Technology Sydney\\
\texttt{\{xfdu,ilias,sharonli\}@cs.wisc.edu, zhen.fang@uts.edu.au} \\
}

\begin{document}

\maketitle
\begin{abstract}
 Using unlabeled data to regularize the machine learning models has demonstrated promise for improving safety and reliability in detecting out-of-distribution (OOD) data. Harnessing the power of unlabeled \textcolor{black}{in-the-wild} data is non-trivial due to the heterogeneity of both in-distribution (ID) and OOD data. This lack of a clean set of OOD samples poses significant challenges in learning an optimal OOD classifier. Currently, there is a lack of research on formally understanding how unlabeled data helps OOD detection. 
 This paper bridges the gap by introducing a new learning framework \model (\textbf{S}eparate
\textbf{A}nd \textbf{L}earn) that offers both strong theoretical guarantees and empirical effectiveness. The framework separates candidate outliers from the unlabeled data and then trains an OOD classifier using the candidate outliers and the labeled ID data. Theoretically, we provide rigorous error bounds from the lens of separability and
learnability, formally justifying the two components in our algorithm.  Our theory shows that \model can separate the candidate outliers with small error rates, which leads to a generalization guarantee for the learned OOD classifier.  Empirically, \model achieves state-of-the-art performance on common benchmarks, reinforcing our theoretical insights. Code is publicly available at
\url{https://github.com/deeplearning-wisc/sal}.
\end{abstract}

\section{Introduction}

When deploying machine learning models in real-world environments, their safety and reliability are
often challenged by the occurrence of out-of-distribution (OOD) data, \textcolor{black}{which arise from unknown categories and should not be predicted by the model}. Concerningly, neural networks are brittle and lack the necessary awareness of OOD data in the wild~\citep{nguyen2015deep}. 
Identifying OOD inputs is a vital but fundamentally challenging problem---the models are not explicitly exposed to the unknown distribution during training, and therefore cannot capture a reliable boundary between in-distribution (ID) vs. OOD data.
To circumvent the challenge, researchers have started to explore training with additional data, which can facilitate a conservative and safe decision boundary against OOD data. In particular, a recent work \textcolor{black}{by}~\cite{katzsamuels2022training} proposed to leverage unlabeled data in the wild to regularize model training, while learning to classify labeled ID data. Such unlabeled \textcolor{black}{wild} data offer the benefits of being freely collectible upon deploying
any machine learning model in its operating environment, and allow capturing the true test-time OOD distribution.

Despite the promise, harnessing the power of unlabeled \textcolor{black}{wild} data is non-trivial due to the heterogeneous mixture of ID and OOD data. This lack of a clean set of OOD training data poses significant challenges in designing effective OOD learning algorithms. Formally, the unlabeled data can be characterized by a Huber contamination model $\mathbb{P}_\text{wild}:= (1-\pi) \mathbb{P}_\text{in} + \pi \mathbb{P}_\text{out}$, where  $\mathbb{P}_\text{in}$ and $\mathbb{P}_\text{out}$ are the marginal distributions of the ID and OOD data. 
It is important to note that the learner only observes samples drawn from such mixture distributions, without knowing the clear membership of whether being ID or OOD. \textcolor{black}{Currently, a formalized understanding of the
problem is lacking for the field}.
This prompts the question underlying the present work: 
\begin{center}
\begin{tcolorbox}[enhanced,attach boxed title to top center={yshift=-3mm,yshifttext=-1mm},colback=gray!5!white,colframe=gray!75!black,colbacktitle=red!80!black,
  title=,fonttitle=\bfseries,
  boxed title style={size=small,colframe=red!50!black},width=3.8in, boxsep=1pt ]
\emph{How does unlabeled \textcolor{black}{wild} data provably help OOD detection?}
\end{tcolorbox}
\end{center}

\textbf{Algorithmic contribution.} In this paper, we propose a new learning framework \model (\textbf{S}eparate \textbf{A}nd \textbf{L}earn), that effectively exploits the unlabeled wild data for OOD detection. At a high level, our framework \model 
builds on two consecutive components: \textbf{(1)}  \textcolor{black}{filtering}---separate \emph{candidate outliers} from the unlabeled data, and \textbf{(2)} \textcolor{black}{classification}---learn an OOD classifier with the candidate outliers, in conjunction with the labeled ID data. To separate the candidate outliers,
our key idea is to
perform singular value decomposition on a gradient matrix, defined over all the unlabeled data \textcolor{black}{whose gradients are computed based on a classification model trained on the clean labeled ID data.} 
\textcolor{black}{In the \model framework, unlabeled wild data are considered candidate outliers when their projection onto the top singular vector exceeds a given threshold. The filtering strategy for identifying candidate outliers is theoretically supported by Theorem \ref{MainT-1}. {We show in Section~\ref{sec:method} (Remark \ref{R1}) that under proper conditions, with a high probability, there exist some specific directions (e.g., the top singular vector direction) where the mean magnitude of the gradients for the wild outlier data is \textcolor{black}{larger than that of ID data}}.} After obtaining the outliers from the wild data, we train an OOD classifier that optimizes the classification between the ID vs. candidate outlier data for OOD detection. 

\textbf{Theoretical significance.} Importantly, we provide new theories from the lens of \emph{separability} and \emph{learnability}, formally justifying the two components in our algorithm. 
Our main Theorem~\ref{MainT-1} analyzes the separability of outliers from unlabeled \textcolor{black}{wild} data using our filtering procedure, and gives a rigorous bound on the error rate. Our theory has practical implications. For example, when the size of the labeled ID data and unlabeled data is sufficiently large, Theorems~\ref{MainT-1} and \ref{The-1.1} imply that the error rates of filtering outliers can be bounded by a small bias proportional to the optimal ID risk, which is a small value close to zero in reality~\citep{frei2022benign}. Based on the error rate estimation, we give a generalization error of the OOD classifier in  Theorem \ref{the:main2}, to quantify its learnability on the  ID data and a noisy set of candidate outliers. Under proper conditions, the generalization error of the learned OOD classifier is \textcolor{black}{upper bounded by the risk  associated with the optimal OOD classifier.}

\textbf{Empirical validation.} Empirically, we show that the generalization bound w.r.t. \model (Theorem~\ref{the:main2}) indeed translates into strong empirical performance. \model can be broadly applicable to non-convex models such as modern neural networks.
We extensively evaluate \model on common OOD detection tasks and establish state-of-the-art performance. For
completeness, we compare \model with two families of methods:
(1) trained with only $\mathbb{P}_\text{in}$, and (2) trained with both
$\mathbb{P}_\text{in}$ and an unlabeled dataset. On \textsc{Cifar-100}, compared to
a strong baseline KNN+~\citep{sun2022out} using only $\mathbb{P}_\text{in}$, \model outperforms
by 44.52\% (FPR95) on average.
While methods such as Outlier Exposure~\citep{hendrycks2018deep} require a clean set of auxiliary unlabeled data, our results are achieved without imposing any such assumption on the unlabeled data and hence offer stronger flexibility. 
Compared to the most related baseline WOODS~\citep{katzsamuels2022training}, our framework can reduce the FPR95 from
7.80\% to 1.88\% on \textsc{Cifar-100}, establishing  near-perfect results on this challenging benchmark. 

\section{Problem Setup}

Formally, we describe the data setup, models and losses and learning goal.

\textbf{Labeled ID data and ID distribution.} \textcolor{black}{Let $\mathcal{X}$ be the input space, and $\mathcal{Y}=\{1,...,K\}$ be the label space for ID data. Given an unknown ID joint distribution $\mathbb{P}_{\mathcal{X}\mathcal{Y}}$ defined over $\mathcal{X}\times \mathcal{Y}$, the labeled ID data $\mathcal{S}^{\text{in}} = \{(\*x_1, y_1),...,(\*x_n, y_n)\}$ are drawn independently and identically from $\mathbb{P}_{\mathcal{X}\mathcal{Y}}$.  We also denote $\mathbb{P}_{\text{in}}$ as the marginal distribution of $\mathbb{P}_{\mathcal{X}\mathcal{Y}}$ on $\mathcal{X}$, which is referred to as the ID distribution. }

\textbf{Out-of-distribution detection.} Our framework concerns a common real-world scenario in which the algorithm is trained on the \textcolor{black}{labeled} ID data, but will then be deployed in environments containing OOD data from {unknown} class, i.e., $y\notin \mathcal{Y}$, and therefore should not be predicted by the model. 
At test time, the goal is to decide whether a test-time input is from  ID  or not (OOD).

\textbf{Unlabeled wild data.} A key challenge in OOD detection is the lack of labeled OOD data. In particular, the sample space for potential OOD data can be prohibitively large, making it expensive to collect labeled OOD data.  In this paper,  to model the realistic environment, we incorporate unlabeled wild
data $\mathcal{S}_{\text{wild}}=\left\{\tilde{\mathbf{x}}_1,...,\tilde{\mathbf{x}}_m \right\}$ into our learning framework. Wild data
 consists of both ID and OOD data, and can be collected   freely upon deploying an existing model trained on $\mathcal{S}^{\text{in}}$. Following~\cite{katzsamuels2022training}, we use the Huber contamination model to characterize the marginal distribution
of the wild data
\begin{equation}
\mathbb{P}_\text{wild} := (1-\pi) \mathbb{P}_\text{in} + \pi \mathbb{P}_\text{out}, 
\label{eq:wild_w_cor}
\end{equation}
where $\pi \in (0,1]$ and \textcolor{black}{$\mathbb{P}_\text{out}$ is the OOD distribution defined over $\mathcal{X}$}.  Note that the case $\pi=0$  is
straightforward since no novelties occur.

\textbf{Models and losses.} We denote by $\*h_{\*w}: \mathcal{X} \mapsto \mathbb{R}^K$  a predictor for ID classification 
 with parameter $\*w \in \mathcal{W}$, where $\mathcal{W}$ is the parameter space. $\*h_{\*w}$ returns the soft classification output. We consider the loss function $\ell: \mathbb{R}^{K} \times \mathcal{Y}\mapsto \mathbb{R}$ on the labeled ID data. In addition, we denote the OOD classifier $\*g_{\boldsymbol{\theta}}:  \mathcal{X} \mapsto \mathbb{R}$ with parameter ${\boldsymbol{\theta}} \in \Theta$, where $\Theta$ is the parameter space. We use $\ell_{\text{b}}(\*g_{\boldsymbol{\theta}}(\mathbf{x}),y_{\text{b}})$ to denote the binary loss function \emph{w.r.t.} $\*g_{\boldsymbol{\theta}}$ and binary label $y_{\text{b}}\in \mathcal{Y}_{\text{b}}:=\{y_{+},y_{-}\}$, where $y_+ \in \mathbb{R}_{>0}$ and $y_-\in \mathbb{R}_{<0}$ correspond to the ID class and the OOD class, respectively.

\textbf{Learning goal.} Our learning framework aims to build the OOD classifier $\*g_{\boldsymbol{\theta}}$ by leveraging data from both $\mathcal{S}^{\text{in}}$ and $\mathcal{S}_{\text{wild}}$. In evaluating our
model, we are interested in the following measurements:
\begin{equation}
\begin{aligned}
&(1)~\downarrow \text{FPR}(\*g_{\boldsymbol{\theta}};\lambda):=\mathbb{E}_{\*x \sim \mathbb{P}_\text{out}}(\mathds{1}{\{ \*g_{\boldsymbol{\theta}} (\*x)>\lambda\}}), \\
 &(2)~\uparrow \text{TPR}(\*g_{\boldsymbol{\theta}};\lambda):=\mathbb{E}_{\*x \sim \mathbb{P}_\text{in}}(\mathds{1}{\{ \*g_{\boldsymbol{\theta}}(\*x)>\lambda\}}),
\end{aligned}
\end{equation}
where $\lambda$ is a threshold, typically chosen so that a high fraction of ID data is correctly classified.

\section{Proposed Methodology}
\label{sec:method}
In this section, we introduce a new learning framework \model that performs
OOD detection \textcolor{black}{by leveraging} the unlabeled wild data. The framework offers substantial advantages over the counterpart approaches that rely only on the ID data, and naturally suits many applications where machine learning models are deployed in the open world.
\model has two integral components: \textbf{(1)} \textcolor{black}{filtering}---separate the candidate outlier data from the unlabeled wild data  (Section~\ref{sec:detect}),  and \textbf{(2)} \textcolor{black}{classification}---train a binary OOD classifier  with the ID data and candidate outliers  (Section~\ref{sec:training}).  In Section~\ref{sec:theory}, we provide theoretical guarantees for \model, provably justifying the two components in our method.

\subsection{Separating Candidate Outliers from the Wild Data}
\label{sec:detect}
\textcolor{black}{To separate candidate outliers from the wild mixture $\mathcal{S}_{\text{wild}}$, our framework employs a level-set estimation based on the gradient information. The gradients are estimated from a classification predictor $\*h_\*w$ trained on the ID data $\mathcal{S}^{\text{in}}$.  We describe the procedure  formally below. }

\textbf{Estimating the reference gradient from ID data.} To begin with, \model estimates the reference gradients by training a classifier  $\*h_\*w$ on the ID data $\mathcal{S}^{\text{in}}$ by empirical risk minimization (ERM):
\begin{equation}
\*w_{\mathcal{S}^{\text{in}}} \in \argmin_{\*w\in \mathcal{W}} R_{\mathcal{S}^{\text{in}}}(\*h_\*w),~~\text{where}~~R_{\mathcal{S}^{\text{in}}} (\*h_\*w)= \frac{1}{n}\sum_{(\*x_i, y_i) \in \mathcal{S}^{\text{in}}} \ell (\*h_\*w(\*x_i),y_i),
    \label{eq:erm}
\end{equation}
$\*w_{\mathcal{S}^{\text{in}}}$ is the learned parameter and $n$ is the size of ID training set $\mathcal{S}^{\text{in}}$. {The average gradient $\bar{\nabla}$ is}
\begin{equation}
    \bar{\nabla}=\frac{1}{n} \sum_{(\*x_i, y_i) \in \mathcal{S}^{\text{in}}} \nabla \ell (\*h_{\*w_{\mathcal{S}^{\text{in}}}}(\*x_i),y_i),
    \label{eq:reference_gradient}
\end{equation}
 where $\bar{\nabla} $  acts as a reference gradient that allows measuring the deviation of any other points from it.

\textcolor{black}{\textbf{Separate candidate outliers from the unlabeled wild  data.} After training the classification predictor on the labeled ID data, we deploy the trained predictor $\*h_{\*w_{\mathcal{S}^{\text{in}}}}$ in the wild, and naturally receives data $\mathcal{S}_{\text{wild}}$---a mixture of unlabeled ID and OOD data. Key to our framework, we perform a filtering procedure on the wild data $\mathcal{S}_{\text{wild}}$, identifying candidate outliers based on a filtering score. To define the filtering score, we represent each point in $\mathcal{S}_{\text{wild}}$ as a gradient vector, relative to the reference gradient $\bar{\nabla} $.}
\textcolor{black}{
Specifically, we calculate the gradient matrix (after subtracting the reference gradient $\bar{\nabla}$) for the wild data as follows:}
\begin{equation}
    \mathbf{G} = \left[\begin{array}{c}
    \nabla \ell (\*h_{\*w_{\mathcal{S}^{\text{in}}}}(\tilde{\*x}_1), \widehat{y}_{\tilde{\*x}_1})-\bar{\nabla} \\
    ...\\
    \nabla \ell (\*h_{\*w_{\mathcal{S}^{\text{in}}}}(\tilde{\*x}_m), \widehat{y}_{\tilde{\*x}_m})-\bar{\nabla}
    \end{array}\right]^\top,
    \label{eq:gradient_matrix}
\end{equation}
\textcolor{black}{where $m$ denotes the size of the wild data, \textcolor{black}{and $\widehat{y}_{\tilde{\*x}}$ is the predicted label for a wild sample $\tilde{\*x}$}. }\textcolor{black}{For each data  point $\tilde{\*x}_i$ in $\mathcal{S}_{\text{wild}}$, we then define our filtering score  as follows:}
\begin{equation}
    \tau_i =\left<\nabla \ell (\*h_{\*w_{\mathcal{S}^{\text{in}}}}\big(\tilde{\*x}_i), \widehat{y}_{\tilde{\*x}_i})-\bar{\nabla} , \*v\right > ^2,
    \label{eq:score}
\end{equation}
\textcolor{black}{where $\left<\cdot, \cdot\right>$ is the dot product operator and $\*v$ is the top singular vector of  $\mathbf{G}$. The top singular vector $\*v$  can be regarded as the principal component of the matrix $\mathbf{G}$ in Eq.~\ref{eq:gradient_matrix}, which maximizes the total distance from the projected gradients (onto the direction of $\*v$) to the origin (sum over all points in $\mathcal{S}_{\text{wild}}$)~\citep{hotelling1933analysis}. Specifically, $\*v$ is a unit-norm vector and can be computed as follows:}
\vspace{-0.1em}
\begin{equation}
 \vspace{-0.1em}
    \*v \in  \argmax_{\|\*u\|_2 = 1} \sum_{\tilde{\*x}_i \in \mathcal{S}_{\text{wild}}} \left< \*u, \nabla \ell (\*h_{\*w_{\mathcal{S}^{\text{in}}}}\big(\tilde{\*x}_i),   \widehat{y}_{\tilde{\*x}_i})-\bar{\nabla} \right>^2.
    \label{eq:pca}
\end{equation}
Essentially, the filtering score $\tau_i$ in Eq.~\ref{eq:score} measures the $\ell_2$ norm of the projected vector. \textcolor{black}{To help readers better understand our design rationale, we provide an illustrative example of the gradient vectors and their projections in Figure~\ref{fig:pca} (see caption for details)}. \textcolor{black}{Theoretically, Remark~\ref{rem:1} below shows that the projection of the OOD gradient vector to the top singular vector of the gradient matrix $\*G$ is  on average provably larger than that of the ID gradient vector, which rigorously justifies our idea of using the score $\tau$ for separating the ID and OOD data. }
\vspace{-0.5em}
\textcolor{black}{
\begin{Remark}\label{R1}
\label{rem:1}
    Theorem \ref{the:main4.0-app} in Appendix~\ref{sec:proof_1} has shown that under proper assumptions, if we have sufficient data and large-size model, then with the high probability:
    \begin{itemize}
    \item the mean projected magnitude of OOD gradients in the direction of the top singular vector of $\*G$ can be lower bounded by a positive constant $C/\pi$;
    \item the mean projected magnitude of ID gradients in the direction of the top singular vector is upper bounded by a small value close to zero.
\end{itemize}
\end{Remark}}

Finally, we regard $\mathcal{S}_T = \{\tilde{\mathbf{x}}_i\in \mathcal{S}_{\text{wild}}: \mathbf{\tau}_i>T\}$ as the (potentially noisy) candidate outlier set, where $T$ is the filtering threshold. The threshold can be chosen on the ID data $\mathcal{S}^{\text{in}}$ so that a high fraction (e.g., 95\%) of ID samples is below it. In Section~\ref{sec:theory}, we will provide formal guarantees, rigorously justifying that the set $\mathcal{S}_T$ returns outliers with \textcolor{black}{a large probability}. \textcolor{black}{We discuss and compare with alternative gradient-based scores (e.g., GradNorm~\citep{huang2021importance}) for filtering in Section~\ref{sec:results}.} {In Appendix~\ref{sec:num_of_sing_vectors}, we discuss the variants of using multiple singular vectors, which yield similar results.}

\begin{figure*}
    \centering
    \scalebox{1.0}{
    \includegraphics[width=\linewidth]{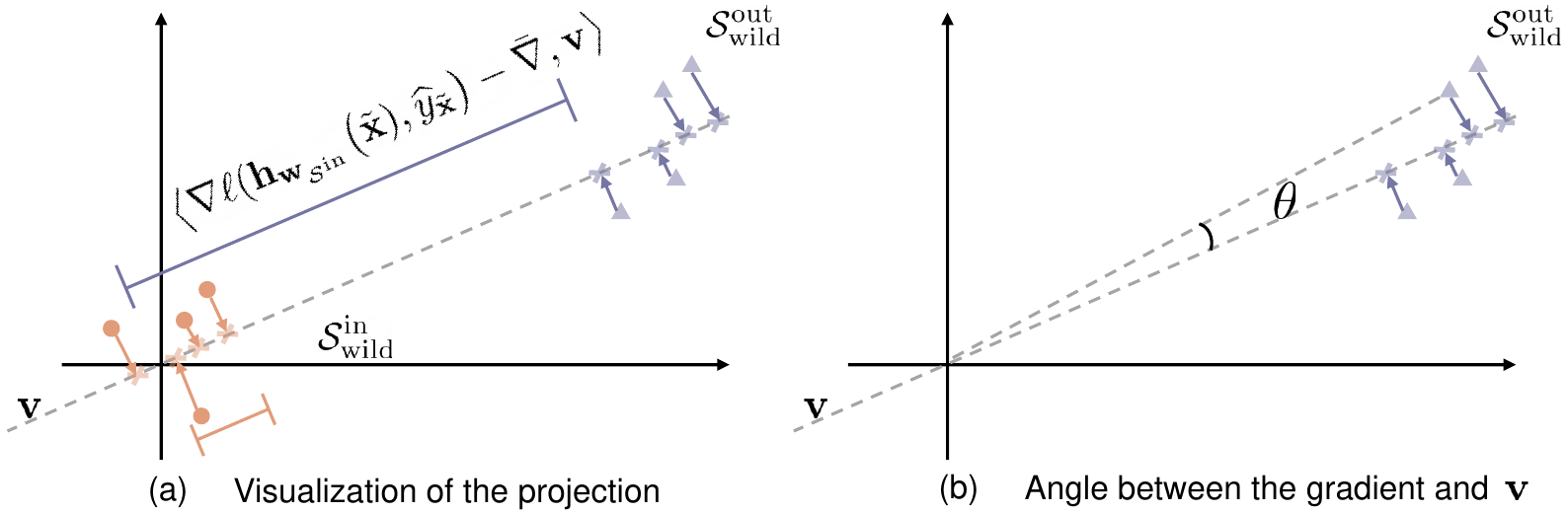}
    }
      \vspace{-1.5em}
    \caption{\small (a) 
    \textcolor{black}{Visualization of the gradient vectors, and their projection onto the top singular vector $\*v$ (in gray dashed line). The gradients of inliers from $\mathcal{S}_{\text{wild}}^{\text{in}}$ (colored in \textcolor{orangeinplot}{orange}) are close to the origin (reference gradient $\bar{\nabla}$). In contrast, the gradients of outliers from $\mathcal{S}_{\text{wild}}^{\text{out}}$ (colored in \textcolor{purpleinplot}{purple}) are farther away.     } (b) The angle $\theta$ between the gradient of set $\mathcal{S}_{\text{wild}}^{\text{out}}$ and the singular vector $\*v$. Since $\*v$ is searched to maximize the distance from the projected points (cross marks) to the origin (sum over all the gradients in $\mathcal{S}_{\text{wild}}$), $\*v$ points to the direction of OOD data in the wild with a small $\theta$. \textcolor{black}{This further translates into a high  filtering score $\tau$, which is essentially the norm after projecting a gradient vector onto $\*v$. As a result, filtering outliers by $\mathcal{S}_T = \{\tilde{\mathbf{x}}_i\in \mathcal{S}_{\text{wild}}: \mathbf{\tau}_i>T\}$ will approximately return the purple  OOD samples in the wild data. }}
  \vspace{-0.5em}
    \label{fig:pca}
\end{figure*}

\textbf{An illustrative example of algorithm effect.} To see the effectiveness of our filtering score, we test on two simulations in Figure~\ref{fig:toy} (a). These simulations are constructed with simplicity in mind, to facilitate understanding. Evaluations on complex high-dimensional data will be provided in Section~\ref{sec:exp}. In particular, the wild data is a mixture of ID (multivariate Gaussian with three classes) and OOD. We consider two scenarios of OOD distribution, with ground truth colored in \textcolor{purpleinplot}{purple}. Figure~\ref{fig:toy} (b) exemplifies the  outliers (in \textcolor{greeninplot}{green}) identified using our proposed method, which largely aligns with the ground truth. The error rate of $\mathcal{S}_T$ containing ID data is only $8.4\%$ and $6.4\%$ for the two scenarios considered. Moreover, the filtering score distribution displays a clear separation between the ID vs. OOD parts, as evidenced in Figure~\ref{fig:toy} (c).

\textbf{Remark 2.} \emph{Our filtering process can be easily extended into $K$-class classification. In this case, one can maintain a class-conditional reference gradient $ \bar{\nabla}_k$, one for each class $k \in [1,K]$, estimated on ID data belonging to class $k$, {which captures the characteristics for each ID class.} Similarly, the top singular vector computation can also be performed in a class-conditional manner, where we replace the gradient matrix with the {class-conditional} $\mathbf{G}_k$, containing gradient vectors of wild samples being predicted as class $k$.}

\subsection{Training the OOD Classifier with the Candidate Outliers}
\label{sec:training}
After obtaining the candidate outlier set $\mathcal{S}_T$  from the wild data, we train an OOD classifier $\*g_{\boldsymbol{\theta}}$ that optimizes for the separability between the ID vs. candidate outlier
data. In particular, our training objective can be viewed as explicitly optimizing the level-set 
based on the model output (threshold at 0), where the labeled ID
data $\*x$ from $\mathcal{S}^{\text{in}}$ has positive values and vice versa.
\begin{equation}\label{eq:reg_loss}
\begin{split}
        R_{\mathcal{S}^{\text{in}},\mathcal{S}_T}(\*g_{\boldsymbol{\theta}}) & =   R_{\mathcal{S}^{\text{in}}}^{+}(\*g_{\boldsymbol{\theta}})+R_{\mathcal{S}_T}^{-}(\*g_{\boldsymbol{\theta}}) \\& =  \mathbb{E}_{\*x \in  \mathcal{S}^{\text{in}}}~~\mathds{1}\{\*g_{\boldsymbol{\theta}}(\*x) \leq 0\}+\mathbb{E}_{\tilde{\*x} \in  \mathcal{S}_T}~~\mathds{1} \{\*g_{\boldsymbol{\theta}}(\tilde{\*x}) > 0\}.
        \end{split}
\end{equation}
To make the $0/1$ loss tractable, we replace it with the
binary sigmoid loss, a smooth approximation of the $0/1$
loss. We  train $\*g_{\boldsymbol{\theta}}$ along with the ID risk in Eq.~\ref{eq:erm} to ensure ID accuracy. 
Notably, the training enables strong generalization performance for test OOD samples drawn from $\mathbb{P}_\text{out}$.
  We provide formal guarantees on the generalization bound in Theorem~\ref{the:main2}, as well as empirical support in Section~\ref{sec:exp}. \textcolor{black}{A pseudo algorithm of \model is  in Appendix (see Algorithm~\ref{alg:algo})}.

\begin{figure*}
    \centering
    \scalebox{1.0}{
    \includegraphics[width=\linewidth]{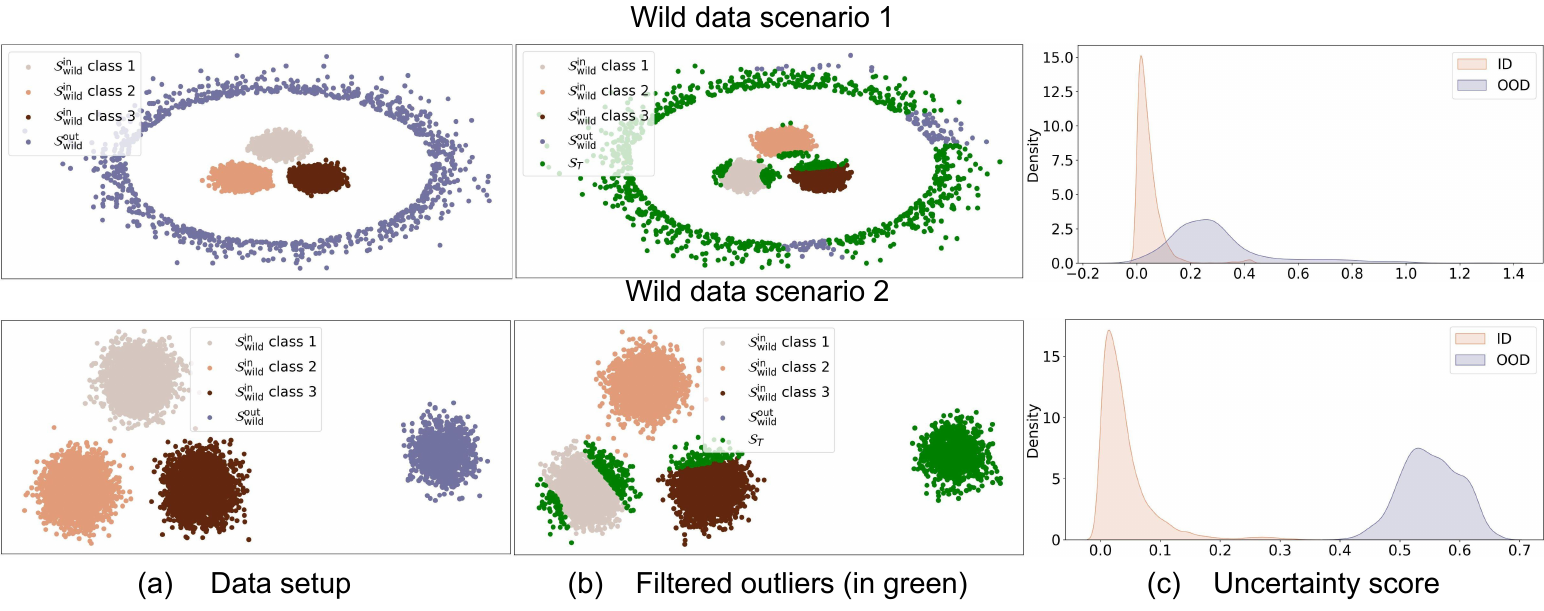}
    }
    \caption{\small Example of \model on two different scenarios of the unlabeled wild  data. (a) Setup of the ID/inlier $\mathcal{S}_{\text{wild}}^{\text{in}}$ and OOD/outlier data  $\mathcal{S}_{\text{wild}}^{\text{out}}$ in the wild. The inliers are sampled from three multivariate Gaussians. We construct two different distributions of outliers (see details in Appendix~\ref{sec:details_of_toy_app}). (b) The filtered outliers (in green) by \model, where the error rate of filtered outliers $\mathcal{S}_T$ containing inlier data  is $8.4\%$ and $6.4\%$, respectively. (c) The density distribution of the filtering score $\tau$, which is separable for inlier and outlier data in the wild and thus benefits the training of the OOD classifier leveraging the filtered outlier data for binary classification. }
    \label{fig:toy}
\vspace{-0.2cm}
\end{figure*}

\section{Theoretical Analysis}
\label{sec:theory}

We now  provide theory to support our proposed algorithm. Our main theorems justify the two components in our algorithm. As an overview,  Theorem~\ref{MainT-1} provides a provable bound on the error rates using our filtering procedure. Based on the estimations on error rates, Theorem \ref{the:main2} gives the generalization bound \emph{w.r.t.} the empirical OOD classifier $\*g_{\boldsymbol{\theta}}$, learned on ID data and noisy set of outliers. We specify several mild assumptions and necessary notations for our theorems in Appendix~\ref{notation,definition,Ass,Const}. {Due to space limitation, we omit unimportant constants and simplify the statements of our theorems. We defer the \textbf{full formal} statements in  Appendix \ref{main_theorems}. All proofs can be found in Appendices \ref{Proofmain} and \ref{NecessaryLemma}.}

\subsection{Analysis on Separability}
\label{sec:ana_1}
Our main theorem quantifies the separability of the outliers in the wild by using the filtering procedure (\emph{c.f.} Section~\ref{sec:detect}).  Let $\text{ERR}_{\text{out}}$ and $\text{ERR}_{\text{in}}$ be the error rate of OOD data being regarded as ID and the error rate of ID data being regarded as OOD, i.e., $\text{ERR}_{\text{out}}=|\{\tilde{\*x}_i\in \mathcal{S}^{\text{out}}_{\text{wild}}: \tau_i \leq T\}|/| \mathcal{S}^{\text{out}}_{\text{wild}}| $ and $\text{ERR}_{\text{in}}= |\{\tilde{\*x}_i\in \mathcal{S}^{\text{in}}_{\text{wild}}: \tau_i > T\}|/| \mathcal{S}^{\text{in}}_{\text{wild}}| $,  where $\mathcal{S}^{\text{in}}_{\text{wild}}$ and $\mathcal{S}^{\text{out}}_{\text{wild}}$ denote the sets of inliers and outliers from the wild data $\mathcal{S}_{\text{wild}}$. Then $\text{ERR}_{\text{out}}$ and $\text{ERR}_{\text{in}}$ have the following  generalization bounds.

\begin{tcolorbox}[enhanced,attach boxed title to top center={yshift=-3mm,yshifttext=-1mm},
  colback=gray!5!white,colframe=gray!75!black,colbacktitle=red!80!black,
  title=,fonttitle=\bfseries,
  boxed title style={size=small,colframe=red!50!black} ]
\begin{theorem}\label{MainT-1}
 (Informal).  Under mild conditions, if $\ell(\mathbf{h}_{\mathbf{w}}(\mathbf{x}),y)$ is $\beta_1$-smooth w.r.t. $\mathbf{w}$, \textcolor{black}{$\mathbb{P}_{\text{wild}}$ has $(\gamma,\zeta)$-discrepancy  w.r.t. $\mathbb{P}_{\mathcal{X}\mathcal{Y}}$ (\emph{c.f.} Appendices~\ref{sec:definition_app},~\ref{sec:assumption_app})}, and there is $\eta\in (0,1)$ s.t. $\Delta = (1-\eta)^2\zeta^2 - 8\beta_1 R_{{\text{in}}}^*>0$,  then when
     $
         n = {\Omega} \big ({d}/{\min \{ \eta^2 \Delta,(\gamma-R_{{\text{in}}}^*)^2\}}  \big), m = \Omega \big ({d}/{\eta^2\zeta^2} \big),
    $  with the probability at least $0.9$, for $0<T<0.9M'$ $($$M'$ is the upper bound of score $\tau_i$),
    \begin{equation}
      {\rm ERR}_{\text{in}}  \leq    \frac{8\beta_1 }{T}R_{\text{in}}^*+ O \Big ( \frac{1}{T}\sqrt{\frac{d}{n}}\Big )+   O \Big ( \frac{1}{T}\sqrt{\frac{d}{(1-\pi)m}}\Big ),
\end{equation}

\begin{equation}\label{Eq::bound1.0}
\begin{split}
    {\rm ERR}_{\text{out}} &\leq \delta(T) + O \Big(  \sqrt{\frac{d}{{\pi}^2n}}\Big)+ O \Big( \sqrt{\frac{\max\{d,{\Delta_{\zeta}^{\eta}}^2/{\pi}^2\}}{{\pi}^2(1-\pi)m}}\Big),
     \end{split}
\end{equation}
     where $R^{*}_{{\text{in}}}$ is the optimal ID risk, i.e., $R^{*}_{{\text{in}}}=\min_{\mathbf{w}\in \mathcal{W}} \mathbb{E}_{(\mathbf{x},y)\sim \mathbb{P}_{\mathcal{X}\mathcal{Y}}} \ell(\mathbf{h}_{\mathbf{w}}(\mathbf{x}),y)$,
     \begin{equation}\label{main-error}
         \delta(T) = {\max\{0,1-\Delta_{\zeta}^{\eta}/\pi\}}/{(1-T/M')},~~~\Delta_{\zeta}^{\eta} = {0.98\eta^2 \zeta^2} - 8\beta_1R_{\text{in}}^*,
     \end{equation}
$d$ is the dimension of the space $\mathcal{W}$, and $\pi$ is the OOD class-prior  probability in the wild.
\end{theorem}
\end{tcolorbox}

\textbf{Practical implications of Theorem~\ref{MainT-1}.} 
    The above theorem states that under mild assumptions, the errors $\text{ERR}_{\text{out}}$ and $\text{ERR}_{\text{in}}$ are upper bounded. For $\text{ERR}_{\text{in}}$,  if the following two regulatory conditions hold: 1) the sizes of the labeled ID $n$ and wild data $m$  are sufficiently large; 2) the optimal ID risk $R_{\text{in}}^*$ is small, then the upper bound is tight. For $\text{ERR}_{\text{out}}$, $\delta(T)$ defined in Eq. \ref{main-error} becomes the main error, if we have sufficient data. To further study the main error $\delta(T)$ in Eq. \ref{Eq::bound1.0}, Theorem~\ref{The-1.1} shows that the error $\delta(T)$ could be close to zero under  practical conditions.

\begin{tcolorbox}[enhanced,attach boxed title to top center={yshift=-3mm,yshifttext=-1mm},
  colback=gray!5!white,colframe=gray!75!black,colbacktitle=red!80!black,
  title=,fonttitle=\bfseries,
  boxed title style={size=small,colframe=red!50!black} ]
\begin{theorem}\label{The-1.1}
 (Informal). 1) If  $\Delta_{\zeta}^{\eta} \geq  (1-\epsilon)\pi$ for a small error $\epsilon \geq 0$, then the main error $\delta(T)$ defined in Eq. \ref{main-error} satisfies that
    \begin{equation}
        \delta(T) \leq \frac{\epsilon}{1-T/{M'}}.
    \end{equation}
    2) If $\zeta\geq 2.011\sqrt{8\beta_1 R_{\text{in}}^*} +1.011 \sqrt{\pi}$, then there exists $\eta\in (0,1)$ ensuring that $\Delta>0$ and $ \Delta_{\zeta}^{\eta}> \pi$ hold, which implies that the main error $\delta(T)=0$.
\end{theorem}
\end{tcolorbox}

\textbf{Practical implications of Theorem~\ref{The-1.1}.}   Theorem~\ref{The-1.1} states that if   the discrepancy $\zeta$ between two data distributions $\mathbb{P}_{\text{wild}}$ and $\mathbb{P}_{\text{in}}$  is larger than some small values,  the main error $\delta(T)$ could be close to zero.  Therefore, by combining with the two regulatory conditions mentioned in Theorem~\ref{MainT-1}, the error $\text{ERR}_{\text{out}}$ could be close to zero. Empirically, we  verify the conditions of Theorem~\ref{The-1.1} in Appendix~\ref{sec:verification_discrepancy}, which can hold true easily in practice. In addition, given fixed optimal ID risk $R_{\text{in}}^*$ and fixed sizes of the labeled ID $n$ and wild data $m$, we observe that the bound of $  {\rm ERR}_{\text{in}} $ will increase when $\pi$ goes from 0 to 1. In contrast, the bound of $  {\rm ERR}_{\text{out}} $ \textcolor{black}{is non-monotonic} when $\pi$ increases, which will firstly decrease and then increase. The observations align well with empirical results in Appendix~\ref{sec:verification_discrepancy}.

 \textbf{Impact of using predicted labels for the wild data.}  Recall in Section~\ref{sec:detect} that the filtering step uses the predicted labels to estimate the gradient for wild data, which is unlabeled. 
 \textcolor{black}{To analyze the impact theoretically,  we show in Appendix Assumption~\ref{Ass2} that the loss incurred by using the predicted label is smaller than the loss by using any label in the label space. This property is included in Appendix Lemmas~\ref{gamma-approximation-1} and \ref{gamma-approximation-2} to constrain the filtering score in Appendix Theorem~\ref{T1} and then filtering error in Theorem~\ref{MainT-1}.  In harder classification cases, the predicted label deviates more from the true label for the wild ID data, which leads to a looser bound for the filtering accuracy in Theorem~\ref{MainT-1}.}

\textcolor{black}{Empirically, we calculate and compare the filtering accuracy and its OOD detection result on  \textsc{Cifar-10} and  \textsc{Cifar-100} ( \textsc{Textures}~\citep{cimpoi2014describing} as the wild OOD). \model achieves a result of $\rm ERR_{\text{in}}=0.018$ and  $\rm ERR_{\text{out}}=0.17$  on \textsc{Cifar-10} (easier classification case), which outperforms the result of $\rm ERR_{\text{in}}=0.037$ and  $\rm ERR_{\text{out}}=0.30$ on \textsc{Cifar-100} (harder classification case), aligning with our reasoning above. The experimental details are provided in Appendix~\ref{sec:experiments_on_predicted_label_app}. Analysis of using random labels for the wild data is provided in Appendix~\ref{sec:experiments_on_random_label_app}.}

\subsection{Analysis on Learnability}
\label{sec:ana_2}
Leveraging the filtered outliers $\mathcal{S}_T$, \model then trains an OOD classifier $\*g_{\boldsymbol{\theta}}$ with the data from  in-distribution $\mathcal{S}^{\text{in}}$ and data from $\mathcal{S}_T$ as OOD. In this section, we provide the generalization error bound for the learned OOD classifier to quantify its learnability. Specifically, we show that a small error guarantee in Theorem~\ref{MainT-1} implies that we can get a tight generalization error bound.

\begin{tcolorbox}[enhanced,attach boxed title to top center={yshift=-3mm,yshifttext=-1mm},
  colback=gray!5!white,colframe=gray!75!black,colbacktitle=red!80!black,
  title=,fonttitle=\bfseries,
  boxed title style={size=small,colframe=red!50!black} ]
\begin{theorem}\label{the:main2}
(Informal). Let $L$ be the upper bound of $\ell_{\text{b}}(\mathbf{g}_{\boldsymbol{\theta}}(\mathbf{x}),y_{\text{b}})$, i.e., $\ell_{\text{b}}(\mathbf{g}_{\boldsymbol{\theta}}(\mathbf{x}),y_{\text{b}}) \leq L$. Under  conditions in Theorem~\ref{MainT-1}, if we  
further require $n = \Omega \big ({d}/{\min \{\pi,\Delta_{\zeta}^{\eta}\}^2}\big)$, $m= \Omega \big( {(d+\Delta_{\zeta}^{\eta})}/{(\pi^2(1-\pi)\min \{\pi,\Delta_{\zeta}^{\eta}\}^2)}\big)$,

then with the probability at least $0.89$, for any $0<T<0.9 M' \min \{1,\Delta_{\zeta}^{\eta}/\pi\}$, the OOD classifier $\*g_{\widehat{\boldsymbol{\theta}}_T}$ learned by \model satisfies
      \begin{equation}\label{Eq::Bound2}
     \begin{aligned}
               & R_{\mathbb{P}_{\text{in}},\mathbb{P}_{\text{out}}}(\mathbf{g}_{\widehat{\boldsymbol{\theta}}_T}) \leq  \min_{\boldsymbol{\theta} \in \Theta} R_{\mathbb{P}_{\text{in}},\mathbb{P}_{\text{out}}}(\mathbf{g}_{{\boldsymbol{\theta}}})+\frac{3.5 L}{1-\delta(T)}\delta({T})+ \frac{9(1-\pi)L\beta_1}{\pi(1-\delta(T))T} R_{\text{in}}^*
                \\ + &O\Big(\frac{\max\{\sqrt{d},\sqrt{d'}\}}{\min\{\pi,\Delta_{\zeta}^{\eta}\}T'}\sqrt{\frac{1}{n}}\Big )+O\Big(\frac{\max \{\sqrt{d},\sqrt{d'},\Delta_{\zeta}^{\eta}\}}{\min\{\pi,\Delta_{\zeta}^{\eta}\}T'}\sqrt{\frac{1}{\pi^2(1-\pi)m}}\Big ),    
             \end{aligned}
             \end{equation}
                where  $\Delta_{\zeta}^{\eta}$, $d$ and $\pi$ are shown in Theorem \ref{MainT-1}, $d'$ is the dimension of space $\Theta$, $T'=T/(1+T)$,
                 and the risk $R_{\mathbb{P}_{\text{in}},\mathbb{P}_{\text{out}}}(\mathbf{g}_{{\boldsymbol{\theta}}})$ corresponds to the  empirical risk in Eq. \ref{eq:reg_loss} with loss $\ell_{\text{b}}$, i.e.,
                \begin{equation}\label{Real-risk}
                    R_{\mathbb{P}_{\text{in}},\mathbb{P}_{\text{out}}}(\mathbf{g}_{\widehat{\boldsymbol{\theta}}_T}) = \mathbb{E}_{\mathbf{x}\sim \mathbb{P}_{\text{in}}} \ell_{\text{b}}(\mathbf{g}_{\boldsymbol{\theta}}(\mathbf{x}),y_{+})+\mathbb{E}_{\mathbf{x}\sim \mathbb{P}_{\text{out}}} \ell_{\text{b}}(\mathbf{g}_{\boldsymbol{\theta}}(\mathbf{x}),y_{-}).
                \end{equation}
\end{theorem}
\end{tcolorbox}

\textbf{Insights.}
The above theorem presents the generalization error bound of the OOD classifier $\*g_{\widehat{\boldsymbol{\theta}}_T}$ learned by using the filtered OOD data $\mathcal{S}_T$. When we have sufficient labeled ID data and wild data,  then the risk of the OOD classifier $\*g_{\widehat{\boldsymbol{\theta}}_T}$ is close to the optimal risk, i.e., $\min_{\boldsymbol{\theta} \in \Theta} R_{\mathbb{P}_{\text{in}},\mathbb{P}_{\text{out}}}(\mathbf{g}_{{\boldsymbol{\theta}}})$, if the optimal ID risk $R_{\text{in}}^*$ is small, and either one of the conditions in Theorem~\ref{The-1.1} is satisfied. 
\vspace{-1em}

\section{Experiments}
\label{sec:exp}

In this section, we verify the effectiveness of our algorithm on modern neural networks. We aim to show that the generalization bound of the OOD classifier (Theorem~\ref{the:main2}) indeed translates into strong empirical performance, establishing state-of-the-art results (Section~\ref{sec:results}).

\vspace{-0.3cm}

\subsection{Experimental Setup}
\label{sec:exp_steup}
\textbf{Datasets.} 
We follow exactly the same experimental setup as WOODS~\citep{katzsamuels2022training}, which introduced the problem of learning OOD detectors with wild data. This allows us to draw fair comparisons. WOODS considered \textsc{Cifar-10} and \textsc{Cifar-100}~\citep{krizhevsky2009learning} as ID datasets ($\mathbb{P}_{\text{in}}$). For OOD test datasets ($\mathbb{P}_{\text{out}}$), we use a suite of natural image datasets including \textsc{Textures}~\citep{cimpoi2014describing}, \textsc{Svhn}~\citep{netzer2011reading}, \textsc{Places365}~\citep{zhou2017places}, \textsc{Lsun-Resize} \& \textsc{Lsun-C}~\citep{DBLP:journals/corr/YuZSSX15}. To simulate the wild data ($\mathbb{P}_{\text{wild}}$), we mix a subset of ID data
(as $\mathbb{P}_{\text{in}}$) with the outlier dataset (as $\mathbb{P}_{\text{out}}$) under the default $\pi=0.1$, which reflects the practical scenario that most data would remain ID.
Take~\textsc{Svhn} as an example, we use \textsc{Cifar+Svhn} as the unlabeled wild  data and test on \textsc{Svhn} as OOD. We simulate this for all OOD datasets {and} provide analysis of differing $\pi \in \{0.05, 0.1,..., 1.0\}$ in Appendix~\ref{sec:verification_discrepancy}. Note that we split \textsc{Cifar} datasets into two halves: $25,000$
images as ID training data, and the remainder $25,000$ for
creating the wild mixture data. {We use the weights from the penultimate layer for gradient calculation, which was shown to be the most informative for OOD detection~\citep{huang2021importance}.} Experimental details are provided in Appendix~\ref{sec:detail_app}.

\textbf{Evaluation metrics.} We report the following metrics: (1) the false positive rate (FPR95$\downarrow$) of OOD samples when the true positive rate of ID samples is 95\%, (2) the area under the receiver operating characteristic curve (AUROC$\uparrow$), and (3) ID classification Accuracy (ID ACC$\uparrow$).

\begin{table}[t]
  \centering
  \small
  \vspace{-1em}
  \caption{\small OOD detection performance on \textsc{Cifar-100} as ID. All methods are trained on Wide ResNet-40-2 for 100 epochs. For each dataset, we create corresponding wild mixture distribution
$\mathbb{P}_\text{wild} = (1 - \pi) \mathbb{P}_\text{in} + \pi \mathbb{P}_\text{out}$ for training and test on the corresponding OOD dataset. Values are percentages {averaged over 10 runs}. {Bold} numbers highlight the best results. Table format credit to~\citet{katzsamuels2022training}.}
    \scalebox{0.67}{
    \begin{tabular}{cccccccccccccc}
    \toprule
    \multirow{3}[4]{*}{Methods} & \multicolumn{12}{c}{OOD Datasets}                                                             & \multirow{3}[4]{*}{ID ACC} \\
    \cmidrule{2-13}
          & \multicolumn{2}{c}{\textsc{Svhn}} & \multicolumn{2}{c}{\textsc{Places365}} & \multicolumn{2}{c}{\textsc{Lsun-C}} & \multicolumn{2}{c}{\textsc{Lsun-Resize}} & \multicolumn{2}{c}{\textsc{Textures}} & \multicolumn{2}{c}{Average} &  \\
\cmidrule{2-13}          & FPR95 & AUROC & FPR95 & AUROC & FPR95 & AUROC & FPR95 & AUROC & FPR95 & AUROC & FPR95 & AUROC &  \\
  \hline
     \multicolumn{14}{c}{With $\mathbb{P}_{\text{in}}$ only} \\
    MSP   &84.59 &71.44 & 82.84 & 73.78  & 66.54 &  83.79 & 82.42& 75.38&83.29& 73.34& 79.94 &75.55 &75.96\\
ODIN&  84.66& 67.26&  87.88&  71.63  &55.55 &87.73&71.96 &81.82 &79.27 &73.45&75.86 &76.38& 75.96
 \\
Mahalanobis & 57.52 &86.01  &  88.83& 67.87&91.18& 69.69&21.23& 96.00&39.39& 90.57&59.63& 82.03& 75.96
\\
Energy & 85.82 &73.99 &80.56& 75.44&  35.32& 93.53& 79.47 &79.23&79.41 &76.28&  72.12& 79.69 &75.96\\
KNN& 66.38	&83.76&	79.17&	71.91&	70.96&	83.71&	77.83&	78.85&	88.00	&67.19	&76.47	&77.08&	75.96 \\
ReAct & 74.33 & 88.04 &  81.33 & 74.32 &  39.30 & 91.19 & 79.86 & 73.69  & 67.38 & 82.80  & 68.44 & 82.01 & 75.96 \\
 DICE&   88.35 & 72.58 & 81.61 & 75.07 & 26.77 & 94.74  & 80.21 & 78.50 &  76.29 & 76.07& 70.65 & 79.39 & 75.96 \\

 ASH   &21.36 &94.28 &68.37 &71.22 &15.27& 95.65& 68.18 &85.42& 40.87 &92.29&  42.81& 87.77 & 75.96\\
CSI&  64.70 &84.97&  82.25& 73.63  &  38.10  & 92.52   &  91.55&  63.42& 74.70 &92.66& 70.26 &81.44& 69.90\\
KNN+ & 32.21 & 93.74 &68.30 & 75.31 &40.37 & 86.13  & 44.86 & 88.88 & 46.26 & 87.40 &46.40 & 86.29& 73.78\\

\hline
 \multicolumn{14}{c}{With $\mathbb{P}_{\text{in}}$ and $\mathbb{P}_{\text{wild}}$ } \\
     OE& 1.57 &  99.63  & 60.24&  83.43 & 3.83 &  99.26 & 0.93 &  99.79 & 27.89 &  93.35 & 18.89 & 95.09 &  71.65\\
  Energy (w/ OE) &1.47 &  99.68  & 54.67 & 86.09 & 2.52 &  99.44 & 2.68 &  99.50 & 37.26 &  91.26 & 19.72 &95.19 &  73.46\\

WOODS&0.12& \textbf{99.96} & 29.58&  90.60 & 0.11&  \textbf{99.96} & 0.07&  \textbf{99.96} & 9.12& 96.65 & 7.80&  97.43& 75.22\\

\rowcolor[HTML]{EFEFEF} \model  & \textbf{0.07} & 99.95 & \textbf{3.53} & \textbf{99.06} & \textbf{0.06} & 99.94 &   \textbf{0.02} & 99.95 & \textbf{5.73} & \textbf{98.65}  & \textbf{1.88}& \textbf{99.51}& 73.71\\
\rowcolor[HTML]{EFEFEF}  (Ours)   & {$^{\pm}$0.02} &$^{\pm}$0.00 & $^{\pm}$0.17 & $^{\pm}$0.06& $^{\pm}$0.01 & $^{\pm}$0.21 & $^{\pm}$0.00 & $^{\pm}$0.03 & $^{\pm}$0.34 & $^{\pm}$0.02 & $^{\pm}$0.11 & $^{\pm}$0.02 & $^{\pm}$0.78  \\
   \hline
\end{tabular}}
    \label{tab:c100}
    \vspace{-1em}
\end{table}

\vspace{-0.5cm}
\subsection{Empirical Results}
\label{sec:results}
\textbf{\model achieves superior empirical performance.}  
We present results in Table~\ref{tab:c100} on \textsc{Cifar-100},  where \model outperforms the state-of-the-art method. 
Our comparison covers an extensive collection of competitive OOD detection methods, which can be divided into two categories: trained with and without the wild data.
For methods using ID data $\mathbb{P}_{\text{in}}$ only, we compare with methods 
such as {MSP}~\citep{hendrycks2016baseline}, ODIN \citep{liang2018enhancing},  Mahalanobis distance~\citep{lee2018simple}, Energy score \citep{liu2020energy}, ReAct~\citep{sun2021react}, DICE~\citep{sun2022dice}, KNN distance~\citep{sun2022out}, and ASH~\citep{djurisic2023extremely}---all of which use a model trained with cross-entropy loss.
We also include the method  based on contrastive loss, including CSI~\citep{tack2020csi} and KNN+~\citep{sun2022out}. 
For methods using both ID and wild data, we compare with Outlier Exposure (OE)~\citep{hendrycks2018deep} and energy-regularization learning~\citep{liu2020energy}, which regularize the model by producing
lower confidence or higher energy on the auxiliary outlier data. Closest to ours is WOODS~\citep{katzsamuels2022training}, which leverages wild data for OOD learning with a constrained optimization approach. For a fair comparison, all the methods in this group are trained using the same ID and in-the-wild data, under
the same mixture ratio $\pi=0.1$.

The results demonstrate that: \textbf{(1)} Methods trained with both ID and wild data perform much better than those trained with only ID data. For example, on \textsc{Places365}, \model reduces the FPR95 by 64.77\% compared with KNN+, which highlights the advantage of using in-the-wild data for model regularization. \textbf{(2)} \model performs even better compared to the competitive methods using $\mathbb{P}_{\text{wild}}$. On \textsc{Cifar-100}, \model achieves an average FPR95 of 1.88\%, which is a 5.92\% improvement from WOODS. At the same time, \model maintains a comparable ID accuracy. {{The slight discrepancy is due to that our method only observes 25,000 labeled ID samples, whereas baseline methods (without using wild data) utilize the entire \textsc{Cifar} training data with 50,000 samples.}} \textbf{(3)} The strong empirical performance achieved by \model directly justifies and echoes our theoretical result in Section~\ref{sec:theory}, where we showed the algorithm has a provably small generalization error. \emph{Overall, our algorithm enjoys both theoretical guarantees and empirical effectiveness.}

\textbf{Comparison with GradNorm as filtering score.} 
\cite{huang2021importance} proposed directly employing the vector norm of gradients, backpropagated
from the KL divergence between the softmax output and a uniform probability
distribution for OOD detection. Differently, our \model derives the filtering score by performing singular value decomposition and using the norm of the projected gradient onto the top singular vector (\emph{c.f.} Section~\ref{sec:detect}). We compare \model with a variant in Table~\ref{tab:gradnorm+}, where we replace the filtering score in \model with the GradNorm score and then train the OOD classifier. The result underperforms \model, showcasing the effectiveness of our filtering score.

\begin{table}[t]
  \centering
  \small
  \caption{\small Comparison with using GradNorm as the filtering score. We use \textsc{Cifar-100} as ID. All methods are trained on Wide ResNet-40-2 for 100 epochs with $\pi=0.1$.  {Bold} numbers are superior results.   }
  \vspace{-1em}
    \scalebox{0.68}{
    \begin{tabular}{cccccccccccccc}
    \toprule
    \multirow{3}[4]{*}{Filter score} & \multicolumn{12}{c}{OOD Datasets}                                                             & \multirow{3}[4]{*}{ID ACC} \\
    \cmidrule{2-13}
          & \multicolumn{2}{c}{\textsc{Svhn}} & \multicolumn{2}{c}{\textsc{Places365}} & \multicolumn{2}{c}{\textsc{Lsun-C}} & \multicolumn{2}{c}{\textsc{Lsun-Resize}} & \multicolumn{2}{c}{\textsc{Textures}} & \multicolumn{2}{c}{Average} &  \\
\cmidrule{2-13}          & FPR95 & AUROC & FPR95 & AUROC & FPR95 & AUROC & FPR95 & AUROC & FPR95 & AUROC & FPR95 & AUROC &  \\
    \midrule

GradNorm &  1.08 & 99.62&62.07 & 84.08& 0.51 & 99.77 & 5.16 & 98.73 &50.39 & 83.39   &23.84 &93.12 & 73.89 \\

\rowcolor[HTML]{EFEFEF} Ours &\textbf{0.07} & \textbf{99.95} & \textbf{3.53} & \textbf{99.06} & \textbf{0.06} & \textbf{99.94} &   \textbf{0.02} & \textbf{99.95} & \textbf{5.73} & \textbf{98.65}  & \textbf{1.88}& \textbf{99.51}& 73.71 \\
   \hline
   
\end{tabular}}
    \label{tab:gradnorm+}
    \vspace{-1.5em}
\end{table}

\textbf{Additional ablations.} {Due to space limitations, we defer additional experiments in the Appendix, including \textbf{(1)} analyzing the effect of  ratio $\pi$ (Appendix~\ref{sec:verification_discrepancy}), \textbf{(2)} results on \textsc{Cifar}-10 (Appendix~\ref{sec:c10_app}), \textbf{(3)} evaluation on  \textbf{\emph{unseen}} OOD datasets (Appendix~\ref{sec:mixing_ratio}), \textbf{(4)} near OOD evaluations (Appendix~\ref{sec:near_ood}), and \textbf{(5)} the effect of using multiple singular vectors for calculating the filtering score (Appendix~\ref{sec:num_of_sing_vectors}) .
\vspace{-1em}

\section{Related Work}
\label{sec:related_work}
\vspace{-1em}

\textbf{OOD detection} has attracted a surge of interest in recent years~\citep{fort2021exploring,yang2021generalized,fang2022learnable,zhu2022boosting,ming2022delving,ming2022spurious,yang2022openood,wang2022outofdistribution,galil2023a,djurisic2023extremely,tao2023nonparametric,zheng2023out,wang2022watermarking,wang2023outofdistribution,narasimhan2023learning,yang2023auto,uppaal2023fine,zhu2023diversified,zhu2023unleashing,bai2023feed,ming2023finetune,zhang2023openood,gu2023critical,ghosal2024how}. 
One line of work performs OOD detection by devising scoring functions, including confidence-based methods~\citep{bendale2016towards,hendrycks2016baseline,liang2018enhancing}, energy-based score~\citep{liu2020energy,wang2021canmulti,wu2023energybased}, distance-based approaches~\citep{lee2018simple,tack2020csi,DBLP:journals/corr/abs-2106-09022,2021ssd,sun2022out, du2022siren, ming2023cider,ren2023outofdistribution}, gradient-based score~\citep{huang2021importance}, and {Bayesian approaches~\citep{gal2016dropout,lakshminarayanan2017simple,maddox2019simple,dpn19nips,Wen2020BatchEnsemble,kristiadi2020being}.} Another  line of work addressed OOD detection by training-time regularization~\citep{bevandic2018discriminative,malinin2018predictive,geifman2019selectivenet,hein2019relu,meinke2019towards,DBLP:conf/nips/JeongK20,liu2020simple,DBLP:conf/icml/AmersfoortSTG20,DBLP:conf/iccv/YangWFYZZ021,DBLP:conf/icml/WeiXCF0L22,du2022unknown,du2023dream,wang2023learning}. For example, the model is
regularized to produce lower confidence~\citep{lee2018training,hendrycks2018deep} or higher energy~\citep{liu2020energy,du2022towards,DBLP:conf/icml/MingFL22} on the outlier data. Most regularization methods
assume the availability of a \emph{clean} set of auxiliary OOD data. Several works~\citep{zhou2021step,katzsamuels2022training,he2023topological} relaxed this assumption by leveraging the unlabeled wild data, but did not have an explicit mechanism for filtering the outliers. Compared to positive-unlabeled learning, which learns classifiers from positive and unlabeled data~\citep{letouzey2000learning,hsieh2015pu, du2015convex,niu2016theoretical,gong2018margin,chapel2020partial,garg2021mixture,xu2021positive,garg2022domain,zhao2022dist,acharya2022positive}, the key difference is that it only considers the task of distinguishing $\mathbb{P}_\text{out}$ and $\mathbb{P}_\text{in}$, not the task of doing classification simultaneously. Moreover, we propose a new filtering score to separate outliers from the unlabeled data, which has a bounded error guarantee.

\textbf{Robust statistics} has systematically studied the estimation in the presence of outliers since the pioneering work of~\citep{tukey1960survey}. Popular methods include RANSAC~\citep{fischler1981random}, minimum covariance determinant~\citep{rousseeuw1999fast}, Huberizing the loss~\citep{owen2007robust}, removal based on $k$-nearest neighbors~\citep{breunig2000lof}. More recently, there are several works that scale up the  robust estimation into high-dimensions~\citep{awasthi2014power,kothari2017outlier,steinhardt2017does,diakonikolas2019recent,diakonikolas2019robust,diakonikolas2022outlier,diakonikolas2022streaming}. \cite{diakonikolas2019sever} designed a gradient-based score for outlier removal but they focused on the error bound for the ID classifier. Instead, we provide new theoretical guarantees on outlier filtering (Theorem~\ref{MainT-1} and Theorem~\ref{The-1.1}) and the generalization bound of OOD detection (Theorem~\ref{the:main2}).

\vspace{-1em}

\section{Conclusion}
\vspace{-1em}
In this paper, we propose a novel learning framework \model that exploits the unlabeled in-the-wild data for OOD detection. \model first explicitly filters the candidate outliers from the wild data using a new filtering score and then trains a binary OOD classifier leveraging the filtered outliers. Theoretically, \model answers the question of \emph{how does unlabeled wild data help OOD detection} by analyzing the separability of the outliers in the wild and the learnability of the OOD classifier, which provide provable error guarantees  for the two integral components. Empirically, \model achieves strong performance compared to competitive baselines, echoing our theoretical insights.  \textcolor{black}{A broad impact statement is included in Appendix~\ref{sec:broader}}. We hope our work will inspire future research on OOD detection with unlabeled wild data.

\section*{Acknowledgement}
We thank Yifei Ming and Yiyou Sun for their valuable suggestions on the draft. The authors would also like to thank
ICLR anonymous reviewers for their helpful feedback. Du is supported by the Jane Street Graduate
Research Fellowship. Li gratefully acknowledges the support from the AFOSR
Young Investigator Program under award number FA9550-23-1-0184, National Science Foundation
(NSF) Award No. IIS-2237037 \& IIS-2331669, Office of Naval Research under grant number
N00014-23-1-2643, Philanthropic Fund from SFF, and faculty research awards/gifts from Google
and Meta.


\bibliography{citation}

\begin{thebibliography}{104}
\providecommand{\natexlab}[1]{#1}
\providecommand{\url}[1]{\texttt{#1}}
\expandafter\ifx\csname urlstyle\endcsname\relax
  \providecommand{\doi}[1]{doi: #1}\else
  \providecommand{\doi}{doi: \begingroup \urlstyle{rm}\Url}\fi

\bibitem[Acharya et~al.(2022)Acharya, Sanghavi, Jing, Bhushanam, Choudhary, Rabbat, and Dhillon]{acharya2022positive}
Anish Acharya, Sujay Sanghavi, Li~Jing, Bhargav Bhushanam, Dhruv Choudhary, Michael Rabbat, and Inderjit Dhillon.
\newblock Positive unlabeled contrastive learning.
\newblock \emph{arXiv preprint arXiv:2206.01206}, 2022.

\bibitem[Awasthi et~al.(2014)Awasthi, Balcan, and Long]{awasthi2014power}
Pranjal Awasthi, Maria~Florina Balcan, and Philip~M Long.
\newblock The power of localization for efficiently learning linear separators with noise.
\newblock In \emph{Proceedings of the forty-sixth annual ACM symposium on Theory of computing}, pp.\  449--458, 2014.

\bibitem[Bai et~al.(2023)Bai, Canal, Du, Kwon, Nowak, and Li]{bai2023feed}
Haoyue Bai, Gregory Canal, Xuefeng Du, Jeongyeol Kwon, Robert~D Nowak, and Yixuan Li.
\newblock Feed two birds with one scone: Exploiting wild data for both out-of-distribution generalization and detection.
\newblock In \emph{International Conference on Machine Learning}, 2023.

\bibitem[Bartlett et~al.(2020)Bartlett, Long, Lugosi, and Tsigler]{bartlett2020benign}
Peter~L Bartlett, Philip~M Long, G{\'a}bor Lugosi, and Alexander Tsigler.
\newblock Benign overfitting in linear regression.
\newblock \emph{Proceedings of the National Academy of Sciences}, 117\penalty0 (48):\penalty0 30063--30070, 2020.

\bibitem[Bendale \& Boult(2016)Bendale and Boult]{bendale2016towards}
Abhijit Bendale and Terrance~E Boult.
\newblock Towards open set deep networks.
\newblock In \emph{Proceedings of the IEEE/CVF Conference on Computer Vision and Pattern Recognition}, pp.\  1563--1572, 2016.

\bibitem[Bevandi{\'c} et~al.(2018)Bevandi{\'c}, Kre{\v{s}}o, Or{\v{s}}i{\'c}, and {\v{S}}egvi{\'c}]{bevandic2018discriminative}
Petra Bevandi{\'c}, Ivan Kre{\v{s}}o, Marin Or{\v{s}}i{\'c}, and Sini{\v{s}}a {\v{S}}egvi{\'c}.
\newblock Discriminative out-of-distribution detection for semantic segmentation.
\newblock \emph{arXiv preprint arXiv:1808.07703}, 2018.

\bibitem[Breunig et~al.(2000)Breunig, Kriegel, Ng, and Sander]{breunig2000lof}
Markus~M Breunig, Hans-Peter Kriegel, Raymond~T Ng, and J{\"o}rg Sander.
\newblock Lof: identifying density-based local outliers.
\newblock In \emph{Proceedings of the 2000 ACM SIGMOD international conference on Management of data}, pp.\  93--104, 2000.

\bibitem[Chapel et~al.(2020)Chapel, Alaya, and Gasso]{chapel2020partial}
Laetitia Chapel, Mokhtar~Z Alaya, and Gilles Gasso.
\newblock Partial optimal tranport with applications on positive-unlabeled learning.
\newblock \emph{Advances in Neural Information Processing Systems}, 33:\penalty0 2903--2913, 2020.

\bibitem[Cimpoi et~al.(2014)Cimpoi, Maji, Kokkinos, Mohamed, and Vedaldi]{cimpoi2014describing}
Mircea Cimpoi, Subhransu Maji, Iasonas Kokkinos, Sammy Mohamed, and Andrea Vedaldi.
\newblock Describing textures in the wild.
\newblock In \emph{Proceedings of the IEEE/CVF Conference on Computer Vision and Pattern Recognition}, pp.\  3606--3613, 2014.

\bibitem[Diakonikolas \& Kane(2019)Diakonikolas and Kane]{diakonikolas2019recent}
Ilias Diakonikolas and Daniel~M Kane.
\newblock Recent advances in algorithmic high-dimensional robust statistics.
\newblock \emph{arXiv preprint arXiv:1911.05911}, 2019.

\bibitem[Diakonikolas et~al.(2019{\natexlab{a}})Diakonikolas, Kamath, Kane, Li, Moitra, and Stewart]{diakonikolas2019robust}
Ilias Diakonikolas, Gautam Kamath, Daniel Kane, Jerry Li, Ankur Moitra, and Alistair Stewart.
\newblock Robust estimators in high-dimensions without the computational intractability.
\newblock \emph{SIAM Journal on Computing}, 48\penalty0 (2):\penalty0 742--864, 2019{\natexlab{a}}.

\bibitem[Diakonikolas et~al.(2019{\natexlab{b}})Diakonikolas, Kamath, Kane, Li, Steinhardt, and Stewart]{diakonikolas2019sever}
Ilias Diakonikolas, Gautam Kamath, Daniel Kane, Jerry Li, Jacob Steinhardt, and Alistair Stewart.
\newblock Sever: A robust meta-algorithm for stochastic optimization.
\newblock In \emph{International Conference on Machine Learning}, pp.\  1596--1606, 2019{\natexlab{b}}.

\bibitem[Diakonikolas et~al.(2022{\natexlab{a}})Diakonikolas, Kane, Lee, and Pensia]{diakonikolas2022outlier}
Ilias Diakonikolas, Daniel Kane, Jasper Lee, and Ankit Pensia.
\newblock Outlier-robust sparse mean estimation for heavy-tailed distributions.
\newblock \emph{Advances in Neural Information Processing Systems}, 35:\penalty0 5164--5177, 2022{\natexlab{a}}.

\bibitem[Diakonikolas et~al.(2022{\natexlab{b}})Diakonikolas, Kane, Pensia, and Pittas]{diakonikolas2022streaming}
Ilias Diakonikolas, Daniel~M Kane, Ankit Pensia, and Thanasis Pittas.
\newblock Streaming algorithms for high-dimensional robust statistics.
\newblock In \emph{International Conference on Machine Learning}, pp.\  5061--5117, 2022{\natexlab{b}}.

\bibitem[Djurisic et~al.(2023)Djurisic, Bozanic, Ashok, and Liu]{djurisic2023extremely}
Andrija Djurisic, Nebojsa Bozanic, Arjun Ashok, and Rosanne Liu.
\newblock Extremely simple activation shaping for out-of-distribution detection.
\newblock In \emph{International Conference on Learning Representations}, 2023.

\bibitem[Du et~al.(2022{\natexlab{a}})Du, Gozum, Ming, and Li]{du2022siren}
Xuefeng Du, Gabriel Gozum, Yifei Ming, and Yixuan Li.
\newblock Siren: Shaping representations for detecting out-of-distribution objects.
\newblock In \emph{Advances in Neural Information Processing Systems}, 2022{\natexlab{a}}.

\bibitem[Du et~al.(2022{\natexlab{b}})Du, Wang, Gozum, and Li]{du2022unknown}
Xuefeng Du, Xin Wang, Gabriel Gozum, and Yixuan Li.
\newblock Unknown-aware object detection: Learning what you don’t know from videos in the wild.
\newblock In \emph{Proceedings of the IEEE/CVF Conference on Computer Vision and Pattern Recognition}, 2022{\natexlab{b}}.

\bibitem[Du et~al.(2022{\natexlab{c}})Du, Wang, Cai, and Li]{du2022towards}
Xuefeng Du, Zhaoning Wang, Mu~Cai, and Yixuan Li.
\newblock Vos: Learning what you don’t know by virtual outlier synthesis.
\newblock In \emph{Proceedings of the International Conference on Learning Representations}, 2022{\natexlab{c}}.

\bibitem[Du et~al.(2023)Du, Sun, Zhu, and Li]{du2023dream}
Xuefeng Du, Yiyou Sun, Xiaojin Zhu, and Yixuan Li.
\newblock Dream the impossible: Outlier imagination with diffusion models.
\newblock In \emph{Advances in Neural Information Processing Systems}, 2023.

\bibitem[Du~Plessis et~al.(2015)Du~Plessis, Niu, and Sugiyama]{du2015convex}
Marthinus Du~Plessis, Gang Niu, and Masashi Sugiyama.
\newblock Convex formulation for learning from positive and unlabeled data.
\newblock In \emph{International conference on machine learning}, pp.\  1386--1394, 2015.

\bibitem[Fang et~al.(2022)Fang, Li, Lu, Dong, Han, and Liu]{fang2022learnable}
Zhen Fang, Yixuan Li, Jie Lu, Jiahua Dong, Bo~Han, and Feng Liu.
\newblock Is out-of-distribution detection learnable?
\newblock In \emph{Advances in Neural Information Processing Systems}, 2022.

\bibitem[Fischler \& Bolles(1981)Fischler and Bolles]{fischler1981random}
Martin~A Fischler and Robert~C Bolles.
\newblock Random sample consensus: a paradigm for model fitting with applications to image analysis and automated cartography.
\newblock \emph{Communications of the ACM}, 24\penalty0 (6):\penalty0 381--395, 1981.

\bibitem[Fort et~al.(2021)Fort, Ren, and Lakshminarayanan]{fort2021exploring}
Stanislav Fort, Jie Ren, and Balaji Lakshminarayanan.
\newblock Exploring the limits of out-of-distribution detection.
\newblock \emph{Advances in Neural Information Processing Systems}, 34:\penalty0 7068--7081, 2021.

\bibitem[Frei et~al.(2022)Frei, Chatterji, and Bartlett]{frei2022benign}
Spencer Frei, Niladri~S Chatterji, and Peter Bartlett.
\newblock Benign overfitting without linearity: Neural network classifiers trained by gradient descent for noisy linear data.
\newblock In \emph{Conference on Learning Theory}, pp.\  2668--2703, 2022.

\bibitem[Gal \& Ghahramani(2016)Gal and Ghahramani]{gal2016dropout}
Yarin Gal and Zoubin Ghahramani.
\newblock Dropout as a bayesian approximation: Representing model uncertainty in deep learning.
\newblock In \emph{Proceedings of the International Conference on Machine Learning}, pp.\  1050--1059, 2016.

\bibitem[Galil et~al.(2023)Galil, Dabbah, and El-Yaniv]{galil2023a}
Ido Galil, Mohammed Dabbah, and Ran El-Yaniv.
\newblock A framework for benchmarking class-out-of-distribution detection and its application to imagenet.
\newblock In \emph{International Conference on Learning Representations}, 2023.

\bibitem[Garg et~al.(2021)Garg, Wu, Smola, Balakrishnan, and Lipton]{garg2021mixture}
Saurabh Garg, Yifan Wu, Alexander~J Smola, Sivaraman Balakrishnan, and Zachary Lipton.
\newblock Mixture proportion estimation and pu learning: a modern approach.
\newblock \emph{Advances in Neural Information Processing Systems}, 34:\penalty0 8532--8544, 2021.

\bibitem[Garg et~al.(2022)Garg, Balakrishnan, and Lipton]{garg2022domain}
Saurabh Garg, Sivaraman Balakrishnan, and Zachary Lipton.
\newblock Domain adaptation under open set label shift.
\newblock \emph{Advances in Neural Information Processing Systems}, 35:\penalty0 22531--22546, 2022.

\bibitem[Geifman \& El-Yaniv(2019)Geifman and El-Yaniv]{geifman2019selectivenet}
Yonatan Geifman and Ran El-Yaniv.
\newblock Selectivenet: A deep neural network with an integrated reject option.
\newblock In \emph{Proceedings of the International Conference on Machine Learning}, pp.\  2151--2159, 2019.

\bibitem[Ghosal et~al.(2024)Ghosal, Sun, and Li]{ghosal2024how}
Soumya~Suvra Ghosal, Yiyou Sun, and Yixuan Li.
\newblock How to overcome curse-of-dimensionality for ood detection?
\newblock In \emph{Proceedings of the AAAI Conference on Artificial Intelligence}, 2024.

\bibitem[Gong et~al.(2018)Gong, Wang, Ye, Xu, and Lin]{gong2018margin}
Tieliang Gong, Guangtao Wang, Jieping Ye, Zongben Xu, and Ming Lin.
\newblock Margin based pu learning.
\newblock In \emph{Proceedings of the AAAI Conference on Artificial Intelligence}, volume~32, 2018.

\bibitem[Gu et~al.(2023)Gu, Ming, Zhou, Kuen, Morariu, Liu, Li, Sun, and Nenkova]{gu2023critical}
Jiuxiang Gu, Yifei Ming, Yi~Zhou, Jason Kuen, Vlad Morariu, Anqi Liu, Yixuan Li, Tong Sun, and Ani Nenkova.
\newblock A critical analysis of out-of-distribution detection for document understanding.
\newblock In \emph{EMNLP-Findings}, 2023.

\bibitem[He et~al.(2023)He, Li, Han, Yang, and Yin]{he2023topological}
Rundong He, Rongxue Li, Zhongyi Han, Xihong Yang, and Yilong Yin.
\newblock Topological structure learning for weakly-supervised out-of-distribution detection.
\newblock In \emph{Proceedings of the 31st ACM International Conference on Multimedia}, pp.\  4858--4866, 2023.

\bibitem[Hein et~al.(2019)Hein, Andriushchenko, and Bitterwolf]{hein2019relu}
Matthias Hein, Maksym Andriushchenko, and Julian Bitterwolf.
\newblock Why relu networks yield high-confidence predictions far away from the training data and how to mitigate the problem.
\newblock In \emph{Proceedings of the IEEE/CVF Conference on Computer Vision and Pattern Recognition}, pp.\  41--50, 2019.

\bibitem[Hendrycks \& Gimpel(2017)Hendrycks and Gimpel]{hendrycks2016baseline}
Dan Hendrycks and Kevin Gimpel.
\newblock A baseline for detecting misclassified and out-of-distribution examples in neural networks.
\newblock \emph{Proceedings of the International Conference on Learning Representations}, 2017.

\bibitem[Hendrycks et~al.(2019)Hendrycks, Mazeika, and Dietterich]{hendrycks2018deep}
Dan Hendrycks, Mantas Mazeika, and Thomas Dietterich.
\newblock Deep anomaly detection with outlier exposure.
\newblock In \emph{Proceedings of the International Conference on Learning Representations}, 2019.

\bibitem[Hotelling(1933)]{hotelling1933analysis}
Harold Hotelling.
\newblock Analysis of a complex of statistical variables into principal components.
\newblock \emph{Journal of educational psychology}, 24\penalty0 (6):\penalty0 417, 1933.

\bibitem[Hsieh et~al.(2015)Hsieh, Natarajan, and Dhillon]{hsieh2015pu}
Cho-Jui Hsieh, Nagarajan Natarajan, and Inderjit Dhillon.
\newblock Pu learning for matrix completion.
\newblock In \emph{International conference on machine learning}, pp.\  2445--2453, 2015.

\bibitem[Huang et~al.(2021)Huang, Geng, and Li]{huang2021importance}
Rui Huang, Andrew Geng, and Yixuan Li.
\newblock On the importance of gradients for detecting distributional shifts in the wild.
\newblock In \emph{Advances in Neural Information Processing Systems}, 2021.

\bibitem[Jeong \& Kim(2020)Jeong and Kim]{DBLP:conf/nips/JeongK20}
Taewon Jeong and Heeyoung Kim.
\newblock Ood-maml: Meta-learning for few-shot out-of-distribution detection and classification.
\newblock \emph{Advances in Neural Information Processing Systems}, 33:\penalty0 3907--3916, 2020.

\bibitem[Katz-Samuels et~al.(2022)Katz-Samuels, Nakhleh, Nowak, and Li]{katzsamuels2022training}
Julian Katz-Samuels, Julia Nakhleh, Robert Nowak, and Yixuan Li.
\newblock Training ood detectors in their natural habitats.
\newblock In \emph{International Conference on Machine Learning}, 2022.

\bibitem[Kothari \& Steurer(2017)Kothari and Steurer]{kothari2017outlier}
Pravesh~K Kothari and David Steurer.
\newblock Outlier-robust moment-estimation via sum-of-squares.
\newblock \emph{arXiv preprint arXiv:1711.11581}, 2017.

\bibitem[Kristiadi et~al.(2020)Kristiadi, Hein, and Hennig]{kristiadi2020being}
Agustinus Kristiadi, Matthias Hein, and Philipp Hennig.
\newblock Being bayesian, even just a bit, fixes overconfidence in relu networks.
\newblock In \emph{International conference on machine learning}, pp.\  5436--5446, 2020.

\bibitem[Krizhevsky et~al.(2009)Krizhevsky, Hinton, et~al.]{krizhevsky2009learning}
Alex Krizhevsky, Geoffrey Hinton, et~al.
\newblock Learning multiple layers of features from tiny images.
\newblock 2009.

\bibitem[Lakshminarayanan et~al.(2017)Lakshminarayanan, Pritzel, and Blundell]{lakshminarayanan2017simple}
Balaji Lakshminarayanan, Alexander Pritzel, and Charles Blundell.
\newblock Simple and scalable predictive uncertainty estimation using deep ensembles.
\newblock In \emph{Advances in Neural Information Processing Systems}, volume~30, pp.\  6402--6413, 2017.

\bibitem[Lee et~al.(2018{\natexlab{a}})Lee, Lee, Lee, and Shin]{lee2018training}
Kimin Lee, Honglak Lee, Kibok Lee, and Jinwoo Shin.
\newblock Training confidence-calibrated classifiers for detecting out-of-distribution samples.
\newblock In \emph{Proceedings of the International Conference on Learning Representations}, 2018{\natexlab{a}}.

\bibitem[Lee et~al.(2018{\natexlab{b}})Lee, Lee, Lee, and Shin]{lee2018simple}
Kimin Lee, Kibok Lee, Honglak Lee, and Jinwoo Shin.
\newblock A simple unified framework for detecting out-of-distribution samples and adversarial attacks.
\newblock \emph{Advances in Neural Information Processing Systems}, 31, 2018{\natexlab{b}}.

\bibitem[Lei \& Ying(2021)Lei and Ying]{lei2021sharper}
Yunwen Lei and Yiming Ying.
\newblock Sharper generalization bounds for learning with gradient-dominated objective functions.
\newblock In \emph{International Conference on Learning Representations}, 2021.

\bibitem[Letouzey et~al.(2000)Letouzey, Denis, and Gilleron]{letouzey2000learning}
Fabien Letouzey, Fran{\c{c}}ois Denis, and R{\'e}mi Gilleron.
\newblock Learning from positive and unlabeled examples.
\newblock In \emph{International Conference on Algorithmic Learning Theory}, pp.\  71--85. Springer, 2000.

\bibitem[Liang et~al.(2018)Liang, Li, and Srikant]{liang2018enhancing}
Shiyu Liang, Yixuan Li, and Rayadurgam Srikant.
\newblock Enhancing the reliability of out-of-distribution image detection in neural networks.
\newblock In \emph{Proceedings of the International Conference on Learning Representations}, 2018.

\bibitem[Liu et~al.(2020{\natexlab{a}})Liu, Lin, Padhy, Tran, Bedrax~Weiss, and Lakshminarayanan]{liu2020simple}
Jeremiah Liu, Zi~Lin, Shreyas Padhy, Dustin Tran, Tania Bedrax~Weiss, and Balaji Lakshminarayanan.
\newblock Simple and principled uncertainty estimation with deterministic deep learning via distance awareness.
\newblock \emph{Advances in Neural Information Processing Systems}, 33:\penalty0 7498--7512, 2020{\natexlab{a}}.

\bibitem[Liu et~al.(2020{\natexlab{b}})Liu, Wang, Owens, and Li]{liu2020energy}
Weitang Liu, Xiaoyun Wang, John Owens, and Yixuan Li.
\newblock Energy-based out-of-distribution detection.
\newblock \emph{Advances in Neural Information Processing Systems}, 33:\penalty0 21464--21475, 2020{\natexlab{b}}.

\bibitem[Maddox et~al.(2019)Maddox, Izmailov, Garipov, Vetrov, and Wilson]{maddox2019simple}
Wesley~J Maddox, Pavel Izmailov, Timur Garipov, Dmitry~P Vetrov, and Andrew~Gordon Wilson.
\newblock A simple baseline for bayesian uncertainty in deep learning.
\newblock \emph{Advances in Neural Information Processing Systems}, 32:\penalty0 13153--13164, 2019.

\bibitem[Malinin \& Gales(2018)Malinin and Gales]{malinin2018predictive}
Andrey Malinin and Mark Gales.
\newblock Predictive uncertainty estimation via prior networks.
\newblock \emph{Advances in Neural Information Processing Systems}, 31, 2018.

\bibitem[Malinin \& Gales(2019)Malinin and Gales]{dpn19nips}
Andrey Malinin and Mark Gales.
\newblock Reverse kl-divergence training of prior networks: Improved uncertainty and adversarial robustness.
\newblock In \emph{Advances in Neural Information Processing Systems}, 2019.

\bibitem[Meinke \& Hein(2020)Meinke and Hein]{meinke2019towards}
Alexander Meinke and Matthias Hein.
\newblock Towards neural networks that provably know when they don't know.
\newblock In \emph{Proceedings of the International Conference on Learning Representations}, 2020.

\bibitem[Ming \& Li(2023)Ming and Li]{ming2023finetune}
Yifei Ming and Yixuan Li.
\newblock How does fine-tuning impact out-of-distribution detection for vision-language models?
\newblock \emph{International Journal of Computer Vision}, 2023.

\bibitem[Ming et~al.(2022{\natexlab{a}})Ming, Cai, Gu, Sun, Li, and Li]{ming2022delving}
Yifei Ming, Ziyang Cai, Jiuxiang Gu, Yiyou Sun, Wei Li, and Yixuan Li.
\newblock Delving into out-of-distribution detection with vision-language representations.
\newblock In \emph{Advances in Neural Information Processing Systems}, 2022{\natexlab{a}}.

\bibitem[Ming et~al.(2022{\natexlab{b}})Ming, Fan, and Li]{DBLP:conf/icml/MingFL22}
Yifei Ming, Ying Fan, and Yixuan Li.
\newblock {POEM:} out-of-distribution detection with posterior sampling.
\newblock In \emph{Proceedings of the International Conference on Machine Learning}, pp.\  15650--15665, 2022{\natexlab{b}}.

\bibitem[Ming et~al.(2022{\natexlab{c}})Ming, Yin, and Li]{ming2022spurious}
Yifei Ming, Hang Yin, and Yixuan Li.
\newblock On the impact of spurious correlation for out-of-distribution detection.
\newblock In \emph{Proceedings of the AAAI Conference on Artificial Intelligence}, 2022{\natexlab{c}}.

\bibitem[Ming et~al.(2023)Ming, Sun, Dia, and Li]{ming2023cider}
Yifei Ming, Yiyou Sun, Ousmane Dia, and Yixuan Li.
\newblock How to exploit hyperspherical embeddings for out-of-distribution detection?
\newblock In \emph{Proceedings of the International Conference on Learning Representations}, 2023.

\bibitem[Narasimhan et~al.(2023)Narasimhan, Menon, Jitkrittum, and Kumar]{narasimhan2023learning}
Harikrishna Narasimhan, Aditya~Krishna Menon, Wittawat Jitkrittum, and Sanjiv Kumar.
\newblock Learning to reject meets ood detection: Are all abstentions created equal?
\newblock \emph{arXiv preprint arXiv:2301.12386}, 2023.

\bibitem[Netzer et~al.(2011)Netzer, Wang, Coates, Bissacco, Wu, and Ng]{netzer2011reading}
Yuval Netzer, Tao Wang, Adam Coates, Alessandro Bissacco, Bo~Wu, and Andrew~Y Ng.
\newblock Reading digits in natural images with unsupervised feature learning.
\newblock 2011.

\bibitem[Nguyen et~al.(2015)Nguyen, Yosinski, and Clune]{nguyen2015deep}
Anh Nguyen, Jason Yosinski, and Jeff Clune.
\newblock Deep neural networks are easily fooled: High confidence predictions for unrecognizable images.
\newblock In \emph{Proceedings of the IEEE/CVF Conference on Computer Vision and Pattern Recognition}, pp.\  427--436, 2015.

\bibitem[Niu et~al.(2016)Niu, du~Plessis, Sakai, Ma, and Sugiyama]{niu2016theoretical}
Gang Niu, Marthinus~Christoffel du~Plessis, Tomoya Sakai, Yao Ma, and Masashi Sugiyama.
\newblock Theoretical comparisons of positive-unlabeled learning against positive-negative learning.
\newblock \emph{Advances in neural information processing systems}, 29, 2016.

\bibitem[Owen(2007)]{owen2007robust}
Art~B Owen.
\newblock A robust hybrid of lasso and ridge regression.
\newblock \emph{Contemporary Mathematics}, 443\penalty0 (7):\penalty0 59--72, 2007.

\bibitem[Ren et~al.(2021)Ren, Fort, Liu, Roy, Padhy, and Lakshminarayanan]{DBLP:journals/corr/abs-2106-09022}
Jie Ren, Stanislav Fort, Jeremiah Liu, Abhijit~Guha Roy, Shreyas Padhy, and Balaji Lakshminarayanan.
\newblock A simple fix to mahalanobis distance for improving near-ood detection.
\newblock \emph{CoRR}, abs/2106.09022, 2021.

\bibitem[Ren et~al.(2023)Ren, Luo, Zhao, Krishna, Saleh, Lakshminarayanan, and Liu]{ren2023outofdistribution}
Jie Ren, Jiaming Luo, Yao Zhao, Kundan Krishna, Mohammad Saleh, Balaji Lakshminarayanan, and Peter~J Liu.
\newblock Out-of-distribution detection and selective generation for conditional language models.
\newblock In \emph{International Conference on Learning Representations}, 2023.

\bibitem[Rousseeuw \& Driessen(1999)Rousseeuw and Driessen]{rousseeuw1999fast}
Peter~J Rousseeuw and Katrien~Van Driessen.
\newblock A fast algorithm for the minimum covariance determinant estimator.
\newblock \emph{Technometrics}, 41\penalty0 (3):\penalty0 212--223, 1999.

\bibitem[Sehwag et~al.(2021)Sehwag, Chiang, and Mittal]{2021ssd}
Vikash Sehwag, Mung Chiang, and Prateek Mittal.
\newblock Ssd: A unified framework for self-supervised outlier detection.
\newblock In \emph{International Conference on Learning Representations}, 2021.

\bibitem[Shalev-Shwartz \& Ben-David(2014)Shalev-Shwartz and Ben-David]{understand_ml}
Shai Shalev-Shwartz and Shai Ben-David.
\newblock \emph{Understanding Machine Learning: From Theory to Algorithms}.
\newblock Cambridge University Press, 2014.
\newblock ISBN 1107057132.

\bibitem[Steinhardt(2017)]{steinhardt2017does}
Jacob Steinhardt.
\newblock Does robustness imply tractability? a lower bound for planted clique in the semi-random model.
\newblock \emph{arXiv preprint arXiv:1704.05120}, 2017.

\bibitem[Sun \& Li(2022)Sun and Li]{sun2022dice}
Yiyou Sun and Yixuan Li.
\newblock Dice: Leveraging sparsification for out-of-distribution detection.
\newblock In \emph{Proceedings of European Conference on Computer Vision}, 2022.

\bibitem[Sun et~al.(2021)Sun, Guo, and Li]{sun2021react}
Yiyou Sun, Chuan Guo, and Yixuan Li.
\newblock React: Out-of-distribution detection with rectified activations.
\newblock In \emph{Advances in Neural Information Processing Systems}, volume~34, 2021.

\bibitem[Sun et~al.(2022)Sun, Ming, Zhu, and Li]{sun2022out}
Yiyou Sun, Yifei Ming, Xiaojin Zhu, and Yixuan Li.
\newblock Out-of-distribution detection with deep nearest neighbors.
\newblock In \emph{Proceedings of the International Conference on Machine Learning}, pp.\  20827--20840, 2022.

\bibitem[Tack et~al.(2020)Tack, Mo, Jeong, and Shin]{tack2020csi}
Jihoon Tack, Sangwoo Mo, Jongheon Jeong, and Jinwoo Shin.
\newblock Csi: Novelty detection via contrastive learning on distributionally shifted instances.
\newblock In \emph{Advances in Neural Information Processing Systems}, 2020.

\bibitem[Tao et~al.(2023)Tao, Du, Zhu, and Li]{tao2023nonparametric}
Leitian Tao, Xuefeng Du, Xiaojin Zhu, and Yixuan Li.
\newblock Non-parametric outlier synthesis.
\newblock In \emph{Proceedings of the International Conference on Learning Representations}, 2023.

\bibitem[Tukey(1960)]{tukey1960survey}
John~Wilder Tukey.
\newblock A survey of sampling from contaminated distributions.
\newblock \emph{Contributions to probability and statistics}, pp.\  448--485, 1960.

\bibitem[Uppaal et~al.(2023)Uppaal, Hu, and Li]{uppaal2023fine}
Rheeya Uppaal, Junjie Hu, and Yixuan Li.
\newblock Is fine-tuning needed? pre-trained language models are near perfect for out-of-domain detection.
\newblock In \emph{Annual Meeting of the Association for Computational Linguistics}, 2023.

\bibitem[van Amersfoort et~al.(2020)van Amersfoort, Smith, Teh, and Gal]{DBLP:conf/icml/AmersfoortSTG20}
Joost van Amersfoort, Lewis Smith, Yee~Whye Teh, and Yarin Gal.
\newblock Uncertainty estimation using a single deep deterministic neural network.
\newblock In \emph{Proceedings of the International Conference on Machine Learning}, pp.\  9690--9700, 2020.

\bibitem[Vershynin(2018)]{Vershynin2018HighDimensionalP}
Roman Vershynin.
\newblock High-dimensional probability.
\newblock 2018.

\bibitem[Wang et~al.(2021)Wang, Liu, Bocchieri, and Li]{wang2021canmulti}
Haoran Wang, Weitang Liu, Alex Bocchieri, and Yixuan Li.
\newblock Can multi-label classification networks know what they don't know?
\newblock \emph{Proceedings of the Advances in Neural Information Processing Systems}, 2021.

\bibitem[Wang et~al.(2022{\natexlab{a}})Wang, Liu, Zhang, Zhang, Gong, Liu, and Han]{wang2022watermarking}
Qizhou Wang, Feng Liu, Yonggang Zhang, Jing Zhang, Chen Gong, Tongliang Liu, and Bo~Han.
\newblock Watermarking for out-of-distribution detection.
\newblock \emph{Advances in Neural Information Processing Systems}, 35:\penalty0 15545--15557, 2022{\natexlab{a}}.

\bibitem[Wang et~al.(2023{\natexlab{a}})Wang, Fang, Zhang, Liu, Li, and Han]{wang2023learning}
Qizhou Wang, Zhen Fang, Yonggang Zhang, Feng Liu, Yixuan Li, and Bo~Han.
\newblock Learning to augment distributions for out-of-distribution detection.
\newblock In \emph{Advances in Neural Information Processing Systems}, 2023{\natexlab{a}}.

\bibitem[Wang et~al.(2023{\natexlab{b}})Wang, Ye, Liu, Dai, Kalander, Liu, HAO, and Han]{wang2023outofdistribution}
Qizhou Wang, Junjie Ye, Feng Liu, Quanyu Dai, Marcus Kalander, Tongliang Liu, Jianye HAO, and Bo~Han.
\newblock Out-of-distribution detection with implicit outlier transformation.
\newblock In \emph{International Conference on Learning Representations}, 2023{\natexlab{b}}.

\bibitem[Wang et~al.(2022{\natexlab{b}})Wang, Zou, Lin, Ling, Pan, Yao, and Mei]{wang2022outofdistribution}
Yu~Wang, Jingjing Zou, Jingyang Lin, Qing Ling, Yingwei Pan, Ting Yao, and Tao Mei.
\newblock Out-of-distribution detection via conditional kernel independence model.
\newblock In \emph{Advances in Neural Information Processing Systems}, 2022{\natexlab{b}}.

\bibitem[Wei et~al.(2022)Wei, Xie, Cheng, Feng, An, and Li]{DBLP:conf/icml/WeiXCF0L22}
Hongxin Wei, Renchunzi Xie, Hao Cheng, Lei Feng, Bo~An, and Yixuan Li.
\newblock Mitigating neural network overconfidence with logit normalization.
\newblock In \emph{Proceedings of the International Conference on Machine Learning}, pp.\  23631--23644, 2022.

\bibitem[Wen et~al.(2020)Wen, Tran, and Ba]{Wen2020BatchEnsemble}
Yeming Wen, Dustin Tran, and Jimmy Ba.
\newblock Batchensemble: an alternative approach to efficient ensemble and lifelong learning.
\newblock In \emph{International Conference on Learning Representations}, 2020.

\bibitem[Wu et~al.(2023)Wu, Chen, Yang, and Yan]{wu2023energybased}
Qitian Wu, Yiting Chen, Chenxiao Yang, and Junchi Yan.
\newblock Energy-based out-of-distribution detection for graph neural networks.
\newblock In \emph{International Conference on Learning Representations}, 2023.

\bibitem[Xu \& Denil(2021)Xu and Denil]{xu2021positive}
Danfei Xu and Misha Denil.
\newblock Positive-unlabeled reward learning.
\newblock In \emph{Conference on Robot Learning}, pp.\  205--219, 2021.

\bibitem[Yang et~al.(2021{\natexlab{a}})Yang, Wang, Feng, Yan, Zheng, Zhang, and Liu]{DBLP:conf/iccv/YangWFYZZ021}
Jingkang Yang, Haoqi Wang, Litong Feng, Xiaopeng Yan, Huabin Zheng, Wayne Zhang, and Ziwei Liu.
\newblock Semantically coherent out-of-distribution detection.
\newblock In \emph{Proceedings of the International Conference on Computer Vision}, pp.\  8281--8289, 2021{\natexlab{a}}.

\bibitem[Yang et~al.(2021{\natexlab{b}})Yang, Zhou, Li, and Liu]{yang2021generalized}
Jingkang Yang, Kaiyang Zhou, Yixuan Li, and Ziwei Liu.
\newblock Generalized out-of-distribution detection: A survey.
\newblock \emph{arXiv preprint arXiv:2110.11334}, 2021{\natexlab{b}}.

\bibitem[Yang et~al.(2022)Yang, Wang, Zou, Zhou, Ding, Peng, Wang, Chen, Li, Sun, et~al.]{yang2022openood}
Jingkang Yang, Pengyun Wang, Dejian Zou, Zitang Zhou, Kunyuan Ding, Wenxuan Peng, Haoqi Wang, Guangyao Chen, Bo~Li, Yiyou Sun, et~al.
\newblock Openood: Benchmarking generalized out-of-distribution detection.
\newblock \emph{Advances in Neural Information Processing Systems}, 35:\penalty0 32598--32611, 2022.

\bibitem[Yang et~al.(2023)Yang, Liang, Cao, and He]{yang2023auto}
Puning Yang, Jian Liang, Jie Cao, and Ran He.
\newblock Auto: Adaptive outlier optimization for online test-time ood detection.
\newblock \emph{arXiv preprint arXiv:2303.12267}, 2023.

\bibitem[Yu et~al.(2015)Yu, Seff, Zhang, Song, Funkhouser, and Xiao]{DBLP:journals/corr/YuZSSX15}
Fisher Yu, Ari Seff, Yinda Zhang, Shuran Song, Thomas Funkhouser, and Jianxiong Xiao.
\newblock Lsun: Construction of a large-scale image dataset using deep learning with humans in the loop.
\newblock \emph{arXiv preprint arXiv:1506.03365}, 2015.

\bibitem[Zagoruyko \& Komodakis(2016)Zagoruyko and Komodakis]{ZagoruykoK16}
Sergey Zagoruyko and Nikos Komodakis.
\newblock Wide residual networks.
\newblock In Richard~C. Wilson, Edwin~R. Hancock, and William A.~P. Smith (eds.), \emph{Proceedings of the British Machine Vision Conference}, 2016.

\bibitem[Zhang et~al.(2023)Zhang, Yang, Wang, Wang, Lin, Zhang, Sun, Du, Zhou, Zhang, et~al.]{zhang2023openood}
Jingyang Zhang, Jingkang Yang, Pengyun Wang, Haoqi Wang, Yueqian Lin, Haoran Zhang, Yiyou Sun, Xuefeng Du, Kaiyang Zhou, Wayne Zhang, et~al.
\newblock Openood v1. 5: Enhanced benchmark for out-of-distribution detection.
\newblock \emph{arXiv preprint arXiv:2306.09301}, 2023.

\bibitem[Zhao et~al.(2022)Zhao, Xu, Jiang, Wen, and Huang]{zhao2022dist}
Yunrui Zhao, Qianqian Xu, Yangbangyan Jiang, Peisong Wen, and Qingming Huang.
\newblock Dist-pu: Positive-unlabeled learning from a label distribution perspective.
\newblock In \emph{Proceedings of the IEEE/CVF Conference on Computer Vision and Pattern Recognition}, pp.\  14461--14470, 2022.

\bibitem[Zheng et~al.(2023)Zheng, Wang, Fang, Xia, Liu, Liu, and Han]{zheng2023out}
Haotian Zheng, Qizhou Wang, Zhen Fang, Xiaobo Xia, Feng Liu, Tongliang Liu, and Bo~Han.
\newblock Out-of-distribution detection learning with unreliable out-of-distribution sources.
\newblock \emph{arXiv preprint arXiv:2311.03236}, 2023.

\bibitem[Zhou et~al.(2017)Zhou, Lapedriza, Khosla, Oliva, and Torralba]{zhou2017places}
Bolei Zhou, Agata Lapedriza, Aditya Khosla, Aude Oliva, and Antonio Torralba.
\newblock Places: A 10 million image database for scene recognition.
\newblock \emph{IEEE transactions on pattern analysis and machine intelligence}, 40\penalty0 (6):\penalty0 1452--1464, 2017.

\bibitem[Zhou et~al.(2021)Zhou, Guo, Cheng, Li, and Pu]{zhou2021step}
Zhi Zhou, Lan-Zhe Guo, Zhanzhan Cheng, Yu-Feng Li, and Shiliang Pu.
\newblock {STEP}: Out-of-distribution detection in the presence of limited in-distribution labeled data.
\newblock In \emph{Advances in Neural Information Processing Systems}, 2021.

\bibitem[Zhu et~al.(2023{\natexlab{a}})Zhu, Li, Yao, Liu, Xu, and Han]{zhu2023unleashing}
Jianing Zhu, Hengzhuang Li, Jiangchao Yao, Tongliang Liu, Jianliang Xu, and Bo~Han.
\newblock Unleashing mask: Explore the intrinsic out-of-distribution detection capability.
\newblock \emph{arXiv preprint arXiv:2306.03715}, 2023{\natexlab{a}}.

\bibitem[Zhu et~al.(2023{\natexlab{b}})Zhu, Yu, Yao, Liu, Niu, Sugiyama, and Han]{zhu2023diversified}
Jianing Zhu, Geng Yu, Jiangchao Yao, Tongliang Liu, Gang Niu, Masashi Sugiyama, and Bo~Han.
\newblock Diversified outlier exposure for out-of-distribution detection via informative extrapolation.
\newblock \emph{arXiv preprint arXiv:2310.13923}, 2023{\natexlab{b}}.

\bibitem[Zhu et~al.(2022)Zhu, Chen, Xie, Li, Zhang, Xue', Tian, bolun zheng, and Chen]{zhu2022boosting}
Yao Zhu, YueFeng Chen, Chuanlong Xie, Xiaodan Li, Rong Zhang, Hui Xue', Xiang Tian, bolun zheng, and Yaowu Chen.
\newblock Boosting out-of-distribution detection with typical features.
\newblock In \emph{Advances in Neural Information Processing Systems}, 2022.

\end{thebibliography}
\bibliographystyle{iclr2024_conference}

\appendix

\clearpage
\begin{center}
    \Large{\textbf{How Does Unlabeled Data Provably \\ Help Out-of-Distribution Detection? (Appendix)}}
\end{center}

\section{Algorithm of \model}
\label{sec:algorithm_block}
We summarize our algorithm in implementation as follows.  
\begin{algorithm}[h]
\SetAlgoLined
\textbf{Input:} In-distribution data $\mathcal{S}^{\text{in}}=\left\{\left(\mathbf{x}_i, y_i\right)\right\}_{i =1}^n$. Unlabeled wild data $\mathcal{S}_{\text{wild}}=\left\{\tilde{\mathbf{x}}_i\right\}_{i=1}^m$. 
$K$-way classification model $\*h_\*w$ and OOD classifier $\*g_{\boldsymbol{\theta}}$. Parameter spaces $\mathcal{W}$ and $\Theta$. Learning rate {\rm lr} for $\*g_{\boldsymbol{\theta}}$. \\
\textbf{Output:} Learned OOD classifier $\*g_{\widehat{\boldsymbol{\theta}}_T}$.\\
\# \textbf{Filtering stage} \\
 1) Perform ERM: $\mathbf{w}_{\mathcal{S}^{\text{in}}} \in \rm argmin_{\mathbf{w}\in \mathcal{W}} R_{\mathcal{S}^{\text{in}}}(\mathbf{h}_\mathbf{w})$.\\
2) Calculate the reference gradient as $\bar{\nabla}=\frac{1}{n} \sum_{(\mathbf{x}_i, y_i) \in \mathcal{S}^{\text{in}}} \nabla \ell (\mathbf{h}_{\mathbf{w}_{\mathcal{S}^{\text{in}}}}(\mathbf{x}_i),y_i)$.\\
 3)  Calculate gradient on $\mathcal{S}_{\text{wild}}$  as $\nabla \ell (\mathbf{h}_{\mathbf{w}_{\mathcal{S}^{\text{in}}}}(\tilde{\mathbf{x}}_i),\widehat{y}_{\tilde{\mathbf{x}}_i})$ and calculate the gradient matrix $\*G$.\\
 4)  Calculate the top singular vector $\*v$ of $\*G$ and the score $\tau_i =\left<\nabla \ell (\mathbf{h}_{\mathbf{w}_{\mathcal{S}^{\text{in}}}}\big(\tilde{\mathbf{x}}_i), \widehat{y}_{\tilde{\mathbf{x}}_i})-\bar{\nabla} , \mathbf{v}\right > ^2$.\\
 5) Get the candidate outliers $\mathcal{S}_T=\{\tilde{\mathbf{x}}_i \in  \mathcal{S}_{\text{wild}}, \tau_i \geq T\}$.\\
\# \textbf{Training Stage} \\
     \For{epoch in epochs}{
   6) Sample batches of data ${\mathcal{B}^{\text{in}},\mathcal{B}_T}$ from ID and candidate outliers ${\mathcal{S}^{\text{in}},\mathcal{S}_T}$.\\
7) Calculate the binary classification loss $R_{\mathcal{B}^{\text{in}},\mathcal{B}_T}(\mathbf{g}_{\boldsymbol{\theta}})$.\\
8) Update the parameter by $\widehat{\boldsymbol{\theta}}_T= \boldsymbol{\theta} - {\rm lr} \cdot \nabla R_{\mathcal{B}^{\text{in}},\mathcal{B}_T}(\mathbf{g}_{\boldsymbol{\theta}})$.}
\caption{\model: Separate And Learn}
   \label{alg:algo}
\end{algorithm}

\section{Notations, Definitions, Assumptions and Important Constants}\label{notation,definition,Ass,Const}
Here we summarize the important notations and constants in Tables~\ref{tab: notation} and \ref{tab: Constants}, restate necessary definitions and assumptions in Sections \ref{sec:definition_app} and \ref{sec:assumption_app}. 
\subsection{Notations}
Please see Table \ref{tab: notation} for detailed notations.
\begin{table}[h]
    \centering
    \caption{Main notations and their descriptions.}
    \begin{tabular}{cl}
    \toprule[1.5pt]
         \multicolumn{1}{c}{Notation} & \multicolumn{1}{c}{Description} \\
    \midrule[1pt]
    \multicolumn{2}{c}{\cellcolor{greyC} Spaces} \\
    $\mathcal{X}$, $\mathcal{Y}$     & the input space and the label space. \\
    $\mathcal{W}$, $\Theta$ & the hypothesis spaces \\

    \multicolumn{2}{c}{\cellcolor{greyC} Distributions} \\
    $\mathbb{P}_{\text{wild}}$, $\mathbb{P}_{{\text{in}}}$, $\mathbb{P}_{{\text{out}}}$& data distribution for wild data, labeled ID data and OOD data
    \\
    $\mathbb{P}_{\mathcal{X}\mathcal{Y}}$ & the joint data distribution for ID data.\\
    
    \multicolumn{2}{c}{\cellcolor{greyC} Data and Models} \\
    $\mathbf{w}$, $\mathbf{x}$, $\mathbf{v}$ & weight/input/the top-1 right singular vector of $G$ \\

    $\widehat{\nabla}$, $\tau$ &  the average gradients on labeled ID data, uncertainty score\\
    $y$ and $y_{\text{b}}$ & label for ID classification and binary label for OOD detection \\
    $\widehat{y}_{\*x}$ &Predicted one-hot label for input $\*x$ \\
     $\mathbf{h}_{\mathbf{w}}$ and $\mathbf{g}_{\boldsymbol{\theta}}$ & predictor on labeled in-distribution and binary predictor for OOD detection \\
      $\mathcal{S}_{{\text{wild}}}^{\text{in}}$, $\mathcal{S}_{{\text{wild}}}^{\text{out}}$ &  inliers and
outliers  in the wild dataset. \\ 
      $\mathcal{S}^{\text{in}}$, $\mathcal{S}_{\text{wild}}$ & labeled ID data and unlabeled wild data \\
      $n$, $m$ & size of $\mathcal{S}^{\text{in}}$, size of $\mathcal{S}_{\text{wild}}$\\
      $T$ & the filtering threshold\\
      $\mathcal{S}_{T}$ & wild data whose uncertainty score higher than threshold $T$\\

    \multicolumn{2}{c}{\cellcolor{greyC} Distances} \\
    $r_1$ and $r_2$ & the radius of the hypothesis spaces $\mathcal{W}$ and $\Theta$, respectively  \\
   
    $\| \cdot \|_2$ & $\ell_2$ norm\\
    
    \multicolumn{2}{c}{\cellcolor{greyC} Loss, Risk and Predictor} \\
    $\ell(\cdot, \cdot)$, $\ell_{\text{b}}(\cdot, \cdot)$ & ID loss function,~binary loss function\\
    $R_{\mathcal{S}}(\mathbf{h}_{\mathbf{w}})$ & the empirical risk w.r.t. predictor $\mathbf{h}_{\mathbf{w}}$ over data $\mathcal{S}$  \\
     $R_{\mathbb{P}_{\mathcal{X}\mathcal{Y}}}(\mathbf{h}_{\mathbf{w}})$ & the  risk w.r.t. predictor $\mathbf{h}_{\mathbf{w}}$ over joint distribution $\mathbb{P}_{\mathcal{X}\mathcal{Y}}$\\$R_{\mathbb{P}_{\text{in}},\mathbb{P}_{\text{out}}}(\mathbf{g}_{{\boldsymbol{\theta}}})$& the risk defined in Eq. \ref{Real-risk}\\
     ${\rm ERR}_{\text{in}}, {\rm ERR}_{\text{out}}$ & the error rates of regarding ID as OOD and OOD as ID\\

    \bottomrule[1.5pt]
    \end{tabular}
    
    \label{tab: notation}
\end{table}

\subsection{Definitions}
\label{sec:definition_app}

\begin{Definition}[$\beta$-smooth]\label{Def::beta-smooth} We say a loss function $\ell(\mathbf{h}_{\mathbf{w}}(\mathbf{x}),y)$ (defined over $\mathcal{X}\times \mathcal{Y}$) is $\beta$-smooth, if for any $\mathbf{x}\in \mathcal{X}$ and $y\in \mathcal{Y}$,
\begin{equation*}
   \big \|\nabla \ell(\mathbf{h}_{\mathbf{w}}(\mathbf{x}),y) - \nabla \ell(\mathbf{h}_{\mathbf{w}'}(\mathbf{x}),y) \big \|_2 \leq \beta \|\mathbf{w}-\mathbf{w}'\|_{2} 
\end{equation*}
\end{Definition}
$~~~~$

\begin{Definition}[Gradient-based Distribution Discrepancy]\label{Def3}
    Given distributions $\mathbb{P}$ and $\mathbb{Q}$ defined over $\mathcal{X}$, the Gradient-based Distribution Discrepancy w.r.t. predictor $\*h_{\mathbf{w}}$ and loss $\ell$ is
    \begin{equation}
        d_{\mathbf{w}}^{\ell}(\mathbb{P},\mathbb{Q}) =\big \| \nabla R_{\mathbb{P}}(\*h_{\mathbf{w}},\widehat{\*h}) -  \nabla R_{\mathbb{Q}}(\*h_{\mathbf{w}},\widehat{\*h}) \big \|_2,
    \end{equation}
    where $\widehat{\*h}$ is a classifier which returns the closest one-hot vector of $\*h_{\mathbf{w}}$, $R_{\mathbb{P}}(\*h_{\mathbf{w}},\widehat{\*h})= \mathbb{E}_{\mathbf{x}\sim \mathbb{P}} \ell(\*h_{\mathbf{w}},\widehat{\*h})$ and $R_{\mathbb{Q}}(\*h_{\mathbf{w}},\widehat{\*h})= \mathbb{E}_{\mathbf{x}\sim \mathbb{Q}} \ell(\*h_{\mathbf{w}},\widehat{\*h})$.
\end{Definition}

\begin{Definition}[$(\gamma,\zeta)$-discrepancy]\label{Def4} We say a wild distribution $\mathbb{P}_{\text{wild}}$ has $(\gamma,\zeta)$-discrepancy w.r.t. an ID joint distribution $\mathbb{P}_{\text{in}}$, if $\gamma > \min_{\mathbf{w}\in \mathcal{W}} R_{\mathbb{P}_{\mathcal{X}\mathcal{Y}}}(\mathbf{h}_{\mathbf{w}})$, and for any parameter $\mathbf{w}\in \mathcal{W}$ satisfying that $R_{\mathbb{P}_{\mathcal{X}\mathcal{Y}}}(\mathbf{h}_{\mathbf{w}})\leq \gamma$ should meet the following condition
\begin{equation*}
d_{\mathbf{w}}^{\ell}(\mathbb{P}_{\text{in}},\mathbb{P}_{\text{wild}}) > \zeta,
\end{equation*}
where $R_{\mathbb{P}_{\mathcal{X}\mathcal{Y}}}(\mathbf{h}_{\mathbf{w}})= \mathbb{E}_{(\mathbf{x},y)\sim \mathbb{P}_{\mathcal{X}\mathcal{Y}}} \ell(\mathbf{h}_{\mathbf{w}}(\mathbf{x}),y)$.
\end{Definition}
In Section~\ref{sec:verification_discrepancy}, we empirically calculate the values of the distribution discrepancy between the ID joint distribution $\mathbb{P}_{\mathcal{X}\mathcal{Y}}$ and the wild distribution $\mathbb{P}_{\text{wild}}$.
$~~~$
\\
\subsection{Assumptions}
\label{sec:assumption_app}
\begin{assumption}\label{Ass1}
$~~~~$
   \begin{itemize}
       \item The parameter space $\mathcal{W}\subset B(\mathbf{w}_0,r_1)\subset \mathbb{R}^d$ ($\ell_2$ ball of radius $r_1$ around $\mathbf{w}_0$);
       \item The parameter space $\Theta \subset B(\boldsymbol{\theta}_0, r_2)\subset \mathbb{R}^{d'}$ ($\ell_2$ ball of radius $r_2$ around $\boldsymbol{\theta}_0$);
        \item $\ell(\mathbf{h}_{\mathbf{w}}(\mathbf{x}),y) \geq 0$ and $\ell(\mathbf{h}_{\mathbf{w}}(\mathbf{x}),y)$ is $\beta_1$-smooth;
        \item $\ell_{\text{b}}(\mathbf{g}_{\boldsymbol{\theta}}(\mathbf{x}),y_{\text{b}}) \geq 0$ and $\ell_{\text{b}}(\mathbf{g}_{\boldsymbol{\theta}}(\mathbf{x}),y_{\text{b}})$ is $\beta_2$-smooth;
        \item  $\sup_{(\mathbf{x},y)\in \mathcal{X}\times \mathcal{Y}} \|\nabla \ell(\mathbf{h}_{\mathbf{w}_0}(\mathbf{x}),y)\|_2=b_1$, $\sup_{(\mathbf{x},y_{\text{b}})\in \mathcal{X} \times \mathcal{Y}_{\text{b}}} \|\nabla \ell(\mathbf{g}_{\boldsymbol{\theta}_0}(\mathbf{x}),y_{\text{b}})\|_2=b_2$;
         \item  $\sup_{(\mathbf{x},y)\in \mathcal{X}\times \mathcal{Y}}  \ell(\mathbf{h}_{\mathbf{w}_0}(\mathbf{x}),y)=B_1$, $\sup_{(\mathbf{x},y_{\text{b}})\in \mathcal{X} \times \mathcal{Y}_{\text{b}}}  \ell(\mathbf{g}_{\boldsymbol{\theta}_0}(\mathbf{x}),y_{\text{b}})=B_2$.
   \end{itemize} 
\end{assumption}

\begin{Remark} For neural networks with smooth activation functions and softmax output function, we can check that the norm of the second derivative of the loss functions (cross-entropy loss and  sigmoid loss) is bounded given the bounded parameter space, which implies that the $\beta$-smoothness of the loss functions can hold true.  Therefore,  our assumptions are reasonable in practice.
\end{Remark}
\vspace{0.5cm}
\begin{assumption}\label{Ass2}
  $\ell(\*h(\*x),\widehat{y}_{\*x})\leq \min_{y\in \mathcal{Y}} \ell(\*h(\mathbf{x}),{y}),$ where  $\widehat{y}_{\*x}$ returns the closest one-hot label of the predictor $\*h$'s output on $\*x$.
\end{assumption}

\begin{Remark}
The assumption means the loss incurred by using the predicted labels given by the classifier itself is smaller or equal to the loss incurred by using any label in the label space.
If $y=\widehat{y}_{\*x}$, the assumption is satisfied obviously. If $y\neq \widehat{y}_{\*x}$, then we provide two examples to illustrate the validity of the assumption. For example, (1) if the loss $\ell$ is the cross entropy loss, let $K=2, \*h(\*x)=[h_1, h_2]$ (classification output after softmax) and $h_1 > h_2$. Therefore, we have $\widehat{y}_{\*x}= 0$. Suppose $y=1$, we can get $\ell(\*h(\*x), \widehat{y}_{\*x}) = -\operatorname{log} (h_1) <   \ell(\*h(\*x), y) = -\operatorname{log}(h_2)$. (2) If $\ell$ is the hinge loss for binary classification, thus we have $K=1$, let $\*h(\*x)=h_1<0$ and thus $\widehat{y}_{\*x}=-1$. Suppose $y=1$, we can get $\ell(\*h(\*x), \widehat{y}_{\*x}) =\operatorname{max}(0,1+h_1) <\operatorname{max}(0,1-h_1) = \ell(\*h(\*x), y) .$
\end{Remark}
\subsection{Constants in Theory}
 \vspace{-0.5cm}
\begin{table}[h]
    \centering
    \caption{Constants in theory.}
    \begin{tabular}{cl}
    \toprule[1.5pt]
         \multicolumn{1}{c}{Constants} & \multicolumn{1}{c}{Description}\\
    \midrule[2pt]
    $M = \beta_1 r_1^2+b_1r_1+B_1$ &the upper bound of loss $\ell(\mathbf{h}_{\mathbf{w}}(\mathbf{x}),y)$, see Proposition \ref{P1}
    \\
    $M' = 2(\beta_1 r_1+b_1)^2$&the upper bound of filtering score $\tau$ \\
    $\tilde{M}=\beta_1 M$&   a constant for simplified representation \\
 $L= \beta_2 r_2^2+b_2r_2+B_2$&the upper bound of loss $\ell_{\text{b}}(\mathbf{g}_{\boldsymbol{\theta}}(\mathbf{x}),y_{\text{b}})$, see Proposition \ref{P1}\\
     $d$, $d'$ &the dimensions of parameter spaces $\mathcal{W}$ and $\Theta$, respectively\\
   $R^{*}_{{\text{in}}}$ & the optimal ID risk, i.e., $R^{*}_{{\text{in}}}=\min_{\mathbf{w}\in \mathcal{W}} \mathbb{E}_{(\mathbf{x},y)\sim \mathbb{P}_{\mathcal{X}\mathcal{Y}}} \ell(\mathbf{h}_{\mathbf{w}}(\mathbf{x}),y)$
   \\
     $\delta(T)$&the main error in Eq. \ref{Eq::bound1.0}\\
      $\zeta$ & the discrepancy between $\mathbb{P}_{\text{in}}$ and $\mathbb{P}_{\text{wild}}$\\
    $\pi$ & the ratio of OOD distribution in $\mathbb{P}_{\text{wild}}$\\
    
    \bottomrule[1.5pt]
    \end{tabular}
    
    \label{tab: Constants}
\end{table}

\section{Main Theorems}\label{main_theorems}
In this section, we provide a detailed and formal version of our main theorems with a complete description of the  constant terms and other additional details that are omitted in the main paper. 
\begin{theorem*}\label{MainT-1-app}
  If Assumptions \ref{Ass1} and \ref{Ass2} hold, $\mathbb{P}_{\text{wild}}$ has $(\gamma,\zeta)$-discrepancy w.r.t. $\mathbb{P}_{\mathcal{X}\mathcal{Y}}$, and there exists $\eta\in (0,1)$ s.t. $\Delta = (1-\eta)^2\zeta^2 - 8\beta_1 R_{{\text{in}}}^*>0$, then for
      \begin{equation*}
             n = \Omega \big( \frac{\tilde{M}+M(r_1+1)d}{\eta^2 \Delta } +\frac{M^2{d}}{(\gamma-R_{\text{in}}^*)^2} \big),~~~~~m = \Omega \big ( \frac{\tilde{M}+M(r_1+1)d}{\eta^2\zeta^2} \big),
     \end{equation*}
     with the probability at least $9/10$, for any $0<T<M'$ $($here $M' =2(\beta_1r_1+b_1)^2 $ is the upper bound of filtering score $\tau_i$, i.e., $\tau_i \leq M'$$)$,
     \begin{equation}
     {\rm ERR}_{\text{in}}  \leq    \frac{8\beta_1 }{T}R_{\text{in}}^*+ O \Big ( \frac{\tilde{M}}{T}\sqrt{\frac{d}{n}}\Big )+   O \Big ( \frac{\tilde{M}}{T}\sqrt{\frac{d}{(1-\pi)m}}\Big ),
\end{equation}
\begin{equation}
\begin{split}
     {\rm ERR}_{\text{out}} &\leq  \delta(T) + O \Big( \frac{\tilde{M}}{1-T/M'}\sqrt{\frac{d}{{\pi}^2n}}\Big)\\&+ O \Big( \frac{\max \{ {\tilde{M}\sqrt{d}}, \Delta_{\zeta}^{\eta}/\pi\}}{1-T/M'}\sqrt{\frac{1}{{\pi}^2(1-\pi)m}}\Big),
     \end{split}
\end{equation}
     where 
     $R^{*}_{{\text{in}}}$ is the optimal ID risk, i.e., $R^{*}_{{\text{in}}}=\min_{\mathbf{w}\in \mathcal{W}} \mathbb{E}_{(\mathbf{x},y)\sim \mathbb{P}_{\mathcal{X}\mathcal{Y}}} \ell(\mathbf{h}_{\mathbf{w}}(\mathbf{x}),y)$,
     \begin{equation}
     \begin{split}
         \delta(T) &= \frac{\max\{0,1-\Delta_{\zeta}^{\eta}/\pi\}}{(1-T/M')},~~~\Delta_{\zeta}^{\eta} = {0.98\eta^2 \zeta^2} - 8\beta_1R_{\text{in}}^*,\\&M=\beta_1r_1^2+b_1r_1+B_1,~~~\tilde{M}=M\beta_1,
          \end{split}
     \end{equation}
and $d$ is the dimension of the parameter space $\mathcal{W}$, here $\beta_1, r_1, B_1$ are given in Assumption \ref{Ass1}.
\end{theorem*}
$~~~~$
\vspace{1cm}
$~~~~~~$

\begin{theorem*}\label{The-1.1-app}
1) If  $\Delta_{\zeta}^{\eta} \geq  (1-\epsilon)\pi$ for a small error $\epsilon \geq 0$, then the main error $\delta(T)$ defined in Eq. \ref{main-error} satisfies that
    \begin{equation*}
        \delta(T) \leq \frac{\epsilon}{1-T/{M'}}.
    \end{equation*}
    2) If $\zeta\geq 2.011\sqrt{8\beta_1 R_{\text{in}}^*} +1.011 \sqrt{\pi}$, then there exists $\eta\in (0,1)$ ensuring that $\Delta>0$ and $ \Delta_{\zeta}^{\eta}> \pi$ hold, which implies that the main error $\delta(T)=0$.
\end{theorem*}
$~~~~$
\vspace{1cm}
$~~~~~~$
\begin{theorem*}\label{the:main2-app}
Given the same conditions in Theorem~\ref{MainT-1}, if we further require that
     \begin{equation*}
     n = \Omega \Big (\frac{\tilde{M}^2d}{\min \{\pi,\Delta_{\zeta}^{\eta}\}^2} \Big),~~~m= \Omega \Big( \frac{\tilde{M}^2d+\Delta_{\zeta}^{\eta}}{\pi^2(1-\pi)\min \{\pi,\Delta_{\zeta}^{\eta}\}^2}\Big),
 \end{equation*}
then with the probability at least $89/100$, for any $0<T<0.9 M' \min \{1,\Delta_{\zeta}^{\eta}/\pi\}$, the OOD classifier $\*g_{\widehat{\boldsymbol{\theta}}_T}$ learned by the proposed algorithm satisfies the following risk estimation
      \begin{equation}\label{Eq::Bound2.0}
     \begin{aligned}
               & R_{\mathbb{P}_{\text{in}},\mathbb{P}_{\text{out}}}(\mathbf{g}_{\widehat{\boldsymbol{\theta}}_T}) \leq  \inf_{\boldsymbol{\theta} \in \Theta} R_{\mathbb{P}_{\text{in}},\mathbb{P}_{\text{out}}}(\mathbf{g}_{{\boldsymbol{\theta}}})+\frac{3.5 L}{1-\delta(T)}\delta({T})+ \frac{9(1-\pi)L\beta_1}{\pi(1-\delta(T))T} R_{\text{in}}^*
                \\ + &O\Big(\frac{L\max\{\tilde{M}\sqrt{d},\sqrt{d'}\}}{\min\{\pi,\Delta_{\zeta}^{\eta}\}T'}\sqrt{\frac{1}{n}}\Big )+O\Big(\frac{L\max \{\tilde{M}\sqrt{d},\sqrt{d'},\Delta_{\zeta}^{\eta}\}}{\min\{\pi,\Delta_{\zeta}^{\eta}\}T'}\sqrt{\frac{1}{\pi^2(1-\pi)m}}\Big ),    
             \end{aligned}
             \end{equation}
                where  $R_{\text{in}}^*$, $\Delta_{\zeta}^{\eta}$, $M$, $M'$, $\tilde{M}$ and $d$ are shown in Theorem \ref{MainT-1}, $d'$ is the dimension of space $\Theta$,
                \begin{equation*}
                    L= \beta_2 r_2^2+b_2r_2+B_2,~~~T'=T/(1+T),
                \end{equation*}
                 and the risk $R_{\mathbb{P}_{\text{in}},\mathbb{P}_{\text{out}}}(\mathbf{g}_{{\boldsymbol{\theta}}})$ is defined as follows:
                \begin{equation*}
                    R_{\mathbb{P}_{\text{in}},\mathbb{P}_{\text{out}}}(\mathbf{g}_{\widehat{\boldsymbol{\theta}}_T}) = \mathbb{E}_{\mathbf{x}\sim \mathbb{P}_{\text{in}}} \ell_{\text{b}}(\mathbf{g}_{\boldsymbol{\theta}}(\mathbf{x}),y_{+})+\mathbb{E}_{\mathbf{x}\sim \mathbb{P}_{\text{out}}} \ell_{\text{b}}(\mathbf{g}_{\boldsymbol{\theta}}(\mathbf{x}),y_{-}).
                \end{equation*}
\end{theorem*}
$~~~~$
\vspace{1cm}
$~~~~~~$
\begin{theorem*}\label{the:main4.0-app}
Given the same conditions in Theorem~\ref{MainT-1}, with the probability at least $9/10$,
\begin{equation*}
\begin{split}
\mathbb{E}_{\tilde{\mathbf{x}}_i \sim \mathcal{S}_{\text{wild}}^{\text{in}}} \tau_i  \leq & 8\beta_1 R_{\text{in}}^*
        +O (\beta_1 M \sqrt{\frac{d}{n}})+O (\beta_1 M \sqrt{\frac{d}{(1-\pi) m}}),
\end{split}
\end{equation*}
\begin{equation*}
\begin{split}
 \mathbb{E}_{\tilde{\mathbf{x}}_i \sim \mathcal{S}_{\text{wild}}^{\text{out}}} \tau_i \geq & \frac{0.98\eta^2 \zeta^2}{\pi} - \frac{8\beta_1 R_{\text{in}}^*}{\pi} - \epsilon'(n,m),
\end{split}
\end{equation*}
furthermore, if the realizability assumption for ID distribution holds \citep{understand_ml,fang2022learnable}, then
\begin{equation*}
    \mathbb{E}_{\tilde{\mathbf{x}}_i \sim \mathcal{S}_{\text{wild}}^{\text{in}}} \tau_i  \leq  O (\beta_1 M \sqrt{\frac{d}{n}})+O (\beta_1 M \sqrt{\frac{d}{(1-\pi) m}})
\end{equation*}
\begin{equation*}
\begin{split}
 \mathbb{E}_{\tilde{\mathbf{x}}_i \sim \mathcal{S}_{\text{wild}}^{\text{out}}} \tau_i \geq & \frac{0.98\eta^2 \zeta^2}{\pi}  - \epsilon'(n,m),
\end{split}
\end{equation*}
where
\begin{equation*}
    \epsilon'(n,m) \leq O (\frac{\beta_1M}{\pi}\sqrt{\frac{d}{n}})+O \Big( ({{\beta_1 M\sqrt{d}}+ \sqrt{1-\pi}\Delta_{\zeta}^{\eta}/\pi})\sqrt{\frac{1}{{\pi}^2(1-\pi)m}}\Big),
\end{equation*}
and $R_{\text{in}}^*$, $\Delta_{\zeta}^{\eta}$, $M$ and $d$ are shown in Theorem \ref{MainT-1}.
\end{theorem*}

\newpage

\section{Proofs of Main Theorems}\label{Proofmain}

\subsection{Proof of Theorem \ref{MainT-1}}
\label{sec:proof_1}
\noindent \textbf{Step 1.} With the probability at least $1-\frac{7}{3}\delta>0$,
\begin{equation*}
\begin{split}
\mathbb{E}_{\tilde{\mathbf{x}}_i \sim \mathcal{S}_{\text{wild}}^{\text{in}}} \tau_i  \leq & 8\beta_1 R_{\text{in}}^*
        \\ + & 4 \beta_1 \Big [C\sqrt{\frac{M r_1 (\beta_1r_1+b_1) d}{n}}+ C\sqrt{\frac{M r_1 (\beta_1r_1+b_1) d}{(1-\pi)m - \sqrt{m\log(6/\delta)/2}}}\\ &+3M \sqrt{\frac{2\log(6/\delta)}{n}}+M \sqrt{\frac{2\log(6/\delta)}{(1-\pi)m - \sqrt{m\log(6/\delta)/2}}} \Big ],
\end{split}
\end{equation*}

This can be proven by Lemma \ref{gamma-approximation-3} and following inequality
\begin{equation*}
\begin{split}
&\mathbb{E}_{\tilde{\mathbf{x}}_i \sim \mathcal{S}_{\text{wild}}^{\text{in}}} \tau_i \leq \mathbb{E}_{\tilde{\mathbf{x}}_i \sim \mathcal{S}_{\text{wild}}^{\text{in}}}
\big \|\nabla \ell(\*h_{\mathbf{w}_{\mathcal{S}^{\text{in}}}}(\tilde{\mathbf{x}}_i),\widehat{\*h}_{\mathbf{w}_{\mathcal{S}^{\text{in}}}}(\tilde{\mathbf{x}}_i))- \mathbb{E}_{(\*x_j,y_j) \sim \mathcal{S}^{\operatorname{in}}} \nabla \ell(\*h_{\mathbf{w}_{\mathcal{S}^{\text{in}}}}(\mathbf{x}_j), y_j) \big \|_2^2,
\end{split}
\end{equation*}

\noindent \textbf{Step 2.} It is easy to check that
\begin{equation*}
   \mathbb{E}_{\tilde{\mathbf{x}}_i \sim \mathcal{S}_{\text{wild}}} \tau_i = \frac{|\mathcal{S}_{\text{wild}}^{\text{in}}|}{|\mathcal{S}_{\text{wild}}|}\mathbb{E}_{\tilde{\mathbf{x}}_i \sim \mathcal{S}_{\text{wild}}^{\text{in}}} \tau_i +
   \frac{|\mathcal{S}_{\text{wild}}^{\text{out}}|}{|\mathcal{S}_{\text{wild}}|}\mathbb{E}_{\tilde{\mathbf{x}}_i \sim \mathcal{S}_{\text{wild}}^{\text{out}}} \tau_i.
\end{equation*}

\noindent \textbf{Step 3.} Let
\begin{equation*}
\begin{split}
    \epsilon(n,m) =& 4\beta_1[C\sqrt{\frac{M r_1 (\beta_1r_1+b_1) d}{n}}+ C\sqrt{\frac{M r_1 (\beta_1r_1+b_1) d}{(1-\pi)m - \sqrt{m\log(6/\delta)/2}}}\\+&3M \sqrt{\frac{2\log(6/\delta)}{n}}+M \sqrt{\frac{2\log(6/\delta)}{(1-\pi)m - \sqrt{m\log(6/\delta)/2}}}].
    \end{split}
\end{equation*}

Under the condition in Theorem \ref{T1}, with the probability at least $\frac{97}{100} - \frac{7}{3}\delta>0$,
\begin{equation*}
\begin{split}
    \mathbb{E}_{\tilde{\mathbf{x}}_i \sim \mathcal{S}_{\text{wild}}^{\text{out}}} \tau_i \geq & \frac{m}{|\mathcal{S}_{\text{wild}}^{\text{out}}|} \big [\frac{98\eta^2 \zeta^2}{100} - \frac{|\mathcal{S}_{\text{wild}}^{\text{in}}|}{m}8\beta_1 R_{\text{in}}^*- \frac{|\mathcal{S}_{\text{wild}}^{\text{in}}|}{m} \epsilon(n,m) \big ] \\ \geq & \frac{m}{|\mathcal{S}_{\text{wild}}^{\text{out}}|} \big [\frac{98\eta^2 \zeta^2}{100} - 8\beta_1 R_{\text{in}}^*-  \epsilon(n,m) \big ]
    \\ 
    \geq & \big [ \frac{1}{\pi} - \frac{\sqrt{\log{6/\delta}}}{\pi^2\sqrt{2m}+\pi \sqrt{\log(6/\delta)}}\big ]\big [\frac{98\eta^2 \zeta^2}{100} - 8\beta_1 R_{\text{in}}^*-  \epsilon(n,m) \big ].
    \end{split}
\end{equation*}

In this proof, we set
\begin{equation*}
  \Delta(n,m) =  \big [ \frac{1}{\pi} - \frac{\sqrt{\log{6/\delta}}}{\pi^2\sqrt{2m}+\pi \sqrt{\log(6/\delta)}}\big ]\big [\frac{98\eta^2 \zeta^2}{100} - 8\beta_1 R_{\text{in}}^*-  \epsilon(n,m) \big ].
\end{equation*}
Note that $\Delta_{\zeta}^{\eta} = {0.98\eta^2 \zeta^2} - 8\beta_1 R_{\text{in}}^*$, then
\begin{equation*}
     \Delta(n,m)= \frac{1}{\pi}\Delta_{\zeta}^{\eta}-\frac{1}{\pi}\epsilon(n,m)-\Delta_{\zeta}^{\eta} \epsilon(m)+\epsilon(n)\epsilon(n,m),
\end{equation*}
where $\epsilon(m) =  {\sqrt{\log{6/\delta}}}/({\pi^2\sqrt{2m}+\pi \sqrt{\log(6/\delta)}})$.

\noindent \textbf{Step 4.} Under the condition in Theorem \ref{T1}, with the probability at least $\frac{97}{100} - \frac{7}{3}\delta>0$,

\begin{equation}
    \frac{|\{\tilde{\*x}_i\in \mathcal{S}_{\operatorname{wild}}^{\operatorname{out}}: \tau_i \leq T\}|}{| \mathcal{S}_{\operatorname{wild}}^{\operatorname{out}}|} \leq  \frac{1-\min\{1,\Delta(n,m)\}}{1-T/{M'}} ,
\end{equation}
and
\begin{equation}
    \frac{|\{\tilde{\*x}_i\in \mathcal{S}_{\operatorname{wild}}^{\operatorname{in}}: \tau_i > T\}|}{| \mathcal{S}_{\operatorname{wild}}^{\operatorname{in}}|} \leq  \frac{8\beta_1 R_{\text{in}}^*+\epsilon(n,m)}{T}.
\end{equation}

We prove this step:
let $Z$ be the \textbf{uniform} random variable with $\mathcal{S}_{\operatorname{wild}}^{\operatorname{out}}$ as its support and $Z(i)=\tau_i/(2(\beta_1r_1+b_1)^2)$, then by the Markov inequality, we have
\begin{equation}
\frac{|\{\tilde{\*x}_i\in \mathcal{S}_{\operatorname{wild}}^{\operatorname{out}}: \tau_i > T\}|}{| \mathcal{S}_{\operatorname{wild}}^{\operatorname{out}}|} =  P(Z(i)>T/(2(\beta_1r_1+b_1)^2)) \geq \frac{\Delta(n,m) -T/(2(\beta_1r_1+b_1)^2)}{1-T/(2(\beta_1r_1+b_1)^2)} .
\end{equation}
Let $Z$ be the \textbf{uniform} random variable with $\mathcal{S}_{\operatorname{wild}}^{\operatorname{in}}$ as its support and $Z(i)=\tau_i$, then by the Markov inequality, we have
\begin{equation}
     \frac{|\{\tilde{\*x}_i\in \mathcal{S}_{\operatorname{wild}}^{\operatorname{in}}: \tau_i > T\}|}{| \mathcal{S}_{\operatorname{wild}}^{\operatorname{in}}|} =  P(Z(i)>T) \leq \frac{\mathbb{E}[Z]}{T}=\frac{8\beta_1 R_{\text{in}}^*+\epsilon(n,m)}{T}.
\end{equation}

\noindent \textbf{Step 5.}  If $\pi \leq \Delta_{\zeta}^{\eta}/ (1-\epsilon/{M'})$, then
with the probability at least $\frac{97}{100} - \frac{7}{3}\delta>0$,

\begin{equation}
    \frac{|\{\tilde{\*x}_i\in \mathcal{S}_{\operatorname{wild}}^{\operatorname{out}}: \tau_i \leq T\}|}{| \mathcal{S}_{\operatorname{wild}}^{\operatorname{out}}|} \leq  \frac{\epsilon+M'\epsilon'(n,m)}{M'-T},
\end{equation}
and
\begin{equation}
    \frac{|\{\tilde{\*x}_i\in \mathcal{S}_{\operatorname{wild}}^{\operatorname{in}}: \tau_i > T\}|}{| \mathcal{S}_{\operatorname{wild}}^{\operatorname{in}}|} \leq  \frac{8\beta_1 R_{\text{in}}^*+\epsilon(n,m)}{T},
\end{equation}
where $\epsilon'(n,m) = {\epsilon(n,m)}/{\pi}+\Delta_{\zeta}^{\eta}\epsilon(m)-\epsilon(n)\epsilon(n,m)$.

\noindent \textbf{Step 6.} If we set $\delta=3/100$, then it is easy to see that
\begin{equation*}
\begin{split}
   & \epsilon(m) \leq O (\frac{1}{\pi^2 \sqrt{m}}),\\& 
   \epsilon(n,m) \leq O (\beta_1 M \sqrt{\frac{d}{n}})+O (\beta_1 M \sqrt{\frac{d}{(1-\pi) m}}),
   \\ &
   \epsilon'(n,m) \leq O (\frac{\beta_1M}{\pi}\sqrt{\frac{d}{n}})+O \Big( ({{\beta_1 M\sqrt{d}}+ \sqrt{1-\pi}\Delta_{\zeta}^{\eta}/\pi})\sqrt{\frac{1}{{\pi}^2(1-\pi)m}}\Big).
    \end{split}
\end{equation*}

\noindent \textbf{Step 7.} By results in Steps 4, 5 and 6, We complete this proof.
\newpage
\subsection{Proof of Theorem \ref{The-1.1}}
\label{sec:proof_2}
The first result is trivial. Hence, we omit it. We mainly focus on the second result in this theorem. In this proof, then we set
\begin{equation*}
    \eta ={\sqrt{8\beta_1R_{\text{in}}^{*}+0.99\pi}}/({\sqrt{0.98}\sqrt{8\beta_1 R_{\text{in}}^*}+ \sqrt{8\beta_1 R_{\text{in}}^*+\pi}})
\end{equation*}

Note that it is easy to check that
\begin{equation*}
    \zeta\geq 2.011\sqrt{8\beta_1 R_{\text{in}}^*} +1.011 \sqrt{\pi} \geq \sqrt{8\beta_1 R_{\text{in}}^*} +1.011 \sqrt{8\beta_1 R_{\text{in}}^*+ \pi}.
\end{equation*}
Therefore,
\begin{equation*}
   \eta \zeta \geq \frac{1}{\sqrt{0.98}} \sqrt{8\beta_1R_{\text{in}}^{*}+0.99\pi}>\sqrt{8\beta_1 R_{\text{in}}^*+\pi},
\end{equation*}
which implies that $\Delta_{\zeta}^{\eta}>\pi$. Note that
\begin{equation*}
\begin{split}
    (1-\eta)\zeta &\geq \frac{1}{\sqrt{0.98}}\big( {{\sqrt{0.98}\sqrt{8\beta_1 R_{\text{in}}^*}+ \sqrt{8\beta_1 R_{\text{in}}^*+\pi}}-{\sqrt{8\beta_1R_{\text{in}}^{*}+0.99\pi}}}\big )  > \sqrt{8\beta_1 R_{\text{in}}^*},
    \end{split}
\end{equation*}
which implies that $\Delta>0$. We have completed this proof.
$~$
\\

\subsection{Proof of Theorem \ref{the:main2}}
\label{sec:proof_3}
    Let
    \begin{equation*}
        \boldsymbol{\theta}^* \in \argmin_{ \boldsymbol{\theta} \in \Theta} R_{\mathbb{P}_{\text{in}},\mathbb{P}_{\text{out}}}(\boldsymbol{\theta}).
    \end{equation*}
Then by Lemma \ref{UC-I} and Lemma \ref{Error-5}, we obtain that with the high probability
\begin{equation*}
    \begin{split}
        &R_{\mathbb{P}_{\text{in}},\mathbb{P}_{\text{out}}}(\mathbf{g}_{\widehat{\boldsymbol{\theta}}_{T}})-R_{\mathbb{P}_{\text{in}},\mathbb{P}_{\text{out}}}(\mathbf{g}_{{\boldsymbol{\theta}}^*})
        \\ = & R_{\mathbb{P}_{\text{in}},\mathbb{P}_{\text{out}}}(\mathbf{g}_{\widehat{\boldsymbol{\theta}}_{T}})-R_{\mathcal{S}^{\text{in}},\mathcal{S}_T}(\mathbf{g}_{\widehat{\boldsymbol{\theta}}_{T}})+R_{\mathcal{S}^{\text{in}},\mathcal{S}_T}(\mathbf{g}_{\widehat{\boldsymbol{\theta}}_{T}})-R_{\mathcal{S}^{\text{in}},\mathcal{S}_T}(\mathbf{g}_{{\boldsymbol{\theta}}^*})\\+&R_{\mathcal{S}^{\text{in}},\mathcal{S}_T}(\mathbf{g}_{{\boldsymbol{\theta}}^*})-R_{\mathbb{P}_{\text{in}},\mathbb{P}_{\text{out}}}(\mathbf{g}_{{\boldsymbol{\theta}}^*})\\ \leq &R_{\mathbb{P}_{\text{in}},\mathbb{P}_{\text{out}}}(\mathbf{g}_{\widehat{\boldsymbol{\theta}}_{T}})-R_{\mathcal{S}^{\text{in}},\mathcal{S}_T}(\mathbf{g}_{\widehat{\boldsymbol{\theta}}_{T}})\\+&R_{\mathcal{S}^{\text{in}},\mathcal{S}_T}(\mathbf{g}_{{\boldsymbol{\theta}}^*})-R_{\mathbb{P}_{\text{in}},\mathbb{P}_{\text{out}}}(\mathbf{g}_{{\boldsymbol{\theta}}^*})\\ \leq & 2\sup_{\boldsymbol{\theta} \in \Theta} \big|R_{\mathcal{S}^{\text{in}}}^+(\mathbf{g}_{{\boldsymbol{\theta}}})-R_{\mathbb{P}_{\text{in}}}^+(\mathbf{g}_{{\boldsymbol{\theta}}})\big|
        \\ + & \sup_{\boldsymbol{\theta} \in \Theta} \big ( R_{\mathcal{S}^{\text{out}}}^-(\mathbf{g}_{{\boldsymbol{\theta}}})-R_{\mathbb{P}_{\text{out}}}^-(\mathbf{g}_{{\boldsymbol{\theta}}}) \big )+ \sup_{\boldsymbol{\theta} \in \Theta} \big ( R_{\mathbb{P}_{\text{out}}}^-(\mathbf{g}_{{\boldsymbol{\theta}}}) -R_{\mathcal{S}^{\text{out}}}^-(\mathbf{g}_{{\boldsymbol{\theta}}})\big )
        \\ \leq & \frac{3.5 L}{1-\delta(T)}\delta({T})+ \frac{9(1-\pi)L\beta_1}{\pi(1-\delta(T))T} R_{\text{in}}^*
        \\ + &O\Big(\frac{L\max\{\beta_1 M\sqrt{d},\sqrt{d'}\}(1+T)}{\min\{\pi,\Delta_{\zeta}^{\eta}\}T}\sqrt{\frac{1}{n}}\Big )\\+&O\Big(\frac{L\max \{\beta_1 M\sqrt{d},\sqrt{d'},\Delta_{\zeta}^{\eta}\}(1+T)}{\min\{\pi,\Delta_{\zeta}^{\eta}\}T}\sqrt{\frac{1}{\pi^2(1-\pi)m}}\Big ),
    \end{split}
\end{equation*}
$~$
\\

\subsection{Proof of Theorem \ref{the:main4.0-app}}\label{sec::proof_4}
    The result is induced by the Steps \textbf{1}, \textbf{3} and \textbf{6} in Proof of Theorem \ref{MainT-1} (see section \ref{sec:proof_1}).

\newpage

\section{Necessary Lemmas, Propositions and Theorems}\label{NecessaryLemma}

\subsection{Boundedness}

\begin{Proposition}\label{P1}
If Assumption \ref{Ass1} holds, 
\begin{equation*}
  \sup_{\mathbf{w}\in \mathcal{W}}  \sup_{(\mathbf{x},y)\in \mathcal{X}\times \mathcal{Y}} \|\nabla \ell(\mathbf{h}_{\mathbf{w}}(\mathbf{x}),y)\|_2\leq \beta_1 r_1+ b_1 =\sqrt{{M'}/{2}},
\end{equation*}
\begin{equation*}
 \sup_{\boldsymbol{\theta} \in \Theta}   \sup_{(\mathbf{x},y_{\text{b}})\in \mathcal{X}\times \mathcal{Y}_{\text{b}}} \|\nabla \ell(\mathbf{g}_{\boldsymbol{\theta}}(\mathbf{x}),y_{\text{b}})\|_2 \leq \beta_2 r_2+ b_2.
\end{equation*}
\begin{equation*}
\begin{split}
& \sup_{\mathbf{w}\in \mathcal{W}}    \sup_{(\mathbf{x},y)\in \mathcal{X}\times \mathcal{Y}} \ell(\*h_{\mathbf{w}}(\*x), y)  \leq  \beta_1 r_1^2 + b_1 r_1 + B_1=M,
\end{split}
\end{equation*}
\begin{equation*}
\begin{split}
& \sup_{\boldsymbol{\theta}\in \Theta}    \sup_{(\mathbf{x},y_{\text{b}})\in \mathcal{X}\times \mathcal{Y}_{\text{b}}} \ell_{\text{b}}(\mathbf{g}_{\boldsymbol{\theta}}(\*x), y_{\text{b}})  \leq  \beta_2 r_2^2+ b_2 r_2 + B_2=L.
\end{split}
\end{equation*}
\end{Proposition}

\begin{proof}
    One can prove this by \textit{Mean Value Theorem of Integrals} easily.
\end{proof}
$~~$
\\$~$

\begin{Proposition}\label{P2}
If Assumption \ref{Ass1} holds, for any $\mathbf{w}\in \mathcal{W}$,
\begin{equation*}
     \big \| \nabla \ell(\*h_{\mathbf{w}}(\*x), y) \big \|_2^2 \leq 2 \beta_1   \ell(\*h_{\mathbf{w}}(\*x), y).
\end{equation*}
\end{Proposition}
\begin{proof}
The details of the self-bounding property can be found in Appendix B of \cite{lei2021sharper}.
\end{proof}
$~~~$
\\$~$
\begin{Proposition}\label{P3}
If Assumption \ref{Ass1} holds, for any labeled data $\mathcal{S}$ and distribution $\mathbb{P}$, 
\begin{equation*}
     \big \| \nabla R_{\mathcal{S}}(\*h_{\mathbf{w}}) \big \|_2^2 \leq 2 \beta_1   R_{\mathcal{S}}(\*h_{\mathbf{w}}),~~~\forall \mathbf{w}\in \mathcal{W},
\end{equation*}
\begin{equation*}
     \big \| \nabla R_{\mathbb{P}}(\*h_{\mathbf{w}}) \big \|_2^2 \leq 2 \beta_1   R_{\mathbb{P}}(\*h_{\mathbf{w}}),~~~\forall \mathbf{w}\in \mathcal{W}.
\end{equation*}
\end{Proposition}
\begin{proof}
Jensen’s inequality implies that $R_{\mathcal{S}}(\*h_{\mathbf{w}})$ and $R_{\mathbb{P}}(\*h_{\mathbf{w}}) $ are $\beta_1$-smooth. Then Proposition \ref{P2} implies the results.
\end{proof}
$~~~$

\newpage

\subsection{Convergence}

\begin{lemma}[Uniform Convergence-I]\label{UC-I}
    If Assumption \ref{Ass1} holds, then for any distribution $\mathbb{P}$, with the probability at least $1-\delta>0$, for any $\mathbf{w}\in \mathcal{W}$,
    \begin{equation*}
        | R_{\mathcal{S}}(\*h_{\mathbf{w}}) - R_{\mathbb{P}}(\*h_{\mathbf{w}})| \leq M \sqrt{\frac{2\log(2/\delta)}{n}} + C\sqrt{\frac{M r_1 (\beta_1r_1+b_1) d}{n}},
    \end{equation*}
    where $n=|\mathcal{S}|$, $M=\beta_1r_1^2+b_1r_1+B_1$, $d$ is the dimension of $\mathcal{W}$, and $C$ is a uniform constant.
\end{lemma}
\begin{proof}[Proof of Lemma \ref{UC-I}] 
Let 
\begin{equation*}
X_{\mathbf{h}_\mathbf{w}} = \mathbb{E}_{(\mathbf{x},y) \sim \mathbb{P}}  \ell(\mathbf{h}_\mathbf{w}(\mathbf{x}), y)- \mathbb{E}_{(\mathbf{x},y)\sim \mathcal{S}} \ell(\mathbf{h}_\mathbf{w}(\mathbf{x}), y).
\end{equation*}
Then it is clear that
\begin{equation*}
\mathbb{E}_{S\sim \mathbb{P}^n} X_{\mathbf{h}_\mathbf{w}} = 0.
\end{equation*}
By Proposition 2.6.1 and Lemma 2.6.8 in \cite{Vershynin2018HighDimensionalP},
\begin{equation*}
 \|X_{\mathbf{h}_\mathbf{w}}-X_{\mathbf{h}_{\mathbf{w}'}}\|_{\Phi_2}\leq \frac{c_0}{\sqrt{n}}\|\ell(\mathbf{h}_\mathbf{w}(\mathbf{x}),y) -\ell(\mathbf{h}_{\mathbf{w}'}(\mathbf{x}),y) \|_{L^{\infty}(\mathcal{X}\times \mathcal{Y})},
\end{equation*}
where $\|\cdot\|_{\Phi_2}$ is the sub-gaussian norm and $c_0$ is a uniform constant.
Therefore, by Dudley’s entropy integral \citep{Vershynin2018HighDimensionalP}, we have
\begin{equation*}
\mathbb{E}_{S\sim \mathbb{P}^n} \sup_{\mathbf{w} \in \mathcal{W}} X_{\mathbf{h}_\mathbf{w}} \leq  \frac{b_0}{\sqrt{n}}\int_{0}^{+\infty} \sqrt{\log \mathcal{N}(\mathcal{F},\epsilon,L^{\infty}) }{\rm d} \epsilon,
\end{equation*}
where $b_0$ is a uniform constant, 
$
\mathcal{F}=\{\ell(\mathbf{h}_\mathbf{w};\mathbf{x},y) : \mathbf{w} \in \mathcal{W}\},
$ and $\mathcal{N}(\mathcal{F},\epsilon,L^{\infty})$ is the covering number under the $L^{\infty}$ norm.
Note that
\begin{equation*}
\begin{split}
\mathbb{E}_{S\sim \mathbb{P}^n} \sup_{\mathbf{w} \in \mathcal{W}} X_{\mathbf{h}_\mathbf{w}}& \leq \frac{b_0}{\sqrt{n}} \int_{0}^{+\infty} \sqrt{ \log \mathcal{N}(\mathcal{F},\epsilon,L^{\infty})} {\rm d} \epsilon
\\ & = \frac{b_0}{\sqrt{n}} \int_{0}^{M} \sqrt{\log \mathcal{N}(\mathcal{F},\epsilon,L^{\infty})} {\rm d} \epsilon\\ & = \frac{b_0}{\sqrt{n}} M  \int_{0}^{1} \sqrt{\log\mathcal{N}(\mathcal{F},M\epsilon,L^{\infty})} {\rm d} \epsilon.
\end{split}
\end{equation*}
Then, we use the McDiarmid's Inequality, then with the probability at least $1-e^{-t}>0$, for any $\mathbf{w} \in \mathcal{W}$, 
\begin{equation*}
X_{\mathbf{h}_\mathbf{w}} \leq \frac{b_0}{\sqrt{n}} M  \int_{0}^{1} \sqrt{\log\mathcal{N}(\mathcal{F},M\epsilon,L^{\infty})}  {\rm d}\epsilon+ M \sqrt{\frac{2t}{n}}.
\end{equation*}
Similarly, we can also prove that  with the probability at least $1-e^{-t}>0$, for any $\mathbf{w} \in \mathcal{W}$, 
\begin{equation*}
-X_{\mathbf{h}_\mathbf{w}} \leq \frac{b_0}{\sqrt{n}} M  \int_{0}^{1} \sqrt{\log\mathcal{N}(\mathcal{F},M\epsilon,L^{\infty})}  {\rm d}\epsilon+ M \sqrt{\frac{2t}{n}}.
\end{equation*}
Therefore, with the probability at least $1-2e^{-t}>0$, for any $\mathbf{w} \in \mathcal{W}$, 
\begin{equation*}
|X_{\mathbf{h}_\mathbf{w}}| \leq \frac{b_0}{\sqrt{n}} M  \int_{0}^{1} \sqrt{\log\mathcal{N}(\mathcal{F},M\epsilon,L^{\infty})}  {\rm d}\epsilon+ M \sqrt{\frac{2t}{n}}.
\end{equation*}
Note that $\ell(\mathbf{h}_{\mathbf{w}}(\mathbf{x}),y)$ is $(\beta_1r_1+b_1)$-Lipschitz w.r.t. variables $\mathbf{w}$ under the $\|\cdot\|_2$ norm. Then
\begin{equation*}
\begin{split}
\mathcal{N}(\mathcal{F},M\epsilon,L^{\infty}) \leq  & \mathcal{N}(\mathcal{W},M\epsilon/(\beta_1r_1+b_1),\|\cdot\|_2)
\leq (1+ \frac{2r_1(\beta_1r_1+b_1) }{M\epsilon})^{d},
\end{split}
\end{equation*}
which implies that
\begin{equation*}
\begin{split}
\int_{0}^{1} \sqrt{\log(\mathcal{N}(\mathcal{F},M\epsilon,L^{\infty})} {\rm d} \epsilon
\leq & \sqrt{d} \int_{0}^1 \sqrt{\log (1+ \frac{2r_1 (\beta_1r_1+b_1)}{M\epsilon})}  {\rm d} \epsilon
\\ 
\leq & \sqrt{d} \int_{0}^1 \sqrt{\frac{2r_1 (\beta_1r_1+b_1)}{M\epsilon}}  {\rm d} \epsilon
= 2\sqrt{\frac{2r_1 d(\beta_1r_1+b_1) }{M}}.
\end{split}
\end{equation*}
We have completed this proof.
\end{proof}

\begin{lemma}[Uniform Convergence-II]\label{UC-II}
    If Assumption \ref{Ass1} holds, then for any distribution $\mathbb{P}$, with the probability at least $1-\delta>0$, 
    \begin{equation*}
       \big \|\nabla R_{\mathcal{S}}(\*h_{\mathbf{w}}) - \nabla R_{\mathbb{P}}(\*h_{\mathbf{w}}) \big \|_2 \leq B \sqrt{\frac{2\log(2/\delta)}{n}} + C\sqrt{\frac{M(r_1+1) d}{n}},
    \end{equation*}
    where $n=|\mathcal{S}|$, $d$ is the dimension of $\mathcal{W}$, and $C$ is a uniform constant.
\end{lemma}

\begin{proof}[Proof of Lemma \ref{UC-II}] 
    Denote $\ell(\mathbf{v},\*h_{\mathbf{w}}(\mathbf{x}),y)= \left <\nabla\ell(\*h_{\mathbf{w}}(\mathbf{x}),y), \mathbf{v} \right>$ by the loss function over parameter space $\mathcal{W}\times \{\mathbf{v}=1:\mathbf{v}\in \mathbb{R}^d\}$. Let $b$ is the upper bound of $\ell(\mathbf{v},\*h_{\mathbf{w}}(\mathbf{x}),y)$.
   Using the same techniques used in Lemma \ref{UC-I}, we can prove that with the probability at least $1-\delta>0$, for any $\mathbf{w}\in \mathcal{W}$ and any unit vector $\mathbf{v}\in \mathbb{R}^d$,
    \begin{equation*}
        \left < \nabla R_{\mathcal{S}}(\*h_{\mathbf{w}}) - \nabla R_{\mathbb{P}}(\*h_{\mathbf{w}}), \mathbf{v} \right > \leq b \sqrt{\frac{2\log(2/\delta)}{n}} + C\sqrt{\frac{b (r_1+1) \beta_1 d}{n}},
    \end{equation*}
    which implies that
      \begin{equation*}
      \big \|\nabla R_{\mathcal{S}}(\*h_{\mathbf{w}}) - \nabla R_{\mathbb{P}}(\*h_{\mathbf{w}}) \big \|_2 \leq b \sqrt{\frac{2\log(2/\delta)}{n}} + C\sqrt{\frac{b (r_1+1) \beta_1 d}{n}}.
    \end{equation*}
    Note that Proposition \ref{P1} implies that
    \begin{equation*}
        b\beta_1 \leq M.
    \end{equation*}
    Proposition \ref{P2} implies that
    \begin{equation*}
        b \leq \sqrt{2\beta_1 M}.
    \end{equation*}
    We have completed this proof.
\end{proof}

$~~~$
\\

\begin{lemma}\label{SC-I}
Let $\mathcal{S}_{\text{wild}}^{\text{in}}\subset \mathcal{S}_{\text{wild}}$ be the samples drawn from $\mathbb{P}_{\text{in}}$.   With the probability at least $1-\delta>0$,
    \begin{equation*}
        \big| |\mathcal{S}_{\text{wild}}^{\text{in}}|/|\mathcal{S}_{\text{wild}}|-(1-\pi) \big| \leq \sqrt{\frac{\log(2/\delta)}{2|\mathcal{S}_{\text{wild}}|}   },
    \end{equation*}
    which implies that
    \begin{equation*}
         \big| |\mathcal{S}_{\text{wild}}^{\text{in}}|-(1-\pi)|\mathcal{S}_{\text{wild}}| \big| \leq \sqrt{\frac{\log(2/\delta)|\mathcal{S}_{\text{wild}}|}{2}   }.
    \end{equation*}
\end{lemma}
\begin{proof}[Proof of Lemma \ref{SC-I}]
Let $X_i$ be the random variable corresponding to the case whether $i$-th data  in the wild data is drawn from $\mathbb{P}_{\operatorname{in}}$, i.e., $X_i =1$, if $i$-th data is drawn from $\mathbb{P}_{\operatorname{in}}$; otherwise, $X_i =0$. Applying Hoeffding’s inequality, we can get that with the probability at least $1-\delta >0$,
\begin{equation}
    \big | { |\mathcal{S}_{\operatorname{wild}}^{\operatorname{in}}|}/{ |\mathcal{S}_{\operatorname{wild}}|} - (1-\pi)  \big| \leq  \sqrt{\frac{\log (2/\delta)}{2 |\mathcal{S}_{\operatorname{wild}}|}}.
\end{equation}
\end{proof}
\newpage
\subsection{Necessary Lemmas and Theorems for Theorem \ref{MainT-1}}
\begin{lemma}\label{gamma-approximation}
    With the probability at least $1-\delta>0$, the ERM optimizer $\mathbf{w}_{\mathcal{S}}$ is the $\min_{\mathbf{w}\in \mathcal{W}} \mathbf{R}_{\mathbb{P}}(\mathbf{h}_{\mathbf{w}})+ O(1/\sqrt{n})$-risk point, i.e.,
    \begin{equation*}
          R_{\mathcal{S}}(\*h_{\mathbf{w}_{\mathcal{S}}}) \leq \min_{\mathbf{w}\in \mathcal{W}} R_{\mathbb{P}}(\*h_{\mathbf{w}}) + M \sqrt{\frac{\log(1/\delta)}{2n}},
    \end{equation*}
    where $n=|\mathcal{S}|$.
\end{lemma}
\begin{proof}[Proof of Lemma \ref{gamma-approximation}] Let $\mathbf{w}^* \in \argmin_{\mathbf{w}\in \mathcal{W}} \mathbf{R}_{\mathbb{P}}(\mathbf{h}_{\mathbf{w}})$. Applying Hoeffding’s inequality, we obtain that with the probability at least $1-\delta>0$,
\begin{equation*}
    R_{\mathcal{S}}(\*h_{\mathbf{w}_{\mathcal{S}}}) - \min_{\mathbf{w}\in \mathcal{W}} {R}_{\mathbb{P}}(\mathbf{h}_{\mathbf{w}}) \leq R_{\mathcal{S}} (\*h_{\mathbf{w}^{*}}) - {R}_{\mathbb{P}}(\mathbf{h}_{\mathbf{w}^*})\leq M \sqrt{\frac{\log(1/\delta)}{2n}}. 
\end{equation*}
\end{proof}
$~~~$
\\
\begin{lemma}\label{gamma-approximation-1}
If Assumptions \ref{Ass1} and \ref{Ass2} hold, then for any data $\mathcal{S}\sim \mathbb{P}^{n}$ and $\mathcal{S}'\sim \mathbb{P}^{n'}$,
    with the probability at least $1-\delta>0$, 
    \begin{equation*}
    \begin{split}
          R_{\mathcal{S}’}(\*h_{\mathbf{w}_{\mathcal{S}}})   \leq    \min_{\mathbf{w}\in \mathcal{W}} R_{\mathbb{P}}(\*h_{\mathbf{w}})&+ C\sqrt{\frac{M r_1 (\beta_1r_1+b_1) d}{n}}+ C\sqrt{\frac{M r_1 (\beta_1r_1+b_1) d}{n'}}\\&+2M \sqrt{\frac{2\log(6/\delta)}{n}}+M \sqrt{\frac{2\log(6/\delta)}{n'}} ,
          \end{split}
    \end{equation*}
    \begin{equation*}
    \begin{split}
     R_{\mathbb{P}}(\*h_{\mathbf{w}_{\mathcal{S}}},\widehat{\*h}) \leq    R_{\mathbb{P}}(\*h_{\mathbf{w}_{\mathcal{S}}})   \leq    \min_{\mathbf{w}\in \mathcal{W}} R_{\mathbb{P}}(\*h_{\mathbf{w}})&+ C\sqrt{\frac{M r_1 (\beta_1r_1+b_1) d}{n}}+2M \sqrt{\frac{2\log(6/\delta)}{n}},
          \end{split}
    \end{equation*}
    where $C$ is a uniform constant, and
\begin{equation*}
    R_{\mathbb{P}}(\*h_{\mathbf{w}_{\mathcal{S}}},\widehat{\*h}) = \mathbb{E}_{(\mathbf{x},y)\sim \mathbb{P}} \ell(\*h_{\mathbf{w}_{\mathcal{S}}}(\mathbf{x}),\widehat{\*h}(\mathbf{x})).
\end{equation*}
\end{lemma}
\begin{proof}[Proof of Lemma \ref{gamma-approximation-1}] Let $\mathbf{w}_{\mathcal{S}'} \in \argmin_{\mathbf{w}\in \mathcal{W}} {R}_{\mathcal{S}'}(\mathbf{h}_{\mathbf{w}})$ and $\mathbf{w}^* \in \argmin_{\mathbf{w}\in \mathcal{W}} {R}_{\mathbb{P}}(\mathbf{h}_{\mathbf{w}^*})$. By Lemma \ref{UC-I} and Hoeffding Inequality, we obtain that with the probability at least $1-\delta>0$,
\begin{equation*}
\begin{split}
  &  R_{\mathcal{S}'}(\*h_{\mathbf{w}_{\mathcal{S}}}) - R_{\mathbb{P}}(\mathbf{h}_{\mathbf{w}^*}) \\ \leq &
  R_{\mathcal{S}'}(\*h_{\mathbf{w}_{\mathcal{S}}}) - {R}_{\mathbb{P}}(\mathbf{h}_{\mathbf{w}_{\mathcal{S}}}) +R_{\mathbb{P}}(\*h_{\mathbf{w}_{\mathcal{S}}}) -{R}_{\mathcal{S}}(\mathbf{h}_{\mathbf{w}_{\mathcal{S}}}) + R_{\mathcal{S}}(\mathbf{w}^*)-R_{\mathbb{P}}(\mathbf{w}^*)\\ \leq &  C\sqrt{\frac{M r_1 (\beta_1r_1+b_1) d}{n}}+ C\sqrt{\frac{M r_1 (\beta_1r_1+b_1) d}{n'}}+2M \sqrt{\frac{2\log(6/\delta)}{n}}+M \sqrt{\frac{2\log(6/\delta)}{n'}},
    \end{split}
\end{equation*}
\begin{equation*}
    \begin{split}
    R_{\mathbb{P}}(\*h_{\mathbf{w}_{\mathcal{S}}}) - R_{\mathbb{P}}(\mathbf{h}_{\mathbf{w}^*})  \leq &
   R_{\mathbb{P}}(\*h_{\mathbf{w}_{\mathcal{S}}}) -{R}_{\mathcal{S}}(\mathbf{h}_{\mathbf{w}_{\mathcal{S}}}) + R_{\mathcal{S}}(\mathbf{w}^*)-R_{\mathbb{P}}(\mathbf{w}^*)\\ \leq &  C\sqrt{\frac{M r_1 (\beta_1r_1+b_1) d}{n}}+2M \sqrt{\frac{2\log(6/\delta)}{n}}.
    \end{split}
\end{equation*}
\end{proof}

$~~~$
\\

\begin{lemma}\label{gamma-approximation-2}
If Assumptions \ref{Ass1} and \ref{Ass2} hold, then for any data $\mathcal{S}\sim \mathbb{P}^{n}$ and $\mathcal{S}'\sim \mathbb{P}^{n'}$, with the probability at least $1-2\delta>0$,
    \begin{equation*}
    \begin{split}
         \mathbb{E}_{(\mathbf{x},y)\sim \mathcal{S}'} \big\| \nabla \ell(\mathbf{h}_{\mathbf{w}_{\mathcal{S}}}(\mathbf{x}),\widehat{\*h}(\mathbf{x})) - &\nabla R_{\mathcal{S}}(\mathbf{h}_{\mathbf{w}_{\mathcal{S}}})\big\|_2^2 \leq 8\beta_1 \min_{\mathbf{w}\in \mathcal{W}} R_{\mathbb{P}}(\*h_{\mathbf{w}}) 
        \\ + & C\sqrt{\frac{M r_1 (\beta_1r_1+b_1) d}{n}}+ C\sqrt{\frac{M r_1 (\beta_1r_1+b_1) d}{n'}}\\+ & 3M \sqrt{\frac{2\log(6/\delta)}{n}}+M \sqrt{\frac{2\log(6/\delta)}{n'}}.
          \end{split}
    \end{equation*}
    where $C$ is a uniform constant.
\end{lemma}
\begin{proof}[Proof of Lemma \ref{gamma-approximation-2}] By Propositions \ref{P2}, \ref{P3} and Lemmas \ref{gamma-approximation} and \ref{gamma-approximation-1}, with the probability at least $1-2\delta>0$,
    \begin{equation*}
 \begin{split}
       & \mathbb{E}_{(\mathbf{x},y)\sim \mathcal{S}'} \big\| \nabla \ell(\mathbf{h}_{\mathbf{w}_{\mathcal{S}}}(\mathbf{x}),\widehat{\*h}(\mathbf{x})) - \nabla R_{\mathcal{S}}(\mathbf{h}_{\mathbf{w}_{\mathcal{S}}})\big\|_2^2
       \\ \leq & 2\mathbb{E}_{(\mathbf{x},y)\sim \mathcal{S}'} \big\| \nabla \ell(\mathbf{h}_{\mathbf{w}_{\mathcal{S}}}(\mathbf{x}),\widehat{\*h}(\mathbf{x}))\big \|^2_2 + 2 \big \| \nabla R_{\mathcal{S}}(\mathbf{h}_{\mathbf{w}_{\mathcal{S}}})\big\|_2^2
       \\  \leq & 4\beta_1 \big ( R_{\mathcal{S}'}(\mathbf{h}_{\mathbf{w}_{\mathcal{S}}},\widehat{\*h})+R_{\mathcal{S}}(\mathbf{h}_{\mathbf{w}_{\mathcal{S}}}) \big ) \leq 4\beta_1 \big ( R_{\mathcal{S}'}(\mathbf{h}_{\mathbf{w}_{\mathcal{S}}})+R_{\mathcal{S}}(\mathbf{h}_{\mathbf{w}_{\mathcal{S}}}) \big )
       \\ \leq &
4\beta_1 \big[ 2\min_{\mathbf{w}\in \mathcal{W}} R_{\mathbb{P}}(\*h_{\mathbf{w}})+ C\sqrt{\frac{M r_1 (\beta_1r_1+b_1) d}{n}}+ C\sqrt{\frac{M r_1 (\beta_1r_1+b_1) d}{n'}}\\&~~~~~~~~+ 3M \sqrt{\frac{2\log(6/\delta)}{n}}+M \sqrt{\frac{2\log(6/\delta)}{n'}} \big ].
\end{split}
    \end{equation*}
\end{proof}
$~~~$
\\

\begin{lemma}\label{gamma-approximation-3}
Let $\mathcal{S}_{\text{wild}}^{\text{in}} \subset \mathcal{S}_{\text{wild}}$ be samples drawn from $\mathbb{P}_{\text{in}}$. If Assumptions \ref{Ass1} and \ref{Ass2} hold, then for any data $\mathcal{S}_{\text{wild}} \sim \mathbb{P}_{\text{wild}}^{m}$ and $\mathcal{S}\sim \mathbb{P}_{\text{in}}^{n}$, with the probability at least $1-\frac{7}{3}\delta>0$, 
    \begin{equation*}
    \begin{split}
         \mathbb{E}_{\mathbf{x} \sim \mathcal{S}_{\text{wild}}^{\text{in}}} \big\| \nabla \ell(\mathbf{h}_{\mathbf{w}_{\mathcal{S}}}(\mathbf{x}),\widehat{\*h}(\mathbf{x})) -& \nabla R_{\mathcal{S}}(\mathbf{h}_{\mathbf{w}_{\mathcal{S}}})\big\|_2^2 \leq 8\beta_1 \min_{\mathbf{w}\in \mathcal{W}} R_{\mathbb{P}}(\*h_{\mathbf{w}}) 
        \\ + & 4 \beta_1 \Big [C\sqrt{\frac{M r_1 (\beta_1r_1+b_1) d}{n}}+ C\sqrt{\frac{M r_1 (\beta_1r_1+b_1) d}{(1-\pi)m - \sqrt{m\log(6/\delta)/2}}}\\ & ~~~~~~~~+3M \sqrt{\frac{2\log(6/\delta)}{n}}+M \sqrt{\frac{2\log(6/\delta)}{(1-\pi)m - \sqrt{m\log(6/\delta)/2}}} \Big ],
          \end{split}
    \end{equation*}
    where $C$ is a uniform constant.
\end{lemma}
\begin{proof}[Proof of Lemma \ref{gamma-approximation-3}]
 Lemma \ref{SC-I} and   Lemma \ref{gamma-approximation-2} imply this result.
\end{proof}
$~~~$
\\

\begin{lemma}\label{gamma-approximation-3.0}
 If Assumptions \ref{Ass1} and \ref{Ass2} hold, then for any data $\mathcal{S}\sim \mathbb{P}_{\text{in}}^{n}$, with the probability at least $1-\delta>0$, 
    \begin{equation*}
    \begin{split}
          \big\| \nabla R_{\mathcal{S}}(\mathbf{h}_{\mathbf{w}_{\mathcal{S}}}(\mathbf{x}),\widehat{\*h}(\mathbf{x})) -& \nabla R_{\mathcal{S}}(\mathbf{h}_{\mathbf{w}_{\mathcal{S}}})\big\|_2^2 \leq 8\beta_1 \min_{\mathbf{w}\in \mathcal{W}} R_{\mathbb{P}}(\*h_{\mathbf{w}}) 
        + 4 M \sqrt{\frac{\log(1/\delta)}{2n}}.
          \end{split}
    \end{equation*}
\end{lemma}
\begin{proof}[Proof of Lemma \ref{gamma-approximation-3.0}] With the probability at least $1-\delta>0$,
\begin{equation*}
\begin{split}
   &\big\| \nabla R_{\mathcal{S}}(\mathbf{h}_{\mathbf{w}_{\mathcal{S}}}(\mathbf{x}),\widehat{\*h}(\mathbf{x})) - \nabla R_{\mathcal{S}}(\mathbf{h}_{\mathbf{w}_{\mathcal{S}}})\big\|_2 
   \\ 
   \leq &\big\| \nabla R_{\mathcal{S}}(\mathbf{h}_{\mathbf{w}_{\mathcal{S}}}(\mathbf{x}),\widehat{\*h}(\mathbf{x}))\big\|_2 +  \big\|\nabla R_{\mathcal{S}}(\mathbf{h}_{\mathbf{w}_{\mathcal{S}}})\big\|_2 
   \\
   \leq & 
   \sqrt{2\beta_1 R_{\mathcal{S}}(\mathbf{h}_{\mathbf{w}_{\mathcal{S}}}(\mathbf{x}),\widehat{\*h}(\mathbf{x}))} +  \sqrt{2\beta_1  R_{\mathcal{S}}(\mathbf{h}_{\mathbf{w}_{\mathcal{S}}})}
   \\
   \leq & 2\sqrt{2\beta_1 (\min_{\mathbf{w}\in \mathcal{W}} R_{\mathbb{P}}(\*h_{\mathbf{w}})+M \sqrt{\frac{\log(1/\delta)}{2n}} ) }.
\end{split}
\end{equation*}
\end{proof}

\newpage

\begin{theorem*}\label{T1}
     If Assumptions \ref{Ass1} and \ref{Ass2} hold and there exists $\eta\in (0,1)$ such that $\Delta = (1-\eta)^2\zeta^2 - 8\beta_1 \min_{\mathbf{w}\in \mathcal{W}} R_{\mathbb{P}_{\text{in}}}(\mathbf{h}_{\mathbf{w}})>0$, when
     \begin{equation*}
         n = \Omega \big( \frac{\tilde{M}+M(r_1+1)d}{\Delta \eta^2 } +\frac{M^2{d}}{(\gamma-R_{\text{in}}^*)^2} \big),~~~~~m = \Omega \big ( \frac{\tilde{M}+M(r_1+1)d}{\eta^2\zeta^2} \big),
     \end{equation*}
     with the probability at least $97/100$,
     \begin{equation*}
         \mathbb{E}_{\tilde{\mathbf{x}}_i \sim \mathcal{S}_{\text{wild}}} \tau_i  >\frac{98\eta^2 \zeta^2}{100}.
     \end{equation*}
\end{theorem*}

\begin{proof}[Proof of Theorem \ref{T1}]

\noindent \textbf{Claim 1.} With the probability at least $1-2\delta>0$, for any $\mathbf{w}\in \mathcal{W}$,
    \begin{equation*}
    \begin{split}
     d^{\ell}_{\mathbf{w}}(\mathbb{P}_{\text{in}},\mathbb{P}_{\text{wild}})-  d^{\ell}_{\mathbf{w}}(\mathcal{S}^{\text{in}}_X,\mathcal{S}_{\text{wild}}) & \leq B \sqrt{\frac{2\log(2/\delta)}{n}}+B \sqrt{\frac{2\log(2/\delta)}{m}} \\ & + C\sqrt{\frac{M (r_1+1) d}{n}}+C\sqrt{\frac{M (r_1+1) d}{m}},
     \end{split}
    \end{equation*}
where $B=\sqrt{2\beta_1 M}$, $\mathcal{S}^{\text{in}}_X$ is the feature part of $\mathcal{S}^{\text{in}}$ and $C$ is a uniform constant.

We prove this Claim: by Lemma \ref{UC-II}, it is notable that with the probability at least $1-2\delta>0$,
\begin{equation*}
\begin{split}
& d^{\ell}_{\mathbf{w}}(\mathbb{P}_{\text{in}},\mathbb{P}_{\text{wild}})-  d^{\ell}_{\mathbf{w}}(\mathcal{S}^{\text{in}}_X,\mathcal{S}_{\text{wild}}) \\
  \leq  & \big \| \nabla R_{\mathbb{P}_{\text{in}}}(\*h_{\mathbf{w}},\widehat{\*h}) -  \nabla R_{\mathbb{P}_{\text{wild}}}(\*h_{\mathbf{w}},\widehat{\*h}) \big \|_2 - \big \| \nabla R_{\mathcal{S}^{\text{in}}_X}(\*h_{\mathbf{w}},\widehat{\*h}) -  \nabla R_{\mathcal{S}_{\text{wild}}}(\*h_{\mathbf{w}},\widehat{\*h}) \big \|_2
  \\
  \leq  &  \big \| \nabla R_{\mathbb{P}_{\text{in}}}(\*h_{\mathbf{w}},\widehat{\*h}) -  \nabla R_{\mathbb{P}_{\text{wild}}}(\*h_{\mathbf{w}},\widehat{\*h}) - \nabla R_{\mathcal{S}^{\text{in}}_X}(\*h_{\mathbf{w}},\widehat{\*h}) +  \nabla R_{\mathcal{S}_{\text{wild}}}(\*h_{\mathbf{w}},\widehat{\*h}) \big \|_2
   \\
  \leq  & \big \| \nabla R_{\mathbb{P}_{\text{in}}}(\*h_{\mathbf{w}},\widehat{\*h}) -  \nabla R_{\mathcal{S}^{\text{in}}_X}(\*h_{\mathbf{w}},\widehat{\*h}) \big \|_2 + \big \| \nabla R_{\mathbb{P}_{\text{wild}}}(\*h_{\mathbf{w}},\widehat{\*h}) -  \nabla R_{\mathcal{S}_{\text{wild}}}(\*h_{\mathbf{w}},\widehat{\*h}) \big \|_2
  \\
  \leq & B \sqrt{\frac{2\log(2/\delta)}{n}}+B \sqrt{\frac{2\log(2/\delta)}{m}} + C\sqrt{\frac{B (r_1+1) \beta_1 d}{n}}+C\sqrt{\frac{B (r_1+1) \beta_1 d}{m}}\\
  \leq & B \sqrt{\frac{2\log(2/\delta)}{n}}+B \sqrt{\frac{2\log(2/\delta)}{m}} + C\sqrt{\frac{M (r_1+1) d}{n}}+C\sqrt{\frac{M (r_1+1) d}{m}}.
 \end{split}
\end{equation*}

    \noindent \textbf{Claim 2.} 
    When 
    \begin{equation}\label{gamma-error}
        \sqrt{n} = \Omega \big (\frac{M\sqrt{d}+M\sqrt{\log (6/\delta})}{\gamma-\min_{\mathbf{w}\in \mathcal{W}} R_{\mathbb{P}_{\text{in}}}(\mathbf{h}_{\mathbf{w}})}\big),
    \end{equation}

    with the probability at least $1-4\delta>0$,
    \begin{equation*}
    \begin{split}
        &\mathbb{E}_{\tilde{\mathbf{x}}_i \sim \mathcal{S}_{\text{wild}}} \tau_i \geq \Big( \zeta   -  B \sqrt{\frac{2\log(2/\delta)}{n}}-B \sqrt{\frac{2\log(2/\delta)}{m}}  \\-& C\sqrt{\frac{M (r_1+1)  d}{n}}-C\sqrt{\frac{M (r_1+1)  d}{m}}-2\sqrt{2\beta_1 (\min_{\mathbf{w}\in \mathcal{W}} R_{\mathbb{P}}(\*h_{\mathbf{w}})+M \sqrt{\frac{\log(1/\delta)}{2n}} ) }                            \Big)^2.
        \end{split}
    \end{equation*}

We prove this Claim: let $\mathbf{v}^*$ be the top-1 right singular vector computed in our algorithm, and
\begin{equation*}
 \tilde{\mathbf{v}} \in   \argmax_{\|\mathbf{v}\|\leq 1}  \left < \mathbb{E}_{(\mathbf{x}_i,y_i)\sim \mathcal{S}^{\text{in}}} \nabla \ell(\*h_{\mathbf{w}_{\mathcal{S}^{\text{in}}}}(\mathbf{x}_i), y_i)   -   \mathbb{E}_{\tilde{\mathbf{x}}_i \in \mathcal{S}_{\text{wild}}}  \nabla \ell(\*h_{\mathbf{w}_{\mathcal{S}^{\text{in}}}}(\tilde{\mathbf{x}}_i), \widehat{\*h}(\tilde{\mathbf{x}}_i))  , \mathbf{v} \right >. 
\end{equation*}

Then with the probability at least $1-4\delta>0$,
\begin{equation*}
    \begin{split}
      &~~~~~\mathbb{E}_{\tilde{\mathbf{x}}_i \sim \mathcal{S}_{\text{wild}}} \tau_i  
      \\& =  \mathbb{E}_{\tilde{\*x}_i \sim \mathcal{S}_{\operatorname{wild}}} \left( \left < \nabla \ell(\*h_{\mathbf{w}_{\mathcal{S}^{\text{in}}}}(\tilde{\mathbf{x}}_i), \widehat{\*h}(\tilde{\mathbf{x}}_i)) - 
 \mathbb{E}_{(\mathbf{x}_j,y_j)\sim \mathcal{S}^{\text{in}}} \nabla \ell(\*h_{\mathbf{w}_{\mathcal{S}^{\text{in}}}}(\mathbf{x}_j), y_j)   , \mathbf{v}^* \right > \right)^2    \\
 & \geq  \mathbb{E}_{\tilde{\*x}_i \sim \mathcal{S}_{\operatorname{wild}}} \left(\left < \nabla \ell(\*h_{\mathbf{w}_{\mathcal{S}^{\text{in}}}}(\tilde{\mathbf{x}}_i), \widehat{\*h}(\tilde{\mathbf{x}}_i)) - 
\mathbb{E}_{(\mathbf{x}_j,y_j)\sim \mathcal{S}^{\text{in}}} \nabla \ell(\*h_{\mathbf{w}_{\mathcal{S}^{\text{in}}}}(\mathbf{x}_j), y_j)   , \tilde{\mathbf{v}} \right > \right)^2    \\
    &  \geq  \left(\left <  \mathbb{E}_{(\mathbf{x}_j,y_j)\sim \mathcal{S}^{\text{in}}} \nabla \ell(\*h_{\mathbf{w}_{\mathcal{S}^{\text{in}}}}(\mathbf{x}_j), y_j)   - \mathbb{E}_{\tilde{\mathbf{x}}_i\sim \mathcal{S}_{\text{wild}}} \nabla \ell(\*h_{\mathbf{w}_{\mathcal{S}^{\text{in}}}}(\tilde{\mathbf{x}}_i), \widehat{\*h}(\tilde{\mathbf{x}}_i))      , \tilde{\mathbf{v}} \right > \right)^2 \\
     &  = \big \| \mathbb{E}_{(\mathbf{x}_j,y_j)\sim \mathcal{S}^{\text{in}}} \nabla \ell(\*h_{\mathbf{w}_{\mathcal{S}^{\text{in}}}}(\mathbf{x}_j), y_j)   - \mathbb{E}_{\tilde{\mathbf{x}}_i\sim \mathcal{S}_{\text{wild}}} \nabla \ell(\*h_{\mathbf{w}_{\mathcal{S}^{\text{in}}}}(\tilde{\mathbf{x}}_i), \widehat{\*h}(\tilde{\mathbf{x}}_i))  \big \|^2_2 \\
    & \geq \big( d^{\ell}_{\mathbf{w}_{\mathcal{S}^{\text{in}}}}(\mathcal{S}^{\text{in}}_X,\mathcal{S}_{\text{wild}})  
    \\ &-  \big \| \mathbb{E}_{(\mathbf{x}_j,y_j)\sim \mathcal{S}^{\text{in}}} \nabla \ell(\*h_{\mathbf{w}_{\mathcal{S}^{\text{in}}}}(\mathbf{x}_j), y_j)   -  \mathbb{E}_{(\mathbf{x}_j,y_j)\sim \mathcal{S}^{\text{in}}} \nabla \ell(\*h_{\mathbf{w}_{\mathcal{S}^{\text{in}}}}({\mathbf{x}}_j), \widehat{\*h}({\mathbf{x}}_j))  \big \|_2                            \big)^2
\\ &  \geq  \Big( \zeta   -  B \sqrt{\frac{2\log(2/\delta)}{n}}-B\sqrt{\frac{2\log(2/\delta)}{m}}  \\-& C\sqrt{\frac{M (r_1+1) d}{n}}-C\sqrt{\frac{M (r_1+1)  d}{m}}-2\sqrt{2\beta_1 (\min_{\mathbf{w}\in \mathcal{W}} R_{\mathbb{P}}(\*h_{\mathbf{w}})+M \sqrt{\frac{\log(1/\delta)}{2n}} ) }                            \Big)^2.
\end{split}
\end{equation*}
In above inequality, we have used the results in Claim 1, Assumption \ref{Ass2}, Lemma \ref{gamma-approximation-1} and Lemma \ref{gamma-approximation-3.0}.

\noindent \textbf{Claim 3.} Given $\delta=1/100$, then when 
 \begin{equation*}
         n = \Omega \big( \frac{\tilde{M}+M(r_1+1)d}{\Delta \eta^2 } \big),~~~~~m = \Omega \big ( \frac{\tilde{M}+M(r_1+1)d}{\eta^2\zeta^2} \big),
     \end{equation*}
     the following inequality holds:
     \begin{equation*}
     \begin{split}
       &  \Big( \zeta   -  B \sqrt{\frac{2\log(2/\delta)}{n}}-B \sqrt{\frac{2\log(2/\delta)}{m}}  - C\sqrt{\frac{M (r_1+1) d}{n}}-C\sqrt{\frac{M (r_1+1) d}{m}}\\-&2\sqrt{2\beta_1 (\min_{\mathbf{w}\in \mathcal{W}} R_{\mathbb{P}}(\*h_{\mathbf{w}})+M \sqrt{\frac{\log(1/\delta)}{2n}} ) }                            \Big)^2  >\frac{98\eta^2\theta^2}{100}.
         \end{split}
     \end{equation*}
We prove this Claim: when
\begin{equation*}
    n \geq \frac{64\sqrt{\log(10)} \beta_1 M}{\Delta},
\end{equation*}
it is easy to check that
\begin{equation*}
    (1-\eta)\zeta \geq 2\sqrt{2\beta_1 (\min_{\mathbf{w}\in \mathcal{W}} R_{\mathbb{P}}(\*h_{\mathbf{w}})+M \sqrt{\frac{\log(1/\delta)}{2n}} )}.
\end{equation*}

Additionally, when 
\begin{equation*}
    n \geq \frac{200^2 \log 200 B^2}{\eta^2 \zeta^2} +\frac{200^2 C^2M(r_1+1)d}{2\eta^2 \zeta^2},
\end{equation*}
it is easy to check that
\begin{equation*}
    \frac{\eta \zeta}{100} \geq B\sqrt{\frac{2\log(200)}{n}}+ C\sqrt{\frac{M(r_1+1)d}{n}}.
\end{equation*}
Because
\begin{equation*}
    \max\{\frac{200^2 \log 200 B^2}{\eta^2 \zeta^2} +\frac{200^2 C^2M(r_1+1)d}{2\eta^2 \zeta^2}, \frac{64\sqrt{\log(10)} \beta_1 M}{\Delta}\} \leq O\big( \frac{\tilde{M}+M(r_1+1)d}{\Delta \eta^2 }\big ),
\end{equation*}
we conclude that when
\begin{equation*}
    n= \Omega \big( \frac{\tilde{M}+M(r_1+1)d}{\Delta \eta^2 }\big ),
\end{equation*}

\begin{equation}\label{3/4eta}
    \eta -  2\sqrt{2\beta_1 (\min_{\mathbf{w}\in \mathcal{W}} R_{\mathbb{P}}(\*h_{\mathbf{w}})+M \sqrt{\frac{\log(1/\delta)}{2n}} )}- B\sqrt{\frac{2\log(200)}{n}}+ C\sqrt{\frac{M(r_1+1)d}{n}} \geq \frac{99}{100}\eta \zeta.
\end{equation}

When 
\begin{equation*}
    m \geq \frac{200^2 \log 200 B^2}{\eta^2 \zeta^2} +\frac{200^2 C^2M(r_1+1)d}{2\eta^2 \zeta^2},
\end{equation*}
we have
\begin{equation*}
\frac{\eta \zeta}{100} \geq B \sqrt{\frac{2\log(200)}{m}}+ C\sqrt{\frac{M(r_1+1)d}{m}}.
\end{equation*}
Therefore, if
\begin{equation*}
  m = \Omega \big ( \frac{\tilde{M}+M(r_1+1)d}{\eta^2\zeta^2} \big), 
\end{equation*}
we have 
\begin{equation}\label{1/4eta}
    \frac{\eta \zeta}{100} \geq B\sqrt{\frac{2\log(200)}{m}}+ C\sqrt{\frac{M(r_1+1)d}{m}}.
\end{equation}
Combining inequalities \ref{gamma-error}, \ref{3/4eta} and \ref{1/4eta}, we complete this proof.

\end{proof}

\subsection{Necessary Lemmas for Theorem \ref{the:main2}}

Let 
\begin{equation*}
    R_{\mathcal{S}_T}^-(\mathbf{g}_{\boldsymbol{\theta}}) =  \mathbb{E}_{\mathbf{x}\sim \mathcal{S}_T} \ell_{\text{b}}(\mathbf{g}_{\boldsymbol{\theta}}(\mathbf{x}), y_{-}),~ R_{\mathcal{S}_{\text{wild}}^{\text{out}}}^-(\mathbf{g}_{\boldsymbol{\theta}}) =  \mathbb{E}_{\mathbf{x}\sim \mathcal{S}_{\text{wild}}^{\text{out}}} \ell_{\text{b}}(\mathbf{g}_{\boldsymbol{\theta}}(\mathbf{x}), y_{-}),
\end{equation*}
\begin{equation*}
    R_{\mathcal{S}^{\text{in}}}^+(\mathbf{g}_{\boldsymbol{\theta}}) =  \mathbb{E}_{\mathbf{x}\sim \mathcal{S}^{\text{in}}} \ell_{\text{b}}(\mathbf{g}_{\boldsymbol{\theta}}(\mathbf{x}), y_{+}),~ R_{\mathbb{P}_{\text{in}}}^+(\mathbf{g}_{\boldsymbol{\theta}}) =  \mathbb{E}_{\mathbf{x}\sim \mathcal{S}^{\text{in}}} \ell_{\text{b}}(\mathbf{g}_{\boldsymbol{\theta}}(\mathbf{x}), y_{+}),
\end{equation*}
and
\begin{equation*}
   R_{\mathbb{P}_{\text{out}}}^-(\mathbf{g}_{\boldsymbol{\theta}}) =  \mathbb{E}_{\mathbf{x}\sim \mathcal{S}^{\text{in}}} \ell_{\text{b}}(\mathbf{g}_{\boldsymbol{\theta}}(\mathbf{x}), y_{-}).
\end{equation*}
$~$
\\

  Let 
    \begin{equation*}
    \begin{split}
   & \mathcal{S}_{+}^{\text{out}} = \{{\tilde{\mathbf{x}}}_i \in \mathcal{S}_{\text{wild}}^{\text{out}}:\tilde{\mathbf{x}}_i \leq T\}, ~~\mathcal{S}_{-}^{\text{in}} = \{{\tilde{\mathbf{x}}}_i \in \mathcal{S}_{\text{wild}}^{\text{in}}:\tilde{\mathbf{x}}_i > T\},
   \\ &\mathcal{S}_{-}^{\text{out}} = \{{\tilde{\mathbf{x}}}_i \in \mathcal{S}_{\text{wild}}^{\text{out}}:\tilde{\mathbf{x}}_i > T\}, ~~\mathcal{S}_{+}^{\text{in}} = \{{\tilde{\mathbf{x}}}_i \in \mathcal{S}_{\text{wild}}^{\text{in}}:\tilde{\mathbf{x}}_i \leq T\}.
    \end{split}
    \end{equation*}
    Then 
    \begin{equation*}
        S_{T} = \mathcal{S}_{-}^{\text{out}}  \cup \mathcal{S}_{-}^{\text{in}}, ~~~  S_{\text{wild}}^{\text{out}} = \mathcal{S}_{-}^{\text{out}}  \cup \mathcal{S}_{+}^{\text{out}}.
    \end{equation*}
$~$
\\

Let 
    \begin{equation*}
    \begin{split}
    \Delta(n,m) =  \frac{1-\min\{1,\Delta_{\zeta}^{\eta}/\pi\}}{1-T/M'} &+ O \Big( \frac{\beta_1 M\sqrt{d}}{1-T/M'}\sqrt{\frac{1}{{\pi}^2n}}\Big)\\ &+ O \Big( \frac{{\beta_1 M\sqrt{d}}+ \sqrt{1-\pi}\Delta_{\zeta}^{\eta}/\pi}{1-T/M'}\sqrt{\frac{1}{{\pi}^2(1-\pi)m}}\Big).
    \end{split}
    \end{equation*}
     \begin{equation*}
    \begin{split}
    \delta(n,m) = \frac{8\beta_1 R_{\text{in}}^*}{T}+ O \Big ( \frac{\beta_1 M\sqrt{d}}{T}\sqrt{\frac{1}{n}}\Big )+   O \Big ( \frac{\beta_1 M\sqrt{d}}{T}\sqrt{\frac{1}{m}}\Big ).
    \end{split}
    \end{equation*}
    $~$
\\
\begin{lemma}\label{Error-0}
    Under the conditions of Theorem \ref{MainT-1}, with the probability at least $9/10$,
    \begin{equation*}
      |\mathcal{S}_T|\leq |\mathcal{S}_{-}^{\text{in}}|+ |\mathcal{S}_{\text{wild}}^{\text{out}}|\leq \delta(n,m) |\mathcal{S}_{\text{wild}}^{\text{in}}|+|\mathcal{S}_{\text{wild}}^{\text{out}}|,
    \end{equation*}
     \begin{equation*}
      |\mathcal{S}_T|\geq |\mathcal{S}_{-}^{\text{out}}| \geq \big [1- \Delta(n,m)\big ]|\mathcal{S}_{\text{wild}}^{\text{out}}|.
    \end{equation*}
     \begin{equation*}
     |\mathcal{S}_{+}^{\text{out}}| \leq \Delta(n,m)|\mathcal{S}_{\text{wild}}^{\text{out}}|.
    \end{equation*}
\end{lemma}

\begin{proof}[Proof of Lemma \ref{Error-0}]
  It is a conclusion of Theorem \ref{MainT-1}.
\end{proof}
$~$
\\
\begin{lemma}\label{Error-1}
    Under the conditions of Theorem \ref{MainT-1}, with the probability at least $9/10$,
    \begin{equation*}
  \frac{-\delta(n,m)|\mathcal{S}_{\text{wild}}^{\text{in}}|}{[\delta(n,m)|\mathcal{S}_{\text{wild}}^{\text{in}}| + |\mathcal{S}_{\text{wild}}^{\text{out}}|]|\mathcal{S}_{\text{wild}}^{\text{out}}|}   \leq \frac{1}{|\mathcal{S}_T|} - \frac{1}{|\mathcal{S}_{\text{wild}}^{\text{out}}|} \leq \frac{\Delta(n,m)}{[1-\Delta(n,m)]|\mathcal{S}_{\text{wild}}^{\text{out}}|}.
    \end{equation*}
\end{lemma}

\begin{proof}[Proof of Lemma \ref{Error-1}]
  This can be conclude by Lemma \ref{Error-0} directly.
\end{proof}

$~$
\\

\begin{lemma}\label{Error-2}
    Under the conditions of Theorem \ref{MainT-1}, with the probability at least $9/10$,
    \begin{equation*}
        R_{\mathcal{S}_T}^-(\mathbf{g}_{\boldsymbol{\theta}}) - R_{\mathcal{S}_{\text{wild}}^{\text{out}}}^-(\mathbf{g}_{\boldsymbol{\theta}})\leq \frac{L\Delta(n,m)}{[1-\Delta(n,m)]} +\frac{L\delta(n,m)}{[1-\Delta(n,m)]}\cdot \big(\frac{1-\pi}{\pi}+O\big(\sqrt{\frac{1}{\pi^4 m}}\big)\big),
    \end{equation*}
    \begin{equation*}
    R_{\mathcal{S}_{\text{wild}}^{\text{out}}}^-(\mathbf{g}_{\boldsymbol{\theta}})-R_{\mathcal{S}_T}^-(\mathbf{g}_{\boldsymbol{\theta}})\leq \frac{L\Delta(n,m)}{[1-\Delta(n,m)]} + L\Delta(n,m).
    \end{equation*}
\end{lemma}

\begin{proof}[Proof of Lemma \ref{Error-2}]
  It is clear that
    \begin{equation*}
    \begin{split}
        R_{\mathcal{S}_T}^-(\mathbf{g}_{\boldsymbol{\theta}}) - R_{\mathcal{S}_{\text{wild}}^{\text{out}}}^-(\mathbf{g}_{\boldsymbol{\theta}})= &\frac{|\mathcal{S}_{-}^{\text{out}}|}{|\mathcal{S}_{T}|} R_{\mathcal{S}_{-}^{\text{out}}}^-(\mathbf{g}_{\boldsymbol{\theta}})+\frac{|\mathcal{S}_{-}^{\text{in}}|}{|\mathcal{S}_{T}|} R_{\mathcal{S}_{-}^{\text{in}}}^-(\mathbf{g}_{\boldsymbol{\theta}})\\ -&\frac{|\mathcal{S}_{-}^{\text{out}}|}{|\mathcal{S}^{\text{out}}_{\text{wild}}|} R_{\mathcal{S}_{-}^{\text{out}}}^-(\mathbf{g}_{\boldsymbol{\theta}})-\frac{|\mathcal{S}_{+}^{\text{out}}|}{|\mathcal{S}^{\text{out}}_{\text{wild}}|} R_{\mathcal{S}_{+}^{\text{out}}}^-(\mathbf{g}_{\boldsymbol{\theta}}) \\  \leq &\frac{L\Delta(n,m)}{[1-\Delta(n,m)]} +\frac{L\delta(n,m)|\mathcal{S}_{\text{wild}}^{\text{in}}|}{[1-\Delta(n,m)]|\mathcal{S}_{\text{wild}}^{\text{out}}]}\\\leq &\frac{L\Delta(n,m)}{[1-\Delta(n,m)]} +\frac{L\delta(n,m)}{[1-\Delta(n,m)]}\cdot \big(\frac{1-\pi}{\pi}+O\big(\sqrt{\frac{1}{\pi^4 m}}\big)\big).
        \end{split}
    \end{equation*}
    \begin{equation*}
    \begin{split}
     R_{\mathcal{S}_{\text{wild}}^{\text{out}}}^-(\mathbf{g}_{\boldsymbol{\theta}}) -  R_{\mathcal{S}_T}^-(\mathbf{g}_{\boldsymbol{\theta}}) = &-\frac{|\mathcal{S}_{-}^{\text{out}}|}{|\mathcal{S}_{T}|} R_{\mathcal{S}_{-}^{\text{out}}}^-(\mathbf{g}_{\boldsymbol{\theta}})-\frac{|\mathcal{S}_{-}^{\text{in}}|}{|\mathcal{S}_{T}|} R_{\mathcal{S}_{-}^{\text{in}}}^-(\mathbf{g}_{\boldsymbol{\theta}})\\ +&\frac{|\mathcal{S}_{-}^{\text{out}}|}{|\mathcal{S}^{\text{out}}_{\text{wild}}|} R_{\mathcal{S}_{-}^{\text{out}}}^-(\mathbf{g}_{\boldsymbol{\theta}})+\frac{|\mathcal{S}_{+}^{\text{out}}|}{|\mathcal{S}^{\text{out}}_{\text{wild}}|} R_{\mathcal{S}_{+}^{\text{out}}}^-(\mathbf{g}_{\boldsymbol{\theta}}) \\  \leq &\frac{L\Delta(n,m)}{[1-\Delta(n,m)]} + L\Delta(n,m).
        \end{split}
    \end{equation*}
\end{proof}

$~$
\\

\begin{lemma}\label{Error-3}
 Let $\Delta(T) = 1-\delta(T)$.   Under the conditions of Theorem \ref{MainT-1}, for any
 $\eta'>0$, when
 \begin{equation*}
     n = \Omega \Big (\frac{\tilde{M}^2d}{{\eta'}^2\pi^2(1-T/M')^2\Delta(T)^2} \Big),~~~m= \Omega \Big( \frac{\tilde{M}^2d\pi^2+\Delta_{\zeta}^{\eta}(1-\pi)}{{\eta'}^2 \pi^4(1-\pi) (1-T/M')^2\Delta(T)^2}\Big),
 \end{equation*}

 with the probability at least $9/10$,  
    \begin{equation*}
        R_{\mathcal{S}_T}^-(\mathbf{g}_{\boldsymbol{\theta}}) - R_{\mathcal{S}_{\text{wild}}^{\text{out}}}^-(\mathbf{g}_{\boldsymbol{\theta}})\leq \frac{L\Delta(n,m)}{(1-\eta')\Delta(T)} +\frac{L\delta(n,m)}{(1-\eta')\Delta(T)}\cdot \big(\frac{1-\pi}{\pi}+O\big(\sqrt{\frac{1}{\pi^4 m}}\big)\big),
    \end{equation*}
   \begin{equation*}
\begin{split}
R_{\mathcal{S}_{\text{wild}}^{\text{out}}}^-(\mathbf{g}_{\boldsymbol{\theta}})&-R_{\mathcal{S}_T}^-(\mathbf{g}_{\boldsymbol{\theta}}) \leq \frac{1.2L}{1-\delta(T)}\delta(T) +L\delta(T)\\ &+O\Big(\frac{L\beta_1 M\sqrt{d}(1+T)}{\min\{\pi,\Delta_{\zeta}^{\eta}\}T}\sqrt{\frac{1}{n}}\Big )+O\Big(\frac{L(\beta_1 M\sqrt{d}+\Delta_{\zeta}^{\eta})(1+T)}{\min\{\pi,\Delta_{\zeta}^{\eta}\}T}\sqrt{\frac{1}{\pi^2(1-\pi)m}}\Big ).
\end{split}
\end{equation*}
\end{lemma}

\begin{proof}[Proof of Lemma \ref{Error-3}]
  This can be concluded by Lemma \ref{Error-2} and by the fact that $(1-\eta')\Delta(T) \geq 1-\Delta(n,m)$ directly.
\end{proof}

$~$
\\

\begin{lemma}\label{Error-4}
 Let $\Delta(T) = 1-\delta(T)$.   Under the conditions of Theorem \ref{MainT-1}, when
 \begin{equation*}
     n = \Omega \Big (\frac{\tilde{M}^2d}{\min \{\pi,\Delta_{\zeta}^{\eta}\}^2} \Big),~~~m= \Omega \Big( \frac{\tilde{M}^2d+\Delta_{\zeta}^{\eta}}{\pi^2(1-\pi)\min \{\pi,\Delta_{\zeta}^{\eta}\}^2}\Big),
 \end{equation*}

 with the probability at least $9/10$, for any $0<T<0.9 M' \min \{1,\Delta_{\zeta}^{\eta}/\pi\}$, 
   \begin{equation*}
\begin{split}
    R_{\mathcal{S}_T}^-(\mathbf{g}_{\boldsymbol{\theta}}) - R_{\mathcal{S}_{\text{wild}}^{\text{out}}}^-(&\mathbf{g}_{\boldsymbol{\theta}}) \leq \frac{1.2L}{1-\delta(T)}\delta(T) +\frac{9L\beta_1(1-\pi)}{T\pi (1-\delta(T))}R_{\text{in}}^*\\ &+O\Big(\frac{L\beta_1 M\sqrt{d}(1+T)}{\min\{\pi,\Delta_{\zeta}^{\eta}\}T}\sqrt{\frac{1}{n}}\Big )+O\Big(\frac{L(\beta_1 M\sqrt{d}+\Delta_{\zeta}^{\eta})(1+T)}{\min\{\pi,\Delta_{\zeta}^{\eta}\}T}\sqrt{\frac{1}{\pi^2(1-\pi)m}}\Big ),
    \end{split}
\end{equation*}
\begin{equation*}
\begin{split}
R_{\mathcal{S}_{\text{wild}}^{\text{out}}}^-(\mathbf{g}_{\boldsymbol{\theta}})-R_{\mathcal{S}_T}^-(&\mathbf{g}_{\boldsymbol{\theta}}) \leq \frac{1.2L}{1-\delta(T)}\delta(T) +L\delta(T)\\ &+O\Big(\frac{L\beta_1 M\sqrt{d}(1+T)}{\min\{\pi,\Delta_{\zeta}^{\eta}\}T}\sqrt{\frac{1}{n}}\Big )+O\Big(\frac{L(\beta_1 M\sqrt{d}+\Delta_{\zeta}^{\eta})(1+T)}{\min\{\pi,\Delta_{\zeta}^{\eta}\}T}\sqrt{\frac{1}{\pi^2(1-\pi)m}}\Big ).
\end{split}
\end{equation*}
\end{lemma}

\begin{proof}[Proof of Lemma \ref{Error-4}]
 Using Lemma \ref{Error-4} with $\eta = 8/9$, we obtain that
 \begin{equation*}
 \begin{split}
    & R_{\mathcal{S}_T}^-(\mathbf{g}_{\boldsymbol{\theta}}) - R_{\mathcal{S}_{\text{wild}}^{\text{out}}}^-(\mathbf{g}_{\boldsymbol{\theta}}) \\ \leq &\frac{1.2L\delta(T)}{1-\delta(T)} +\frac{9L\beta_1(1-\pi)}{T\pi (1-\delta(T))}R_{\text{in}}^*+\frac{L\epsilon(n)}{\Delta(T)}+\frac{L\bar{\epsilon}(n)}{\pi \Delta(T)}+\frac{L\epsilon(m)}{\Delta(T)}\\+&\frac{8L\beta_1 R^*_{\text{in}}}{\pi^2\Delta(T)T}O\big( \sqrt{\frac{1}{m}}\big )+\frac{L\bar{\epsilon}(m)}{{\pi}^2\Delta(T)}O\big( \sqrt{\frac{1}{m}}\big ) +\frac{L\bar{\epsilon}(m)}{{\pi}\Delta(T)} +\frac{L\bar{\epsilon}(n)}{{\pi}^2\Delta(T)}O\big( \sqrt{\frac{1}{m}}\big ),
     \end{split}
 \end{equation*}
 where
 \begin{equation*}
     \epsilon(n) = O \Big( \frac{\beta_1 M\sqrt{d}}{1-T/M'}\sqrt{\frac{1}{{\pi}^2n}}\Big),~~~~\epsilon(m) =O \Big( \frac{{\beta_1 M\sqrt{d}}+ \sqrt{1-\pi}\Delta_{\zeta}^{\eta}/\pi}{1-T/M'}\sqrt{\frac{1}{{\pi}^2(1-\pi)m}}\Big).
 \end{equation*}
 \begin{equation*}
 \bar{\epsilon}(n) = O \Big ( \frac{\beta_1 M\sqrt{d}}{T}\sqrt{\frac{1}{n}}\Big ),~~~ \bar{\epsilon}(m) = O \Big ( \frac{\beta_1 M\sqrt{d}}{T}\sqrt{\frac{1}{(1-\pi)m}}\Big ).
 \end{equation*}
Using the condition that $0<T<0.9 M' \min \{1,\Delta_{\zeta}^{\eta}/\pi\}$, we have
\begin{equation*}
    \frac{1}{\Delta(T)}\big [\frac{1}{T}+\frac{1}{1-T/{M'}}\big ] \leq O\big(\frac{T+1}{\min\{1,\Delta_{\zeta}^{\eta}/\pi\}T}\big ).
\end{equation*}

 Then, we obtain that 
\begin{equation*}
\begin{split}
    R_{\mathcal{S}_T}^-(\mathbf{g}_{\boldsymbol{\theta}}) - R_{\mathcal{S}_{\text{wild}}^{\text{out}}}^-&(\mathbf{g}_{\boldsymbol{\theta}}) \leq \frac{1.2L}{1-\delta(T)}\delta(T) +\frac{9L\beta_1(1-\pi)}{T\pi (1-\delta(T))}R_{\text{in}}^*\\ &+O\Big(\frac{L\beta_1 M\sqrt{d}(1+T)}{\min\{\pi,\Delta_{\zeta}^{\eta}\}T}\sqrt{\frac{1}{n}}\Big )+O\Big(\frac{L(\beta_1 M\sqrt{d}+\Delta_{\zeta}^{\eta})(1+T)}{\min\{\pi,\Delta_{\zeta}^{\eta}\}T}\sqrt{\frac{1}{\pi^2(1-\pi)m}}\Big ).
    \end{split}
\end{equation*}
Using the similar strategy, we can obtain that 
\begin{equation*}
\begin{split}
R_{\mathcal{S}_{\text{wild}}^{\text{out}}}^-(\mathbf{g}_{\boldsymbol{\theta}})-R_{\mathcal{S}_T}^-&(\mathbf{g}_{\boldsymbol{\theta}}) \leq \frac{1.2L}{1-\delta(T)}\delta(T) +L\delta(T)\\ &+O\Big(\frac{L\beta_1 M\sqrt{d}(1+T)}{\min\{\pi,\Delta_{\zeta}^{\eta}\}T}\sqrt{\frac{1}{n}}\Big )+O\Big(\frac{L(\beta_1 M\sqrt{d}+\Delta_{\zeta}^{\eta})(1+T)}{\min\{\pi,\Delta_{\zeta}^{\eta}\}T}\sqrt{\frac{1}{\pi^2(1-\pi)m}}\Big ).
\end{split}
\end{equation*}
\end{proof}

$~$
\\

\begin{lemma}\label{Error-5}
    Under the conditions of Theorem \ref{MainT-1}, when
 \begin{equation*}
     n = \Omega \Big (\frac{\tilde{M}^2d}{\min \{\pi,\Delta_{\zeta}^{\eta}\}^2} \Big),~~~m= \Omega \Big( \frac{\tilde{M}^2d+\Delta_{\zeta}^{\eta}}{\pi^2(1-\pi)\min \{\pi,\Delta_{\zeta}^{\eta}\}^2}\Big),
 \end{equation*}

 with the probability at least $0.895$, for any $0<T<0.9 M' \min \{1,\Delta_{\zeta}^{\eta}/\pi\}$, 
   \begin{equation*}
\begin{split}
    R_{\mathcal{S}_T}^-(\mathbf{g}_{\boldsymbol{\theta}}) &- R_{\mathbb{P}_{\text{out}}}^-(\mathbf{g}_{\boldsymbol{\theta}})\leq \frac{1.2L}{1-\delta(T)}\delta(T) +\frac{9L\beta_1(1-\pi)}{T\pi (1-\delta(T))}R_{\text{in}}^*\\ &+O\Big(\frac{L\beta_1 M\sqrt{d}(1+T)}{\min\{\pi,\Delta_{\zeta}^{\eta}\}T}\sqrt{\frac{1}{n}}\Big )+O\Big(\frac{L\max \{\beta_1 M\sqrt{d},\sqrt{d'},\Delta_{\zeta}^{\eta}\}(1+T)}{\min\{\pi,\Delta_{\zeta}^{\eta}\}T}\sqrt{\frac{1}{\pi^2(1-\pi)m}}\Big ),
    \end{split}
\end{equation*}
\begin{equation*}
\begin{split}
R_{\mathbb{P}_{\text{out}}}^-(\mathbf{g}_{\boldsymbol{\theta}})&-R_{\mathcal{S}_T}^-(\mathbf{g}_{\boldsymbol{\theta}}) \leq \frac{1.2L}{1-\delta(T)}\delta(T) +L\delta(T)\\ &+O\Big(\frac{L\beta_1 M\sqrt{d}(1+T)}{\min\{\pi,\Delta_{\zeta}^{\eta}\}T}\sqrt{\frac{1}{n}}\Big )+O\Big(\frac{L\max \{\beta_1 M\sqrt{d},\sqrt{d'},\Delta_{\zeta}^{\eta}\}(1+T)}{\min\{\pi,\Delta_{\zeta}^{\eta}\}T}\sqrt{\frac{1}{\pi^2(1-\pi)m}}\Big ).
\end{split}
\end{equation*}
\end{lemma}
\begin{proof}
    By Lemmas \ref{UC-I} and \ref{SC-I}, under the condition of this lemma, we can obtain that with the high probability,
     \begin{equation*}
    \big| R^{-}_{\mathbb{P}_{\text{out}}}(\mathbf{g}_{\boldsymbol{\theta}})-R^{-}_{\mathcal{S}_{\text{wild}}^{\text{out}}}(\mathbf{g}_{\boldsymbol{\theta}}) \big| \leq O\big (L\sqrt{\frac{d'}{ \pi m}}\big ).
 \end{equation*}
 Then by Lemma \ref{Error-4}, we can prove this lemma.
\end{proof}

   $~~~$

\section{Empirical Verification on the Main Theorems}
\label{sec:verification_discrepancy}

 \textbf{Verification on the regulatory conditions.} In  Table~\ref{tab:discrepancy_value}, we provide empirical verification on whether the distribution discrepancy $\zeta$ satisfies the necessary regulatory condition in Theorem~\ref{The-1.1}, i.e., $\zeta\geq 2.011\sqrt{8\beta_1 R_{\rm in}^*} +1.011 \sqrt{\pi}$. We use  \textsc{Cifar-100} as ID and \textsc{Textures} as the wild OOD data.

Since $R^{*}_{{\rm in}}$ is the optimal ID risk, i.e., $R^{*}_{{\rm in}}=\min_{\mathbf{w}\in \mathcal{W}} \mathbb{E}_{(\mathbf{x},y)\sim \mathbb{P}_{\mathcal{X}\mathcal{Y}}} \ell(\mathbf{h}_{\mathbf{w}}(\mathbf{x}),y)$, it can be a small value close to 0 in over-parametrized neural networks~\citep{frei2022benign,bartlett2020benign}.  Therefore, we can omit the value of $2.011\sqrt{8\beta_1 R_{\rm in}^*}$. The  empirical result shows that $\zeta$ can easily satisfy the regulatory condition in Theorem~\ref{The-1.1}, which means our bound is useful in practice.

\begin{table}[!h]
    \centering
     \caption{\small Discrepancy value $\zeta$ with different ratios $\pi$.}
    \begin{tabular}{c|ccccccc}
    \hline
     $\pi$  &  0.05 & 0.1 &  0.2 & 0.5 & 0.7 & 0.9 & 1.0 \\
     \hline
     $\zeta$& 0.91&1.09 & 1.43 & 2.49 & 3.16& 3.86&4.18 \\
     $1.011 \sqrt{\pi}$ &0.23 &  0.32 &0.45&  0.71& 0.84& 0.96& 1.01\\
      \hline
    \end{tabular}
   
    \label{tab:discrepancy_value}
\end{table}
 \textbf{Verification on the filtering errors and OOD detection results with varying $\pi$.} In Table~\ref{tab:err_trend}, we empirically verify the value of ${\rm ERR}_{\rm out}$ and ${\rm ERR}_{\rm in}$ in Theorem~\ref{MainT-1} and the corresponding OOD detection results with various mixing ratios $\pi$. We use  \textsc{Cifar-100} as ID and \textsc{Textures} as the wild OOD data. The result  aligns well with our observation of the bounds presented in Section~\ref{sec:ana_1} of the main paper. 

\begin{table}[!h]
    \centering
     \caption{\small The values of ${\rm ERR}_{\rm in}$, ${\rm ERR}_{\rm out}$ and  the OOD detection results with various mixing ratios $\pi$.}
    \begin{tabular}{c|ccccccc}
    \hline
     $\pi$  &  0.05 & 0.1 &  0.2 & 0.5& 0.7 & 0.9 & 1.0 \\
     \hline
     ${\rm ERR}_{\rm out}$& 0.37 & 0.30 & 0.22 & 0.20 &0.23 & 0.26  & 0.29\\
${\rm ERR}_{\rm in}$ & 0.031 & 0.037& 0.045 & 0.047 & 0.047&0.048 &0.048 \\
FPR95& 5.77 &   5.73 & 5.71& 5.64& 5.79 & 5.88   &  5.92 \\
      \hline
    \end{tabular}
   
    \label{tab:err_trend}
\end{table}

\section{Additional Experimental Details}
\label{sec:detail_app}
\textbf{Dataset details.} For Table~\ref{tab:c100}, following WOODS~\citep{katzsamuels2022training}, we split the data as follows: We use 70\% of the OOD datasets (including \textsc{Textures}, \textsc{Places365}, \textsc{Lsun-Resize} and \textsc{Lsun-C}) for the OOD data in the wild. We use the remaining samples for testing-time OOD detection.  For \textsc{Svhn}, we use the training set for the OOD data in the wild and use the test set for evaluation. 

\textbf{Training details.} Following WOODS~\citep{katzsamuels2022training}, we use Wide ResNet~\citep{ZagoruykoK16} with 40 layers and widen factor of 2 for the classification model $\*h_\*w$. We train the ID classifier $\*h_\*w$ using stochastic gradient descent with a momentum of 0.9, weight decay of 0.0005, and an initial learning rate of 0.1. We train for 100 epochs using cosine learning rate decay,
a batch size of 128, and a dropout rate of 0.3. For the OOD classifier $\*g_{\boldsymbol{\theta}}$, we load the pre-trained ID classifier of $\*h_\*w$ and add an additional linear layer which takes in the penultimate-layer features for binary classification. We set the initial learning rate to 0.001 and fine-tune for 100 epochs by Eq.~\ref{eq:reg_loss}. We add the binary classification loss to the ID classification loss and set the loss weight for  binary classification to 10. The other  details are kept the same as training $\*h_\*w$.

\section{Additional  Results on \textsc{Cifar-10}}
\label{sec:c10_app}
In Table~\ref{tab:c10_app}, we compare our \model with  baselines with the ID data to be \textsc{Cifar-10}, where the strong performance of \model still holds.
\begin{table}[!h]
  \centering
  \small
  \vspace{-1em}
  \caption{\small OOD detection performance on \textsc{Cifar-10} as ID. All methods are trained on Wide ResNet-40-2 for 100 epochs with $\pi=0.1$. For each dataset, we create corresponding wild mixture distribution
$\mathbb{P}_\text{wild} := (1 - \pi) \mathbb{P}_\text{in} + \pi \mathbb{P}_\text{out}$ for training and test on the corresponding OOD dataset. Values are percentages \textbf{averaged over 10 runs}. {Bold} numbers highlight the best results. Table format credit to~\citep{katzsamuels2022training}. }
    \scalebox{0.68}{
    \begin{tabular}{cccccccccccccc}
    \toprule
    \multirow{3}[4]{*}{Methods} & \multicolumn{12}{c}{OOD Datasets}                                                             & \multirow{3}[4]{*}{ID ACC} \\
    \cmidrule{2-13}
          & \multicolumn{2}{c}{\textsc{Svhn}} & \multicolumn{2}{c}{\textsc{Places365}} & \multicolumn{2}{c}{\textsc{Lsun-C}} & \multicolumn{2}{c}{\textsc{Lsun-Resize}} & \multicolumn{2}{c}{\textsc{Textures}}& \multicolumn{2}{c}{Average} &  \\
\cmidrule{2-13}          & FPR95 & AUROC & FPR95 & AUROC & FPR95 & AUROC & FPR95 & AUROC & FPR95 & AUROC & FPR95 & AUROC &  \\
    \hline
    \multicolumn{14}{c}{With $\mathbb{P}_{\text{in}}$ only} \\

    MSP   &  48.49 & 91.89 &59.48 & 88.20& 30.80 & 95.65 & 52.15 & 91.37& 59.28&  88.50 & 50.04 & 91.12 & 94.84\\
ODIN& 33.35 & 91.96 &  57.40& 84.49&15.52 & 97.04  & 26.62 & 94.57 & 49.12 & 84.97 & 36.40 & 90.61 & 94.84 
\\
Mahalanobis & 12.89 &97.62     & 68.57&  84.61& 39.22&  94.15& 42.62&  93.23 & 15.00 & 97.33& 35.66& 93.34 &94.84
\\
Energy & 35.59  &90.96 & 40.14 &  89.89 & 8.26 & 98.35  & 27.58&  94.24 & 52.79 &85.22& 32.87 &91.73& 94.84\\
KNN& 24.53&	95.96&	25.29	&95.69&	25.55	&95.26&	27.57	&94.71&	50.90	&89.14&	30.77	&94.15&	94.84\\
 ReAct&    40.76 & 89.57 & 41.44 & 90.44 & 14.38 & 97.21 & 33.63 & 93.58 &  53.63 & 86.59 & 36.77 & 91.48 & 94.84 \\
DICE & 35.44 & 89.65 & 46.83 & 86.69 & 6.32 & 98.68  &28.93 & 93.56& 53.62 & 82.20 & 34.23 & 90.16 & 94.84\\

 ASH   &  6.51&  98.65& 48.45 & 88.34 & 0.90&  99.73&  4.96&  98.92  &24.34 &95.09&  17.03& 96.15   & 94.84\\
CSI& 17.30 &97.40&  34.95 &93.64& 1.95 &99.55 &  12.15& 98.01&  20.45& 95.93 & 17.36 &96.91 &  94.17 \\
 
KNN+ & 2.99 & 99.41 & 24.69 & 94.84  &  2.95 & 99.39& 11.22 & 97.98  & 9.65 & 98.37& 10.30 &  97.99 & 93.19\\
\hline
    \multicolumn{14}{c}{ With $\mathbb{P}_{\text{in}}$ and $\mathbb{P}_{\text{wild}}$ } \\
    OE& 0.85 &  99.82 & 23.47 &  94.62 & 1.84 &  99.65 & 0.33&  99.93 & 10.42&  98.01 & 7.38 & 98.41 &  94.07\\
  Energy (w/ OE) & 4.95 &  98.92  & 17.26 &  95.84 & 1.93 & 99.49 &5.04 &  98.83 & 13.43 & 96.69 & 8.52 & 97.95 &  94.81\\
WOODS& 0.15& 99.97 &  12.49& 97.00 & 0.22&{99.94} &{0.03}&  \textbf{99.99} &  5.95&  98.79& 3.77&  99.14 & 94.84\\

\rowcolor[HTML]{EFEFEF}\model &   \textbf{0.02} & \textbf{99.98}   &  \textbf{2.57} & \textbf{99.24}   & \textbf{0.07} & \textbf{99.99}   &\textbf{0.01 }& \textbf{99.99}   & \textbf{0.90} & \textbf{99.74}  & \textbf{0.71}&  \textbf{99.78}&  93.65 \\
\rowcolor[HTML]{EFEFEF}  (Ours)& $^{\pm}$0.00 & $^{\pm}$0.00 & $^{\pm}$0.03 & $^{\pm}$0.00 & $^{\pm}$0.01 & $^{\pm}$0.00 & $^{\pm}$0.00 & $^{\pm}$0.00& $^{\pm}$0.02 & $^{\pm}$0.01 & $^{\pm}$0.01& $^{\pm}$0.00 & $^{\pm}$0.57\\

   \hline
   
\end{tabular}}

    \label{tab:c10_app}
\end{table}

\section{Additional  Results on Unseen OOD Datasets }

\label{sec:mixing_ratio}
In Table~\ref{tab:mixing_ratio}, we evaluate \model on unseen OOD datasets, which are different from the OOD data we use in the wild. Here we consistently use \textsc{300K Random Images} as the unlabeled wild dataset and \textsc{Cifar-10} as labeled in-distribution data.  We use the 5 different OOD datasets (\textsc{Textures}, \textsc{Places365}, \textsc{Lsun-Resize}, \textsc{Svhn} and \textsc{Lsun-C}) for evaluation. When evaluating on \textsc{300K Random Images}, we use 99\% of the  \textsc{300K Random Images} dataset~\citep{hendrycks2018deep} as the wild OOD data   and the remaining 1\% of the dataset for evaluation. $\pi$ is set to 0.1. We observe that \model can perform competitively on unseen datasets as well, compared to the  most relevant baseline WOODS.

\begin{table}[!h]
  \centering
  \small
  \vspace{-1em}
  \caption{\small Evaluation on unseen OOD datasets.  We use  \textsc{Cifar-10} as ID  and \textsc{300K Random Images} as the wild data. All methods are trained on Wide ResNet-40-2 for 50 epochs.  {Bold} numbers highlight the best results.  }
    \scalebox{0.67}{
    \begin{tabular}{cccccccccccccc}
    \toprule
    \multirow{3}[4]{*}{Methods} & \multicolumn{12}{c}{OOD Datasets}                                                             & \multirow{3}[4]{*}{ID ACC} \\
    \cmidrule{2-13}
          & \multicolumn{2}{c}{\textsc{Svhn}} & \multicolumn{2}{c}{\textsc{Places365}} & \multicolumn{2}{c}{\textsc{Lsun-C}} & \multicolumn{2}{c}{\textsc{Lsun-Resize}} & \multicolumn{2}{c}{\textsc{Textures}}& \multicolumn{2}{c}{\textsc{300K Rand. Img.}}  &  \\
\cmidrule{2-13}          & FPR95 & AUROC & FPR95 & AUROC & FPR95 & AUROC & FPR95 & AUROC & FPR95 & AUROC & FPR95 & AUROC &  \\
    \midrule

    OE & 13.18&  97.34 & 30.54&  93.31 & 5.87&  98.86& 14.32&  97.44 &25.69&  94.35 & 30.69& 92.80  & 94.21\\
    Energy (w/ OE) &  8.52&  98.13 & 23.74 & 94.26 &2.78&  99.38 & 9.05&  98.13 & 22.32&  94.72 & 24.59&  93.99 & 94.54\\
WOODS & 5.70& \textbf{98.54} &  19.14 &95.74 &\textbf{1.31}&  \textbf{99.66} &4.13& \textbf{99.01} & 17.92& 96.43 & 19.82&  95.52 & 94.74\\
\rowcolor[HTML]{EFEFEF} \model & \textbf{4.94} & 97.53  & \textbf{14.76} & \textbf{96.25} & 2.73 & 98.23 & \textbf{3.46} & 98.15 & \textbf{11.60} & \textbf{97.21} & \textbf{10.20} & \textbf{97.23} & 93.48 \\

   \hline
   
\end{tabular}}
    \label{tab:mixing_ratio}
\end{table}

\textcolor{black}{Following~\citep{he2023topological}, we use the \textsc{Cifar-100} as ID, Tiny ImageNet-crop (TINc)/Tiny ImageNet-resize (TINr) dataset as the OOD in the wild dataset and TINr/TINc as the test OOD. The comparison with baselines is shown below, where the strong performance of SAL still holds.}

\begin{table}[!h]
  \centering
  \small

  \caption{\small \textcolor{black}{Additional results on unseen OOD datasets with \textsc{Cifar-100} as ID. {Bold} numbers are superior results.  }}
    \scalebox{0.88}{
    {\color{black}\begin{tabular}{ccccc}
    \toprule
    \multirow{3}[4]{*}{Methods} & \multicolumn{4}{c}{OOD Datasets}                                                   \\
    \cmidrule{2-5}
          & \multicolumn{2}{c}{\textsc{TINr}} & \multicolumn{2}{c}{\textsc{TINc}}  \\
\cmidrule{2-5}     
& FPR95 & AUROC & FPR95 & AUROC  \\
    \midrule
  STEP &72.31	&74.59	&48.68&	91.14\\
  TSL & 57.52&	82.29&	29.48	&94.62\\
\rowcolor[HTML]{EFEFEF} \model (Ours)&\textbf{43.11}	&\textbf{89.17}&	\textbf{19.30}	&\textbf{96.29} \\
   \hline
\end{tabular}}}
\end{table}

\section{Additional Results on Near OOD Detection}
\label{sec:near_ood}
In this section, we investigate the performance of  \model on  near OOD detection, which is a more challenging OOD detection scenario where the OOD data has a closer distribution to the in-distribution. Specifically, we use the \textsc{Cifar-10} as the in-distribution data and \textsc{Cifar-100} training set as the OOD data in the wild. During test time, we use the test set of \textsc{Cifar-100} as the OOD for evaluation. With a mixing ratio $\pi$ of 0.1, our \model achieves an FPR95 of 24.51\% and AUROC of 95.55\% compared to 38.92\% (FPR95) and 93.27\% (AUROC) of WOODS.

\textcolor{black}{In addition, we study near OOD detection in a different data setting, i.e., the first 50 classes of \textsc{Cifar-100} as ID and the last 50 classes as OOD. The comparison with the most competitive baseline WOODS is reported as follows.}

\begin{table}[!h]
    \centering
    \small
        \caption{\small\textcolor{black}{ Near OOD detection  with the first 50 classes of \textsc{Cifar-100} as ID and the last 50 classes as OOD. {Bold} numbers are superior results. }}
    {\color{black}\begin{tabular}{cccc}
    \hline
       \multirow{3}{*}{Methods} & \multicolumn{3}{c}{OOD dataset} \\
    \cline{2-4}
  & \multicolumn{3}{c}{\textsc{Cifar-50}} \\
  \cline{2-4}
       & FPR95 & AUROC &ID ACC \\
      \hline
      WOODS&41.28 &	89.74	 &74.17 \\
     \rowcolor[HTML]{EFEFEF} \model & \textbf{29.71} &	\textbf{93.13} &	73.86\\
          \hline
    \end{tabular}}
\end{table}

\section{Additional Results on Using Multiple Singular Vectors}
\label{sec:num_of_sing_vectors}
In this section, we ablate on the effect of using $c$ singular vectors to calculate the filtering score (Eq.~\ref{eq:score}). Specifically, we calculate the scores by projecting the gradient $\nabla \ell (\*h_{\*w_{\mathcal{S}^{\text{in}}}}\big(\tilde{\*x}_i), \widehat{y}_{\tilde{\*x}_i})-\bar{\nabla}$ for the wild data $\tilde{\*x}_i$ to each of the singular vectors.  The final filtering score is the average over the $c$ scores. The result is summarized in Table~\ref{tab:num_of_sing_vectors}. We observe that using the top 1 singular vector for projection achieves the best performance. As revealed in Eq.~\ref{eq:pca}, the top 1 singular vector $\*v$ maximizes the total distance from the projected gradients (onto the direction of $\*v$) to the origin (sum over all points in $\mathcal{S}_{\rm wild}$), where  outliers lie approximately close to and thus leads to a better separability between the ID and OOD in the wild. 

\begin{table}[!h]
    \centering
        \caption{\small The effect of the number of singular vectors used for the filtering score. Models are trained on Wide ResNet-40-2 for 100 epochs with $\pi=0.1$.  We use \textsc{Textures}  as the  wild OOD data and \textsc{Cifar-100} as the ID.}
    \begin{tabular}{c|cc}
    \hline
      Number of singular vectors $c$  & FPR95 & AUROC \\
      \hline
         1&\textbf{5.73} & \textbf{98.65}   \\
         2 &6.28 & 98.42  \\
         3 &  6.93 & 98.43\\
            4 & 7.07 & 98.37\\
               5 &  7.43 & 98.27\\
                  6 & 7.78 & 98.22 \\
          \hline
    \end{tabular}

    \label{tab:num_of_sing_vectors}
\end{table}

\section{Additional Results on Class-agnostic SVD}
In this section, we evaluate our \model by using class-agnostic SVD as opposed to class-conditional SVD as described in Section~\ref{sec:detect} of the main paper. Specifically,    we  maintain a class-conditional reference gradient $ \bar{\nabla}_k$, one for each class $k \in [1,K]$, estimated on ID samples belonging to class $k$. Different from calculating the singular vectors based on gradient matrix with $\mathbf{G}_k$ (containing gradient vectors of wild samples being predicted as class $k$), we formulate a single gradient matrix $\*G$ where each row is the vector $\nabla \ell (\*h_{\*w_{\mathcal{S}^{\text{in}}}}\big(\tilde{\*x}_i), \widehat{y}_{\tilde{\*x}_i})-\bar{\nabla}_{\widehat{y}_{\tilde{\*x}_i}}$, for $\tilde{\*x}_i \in \mathcal{S}_{\rm wild} $.  The result is shown in Table~\ref{tab:single_g}, which shows a similar performance compared with using class-conditional SVD.

\begin{table}[!h]
  \centering
  \small
  \vspace{-1em}
  \caption{\small The effect of using class-agnostic SVD. Models are trained on Wide ResNet-40-2 for 100 epochs with $\pi=0.1$. \textsc{Cifar-100} is the in-distribution data. {Bold} numbers are superior results.  }
    \scalebox{0.63}{
    \begin{tabular}{cccccccccccccc}
    \toprule
    \multirow{3}[4]{*}{Methods} & \multicolumn{12}{c}{OOD Datasets}                                                             & \multirow{3}[4]{*}{ID ACC} \\
    \cmidrule{2-13}
          & \multicolumn{2}{c}{\textsc{Svhn}} & \multicolumn{2}{c}{\textsc{Places365}} & \multicolumn{2}{c}{\textsc{Lsun-C}} & \multicolumn{2}{c}{\textsc{Lsun-Resize}} & \multicolumn{2}{c}{\textsc{Textures}} & \multicolumn{2}{c}{Average} &  \\
\cmidrule{2-13}          & FPR95 & AUROC & FPR95 & AUROC & FPR95 & AUROC & FPR95 & AUROC & FPR95 & AUROC & FPR95 & AUROC &  \\
    \midrule

\model (Class-agnostic SVD) &  0.12 & 99.43 & \textbf{3.27} & \textbf{99.21} & \textbf{0.04} & 99.92 & 0.03 & 99.27 & \textbf{5.18} & \textbf{98.77} &\textbf{1.73} & {99.32}  & 73.31\\

\rowcolor[HTML]{EFEFEF}\model &  \textbf{0.07} & \textbf{99.95} & {3.53} & {99.06} & {0.06} & \textbf{99.94} &   \textbf{0.02} & \textbf{99.95} & {5.73} & {98.65}  & {1.88}& \textbf{99.51}& 73.71 \\
   \hline
   
\end{tabular}}
    \label{tab:single_g}
\end{table}

\section{Additional Results on Post-hoc Filtering Score}
\label{sec:addition_results_posthoc}

We investigate the importance of training the binary classifier with the filtered candidate outliers for OOD detection in Tables~\ref{tab:c10_host_hoc} and~\ref{tab:c100_host_hoc}. Specifically, we calculate our filtering score directly for the test ID and OOD data on the model trained on the labeled ID set $\mathcal{S}^{\text{in}}$ only. The results are shown in the row "\model (Post-hoc)" in Tables~\ref{tab:c10_host_hoc} and~\ref{tab:c100_host_hoc}. Without explicit knowledge of the OOD data, the OOD detection performance degrades significantly compared to training an additional binary classifier (a 15.73\% drop on FPR95 for \textsc{Svhn} with \textsc{Cifar-10} as ID). However, the post-hoc filtering score can still outperform most of the baselines that use $\mathbb{P}_{\text{in}}$ only (\emph{c.f.} Table~\ref{tab:c100}), showcasing its effectiveness.

\begin{table}[!h]
  \centering
  \small
  \vspace{-1em}
  \caption{\small OOD detection results of using post-hoc filtering score on \textsc{Cifar-10} as ID. \model is trained on Wide ResNet-40-2 for 100 epochs with $\pi=0.1$.  {Bold} numbers are superior results.  }
    \scalebox{0.68}{
    \begin{tabular}{cccccccccccccc}
    \toprule
    \multirow{3}[4]{*}{Methods} & \multicolumn{12}{c}{OOD Datasets}                                                             & \multirow{3}[4]{*}{ID ACC} \\
    \cmidrule{2-13}
          & \multicolumn{2}{c}{\textsc{Svhn}} & \multicolumn{2}{c}{\textsc{Places365}} & \multicolumn{2}{c}{\textsc{Lsun-C}} & \multicolumn{2}{c}{\textsc{Lsun-Resize}} & \multicolumn{2}{c}{\textsc{Textures}}& \multicolumn{2}{c}{Average} &  \\
\cmidrule{2-13}          & FPR95 & AUROC & FPR95 & AUROC & FPR95 & AUROC & FPR95 & AUROC & FPR95 & AUROC & FPR95 & AUROC &  \\
    \midrule

\model (Post-hoc) & 15.75 & 93.09  &  23.18 & 86.35  & 6.28 & 96.72 & 15.59 & 89.83  & 23.63 & 87.72  & 16.89 &90.74& 94.84\\

\rowcolor[HTML]{EFEFEF}\model &\textbf{0.02} & \textbf{99.98}   &  \textbf{2.57} & \textbf{99.24}   & \textbf{0.07} & \textbf{99.99}   &\textbf{0.01 }& \textbf{99.99}   & \textbf{0.90} & \textbf{99.74}  & \textbf{0.71}&  \textbf{99.78}&  93.65 \\
   \hline
   
\end{tabular}}
    \label{tab:c10_host_hoc}
\end{table}

\begin{table}[!h]
  \centering
  \small
  \vspace{-1em}
  \caption{\small OOD detection results of using post-hoc filtering score on \textsc{Cifar-100} as ID. \model is trained on Wide ResNet-40-2 for 100 epochs  with $\pi=0.1$. {Bold} numbers are superior results.  }
    \scalebox{0.68}{
    \begin{tabular}{cccccccccccccc}
    \toprule
    \multirow{3}[4]{*}{Methods} & \multicolumn{12}{c}{OOD Datasets}                                                             & \multirow{3}[4]{*}{ID ACC} \\
    \cmidrule{2-13}
          & \multicolumn{2}{c}{\textsc{Svhn}} & \multicolumn{2}{c}{\textsc{Places365}} & \multicolumn{2}{c}{\textsc{Lsun-C}} & \multicolumn{2}{c}{\textsc{Lsun-Resize}} & \multicolumn{2}{c}{\textsc{Textures}}& \multicolumn{2}{c}{Average} &  \\
\cmidrule{2-13}          & FPR95 & AUROC & FPR95 & AUROC & FPR95 & AUROC & FPR95 & AUROC & FPR95 & AUROC & FPR95 & AUROC &  \\
    \midrule

\model (post-hoc) &  39.75 & 81.47 &  35.94 & 84.53   & 23.22 & 90.90   &  32.59 & 87.12 & 36.38 & 83.25   & 33.58 & 85.45 & 75.96\\

\rowcolor[HTML]{EFEFEF}\model &  \textbf{0.07} & \textbf{99.95} & \textbf{3.53} & \textbf{99.06} & \textbf{0.06} & \textbf{99.94} &   \textbf{0.02} & \textbf{99.95} & \textbf{5.73} & \textbf{98.65}  & \textbf{1.88}& \textbf{99.51}& 73.71 \\
   \hline
   
\end{tabular}}
    \label{tab:c100_host_hoc}
\end{table}

\section{Additional Results on Leveraging the Candidate ID data}
\label{sec:addition_results_can_id}
\textcolor{black}{
In this section, we investigate the effect of incorporating the filtered wild data which has a score smaller than the threshold $T$ (candidate ID data) for training the binary classifier $\*g_{\boldsymbol{\theta}}$. Specifically, the candidate ID data and the labeled ID data are used jointly to train the binary classifier. The comparison with \model on CIFAR-100 is shown as follows:}

\begin{table}[!h]
  \centering
  \small
  \vspace{-1em}
  \caption{\small \textcolor{black}{OOD detection results of selecting candidate ID data for training  on \textsc{Cifar-100} as ID. \model is trained on Wide ResNet-40-2 for 100 epochs  with $\pi=0.1$. {Bold} numbers are superior results.  }}
    \scalebox{0.68}{
    \begin{tabular}{cccccccccccccc}
    \toprule
    \multirow{3}[4]{*}{Methods} & \multicolumn{12}{c}{OOD Datasets}                                                             & \multirow{3}[4]{*}{ID ACC} \\
    \cmidrule{2-13}
          & \multicolumn{2}{c}{\textsc{Svhn}} & \multicolumn{2}{c}{\textsc{Places365}} & \multicolumn{2}{c}{\textsc{Lsun-C}} & \multicolumn{2}{c}{\textsc{Lsun-Resize}} & \multicolumn{2}{c}{\textsc{Textures}}& \multicolumn{2}{c}{Average} &  \\
\cmidrule{2-13}          & FPR95 & AUROC & FPR95 & AUROC & FPR95 & AUROC & FPR95 & AUROC & FPR95 & AUROC & FPR95 & AUROC &  \\
    \midrule

  Candidate ID data   &1.23 &99.87&\textbf{2.62}&\textbf{99.18}& \textbf{0.04}&\textbf{99.95}& \textbf{0.02}& 99.91& \textbf{4.71}&\textbf{98.97} &\textbf{1.72}&\textbf{99.58}&73.83\\
\rowcolor[HTML]{EFEFEF} \model (Ours)&\textbf{0.07}& \textbf{99.95} &3.53& 99.06& 0.06 &99.94 &\textbf{0.02}& \textbf{99.95} &5.73 &98.65 &1.88& 99.51 &73.71 \\
   \hline
   
\end{tabular}}
\end{table}

\textcolor{black}{The result of selecting candidate ID data (and combine with labeled ID data) shows slightly better performance, which echoes our theory that the generalization bound of the OOD detector will be better if we have more ID training data (Theorem~\ref{the:main2}). }

  \section{Analysis  on Using Random Labels}
\label{sec:experiments_on_random_label_app}

\textcolor{black}{We present the OOD detection result of replacing the predicted labels with the random labels for the wild data as follows. The other experimental details are kept the same as \model. }

\begin{table}[!h]
  \centering
  \small
  \caption{\small \textcolor{black}{OOD detection results of using random labels for the wild data  on \textsc{Cifar-100} as ID. \model is trained on Wide ResNet-40-2 for 100 epochs  with $\pi=0.1$. {Bold} numbers are superior results.  }}
    \scalebox{0.68}{
    \begin{tabular}{cccccccccccccc}
    \toprule
    \multirow{3}[4]{*}{Methods} & \multicolumn{12}{c}{OOD Datasets}                                                   
    & \multirow{3}[4]{*}{ID ACC} \\
    \cmidrule{2-13}
          & \multicolumn{2}{c}{\textsc{Svhn}} & \multicolumn{2}{c}{\textsc{Places365}} & \multicolumn{2}{c}{\textsc{Lsun-C}} & \multicolumn{2}{c}{\textsc{Lsun-Resize}} & \multicolumn{2}{c}{\textsc{Textures}}& \multicolumn{2}{c}{Average} &  \\
\cmidrule{2-13}          & FPR95 & AUROC & FPR95 & AUROC & FPR95 & AUROC & FPR95 & AUROC & FPR95 & AUROC & FPR95 & AUROC &  \\
    \midrule

  w/ Random labels  &39.36 & 89.31 &  77.98 & 78.31 & 47.46 & 88.90 & 67.28& 80.23 & 54.86& 86.92& 57.39& 84.73&73.68 \\
\rowcolor[HTML]{EFEFEF} \model (Ours)&\textbf{0.07}& \textbf{99.95} &\textbf{3.53}& \textbf{99.06}& \textbf{0.06} &\textbf{99.94} &\textbf{0.02}& \textbf{99.95} &\textbf{5.73} &\textbf{98.65} &\textbf{1.88}& \textbf{99.51} &73.71 \\
   \hline
\end{tabular}}
\end{table}
\textcolor{black}{As we can observe, using the random labels leads to worse OOD detection performance because the gradient of the wild data can be wrong. In our theoretical analysis (Theorem~\ref{T1}), we have proved that using the predicted label can lead to a good separation of the wild ID and OOD data. However, the analysis using random labels might hold since it violates the assumption (Definitions~\ref{Def3} and \ref{Def4}) that the expected gradient of ID data should be different from that of wild data.}

  \section{Details of the Illustrative Experiments on the Impact of Predicted Labels}
\label{sec:experiments_on_predicted_label_app}
For calculating the filtering accuracy, \model is trained on Wide ResNet-40-2 for 100 epochs  with $\pi=0.1$ on two separate ID datasets. The other training details are kept the same as Section~\ref{sec:exp_steup} and Appendix~\ref{sec:detail_app}.

\section{Details of  Figure~\ref{fig:toy}}
\label{sec:details_of_toy_app}
For Figure~\ref{fig:toy} in the main paper, we generate the in-distribution data from three multivariate Gaussian distributions, forming three classes. The mean vectors are set to $[-2,0], [2,0]$ and $[0,2\sqrt{3}]$, respectively. The covariance matrix for all three classes is set to $\left[\begin{array}{ll}
0.25 & 0 \\
0 & 0.25
\end{array}\right]$. For each class, we generate $1,000$ samples.

For wild scenario 1, we generate the outlier data in the wild by sampling $100,000$ data points from a multivariate Gaussian $\mathcal{N}([0,\frac{2}{\sqrt{3}}], 7\cdot \mathbf{I})$ where $\mathbf{I} $ is $2\times2$ identity matrix, and only keep the $1,000$ data points that have the largest distance to the mean vector $[0, \frac{2}{\sqrt{3}}]$. For wild scenario 2,  we generate the outlier data in the wild by sampling $1,000$ data points from a multivariate Gaussian $\mathcal{N}([10,\frac{2}{\sqrt{3}}], 0.25\cdot \mathbf{I})$. For the in-distribution data in the wild,  we sample $3,000$ data points per class from the same three multivariate Gaussian distributions as mentioned before.

\section{Software and Hardware}
\label{sec:hardware}
We run all experiments with Python 3.8.5 and PyTorch 1.13.1, using NVIDIA GeForce RTX 2080Ti GPUs.

{\color{black}\section{Results with Varying Mixing Ratios}

We provide additional results of \model with varying $\pi$, i.e., 0.05, 0.2, 0.5, 0.9, and contrast with the baselines, which are shown below (\textsc{Cifar-100} as the in-distribution dataset). We found that the advantage of \model still holds. 
\begin{table}[!h]
  \centering
  \small
  \caption{\small \textcolor{black}{OOD detection results with multiple mixing ratios $\pi$ with \textsc{Cifar-100} as ID. \model is trained on Wide ResNet-40-2 for 100 epochs. {Bold} numbers are superior results.  }}
    \scalebox{0.78}{
    {\color{black}\begin{tabular}{ccccccccccccc}
    \toprule
    \multirow{3}[4]{*}{Methods} & \multicolumn{10}{c}{OOD Datasets}                                     
    & \multirow{3}[4]{*}{ID ACC} \\
    \cmidrule{2-11}
          & \multicolumn{2}{c}{\textsc{Svhn}} & \multicolumn{2}{c}{\textsc{Places365}} & \multicolumn{2}{c}{\textsc{Lsun-C}} & \multicolumn{2}{c}{\textsc{Lsun-Resize}} & \multicolumn{2}{c}{\textsc{Textures}} &  \\
\cmidrule{2-13}          & FPR95 & AUROC & FPR95 & AUROC & FPR95 & AUROC & FPR95 & AUROC & FPR95 & AUROC & &  \\
    \midrule
& \multicolumn{10}{c}{$\pi=0.05$} \\
   OE &2.78	&98.84	&63.63	&80.22	&6.73&	98.37	&2.06	&99.19	&32.86&	90.88	&71.98 \\
  Energy w/ OE& 2.02&	99.17&	56.18	&83.33	&4.32	&98.42	&3.96	&99.29	&40.41	&89.80&	73.45\\    
  WOODS& 0.26&	99.89&	32.71	&90.01	&\textbf{0.64}&	99.77&	\textbf{0.79}	&99.10	&12.26	&94.48&	74.15\\
\rowcolor[HTML]{EFEFEF} \model (Ours)&\textbf{0.17}&	\textbf{99.90}	&\textbf{6.21}&	\textbf{96.87}&	0.94	&\textbf{99.79}	&0.84	&\textbf{99.37}&	\textbf{5.77}&	\textbf{97.12}	&73.99 \\
& \multicolumn{10}{c}{$\pi=0.2$} \\
  OE &  2.59&	98.90&	55.68&	84.36&	4.91	&99.02	&1.97&	99.37&	25.62	&93.65&	73.72\\
  Energy w/ OE& 	1.79	& 99.25& 	47.28& 	86.78	& 4.18	& 99.00& 	3.15& 	99.35	& 36.80	& 91.48	& 73.91\\
  WOODS&	0.22	&99.82	&29.78&	91.28	&0.52	&99.79	&0.89&	99.56	&10.06	&95.23	&73.49\\
 \rowcolor[HTML]{EFEFEF} \model (Ours)	&\textbf{0.08}&	\textbf{99.92}&	\textbf{2.80}	&\textbf{99.31}	&\textbf{0.05}&	\textbf{99.94}&	\textbf{0.02}	&\textbf{99.97}&	\textbf{5.71}	&\textbf{98.71}&	73.86\\
 & \multicolumn{10}{c}{$\pi=0.5$} \\
 OE &  2.86&	99.05&	40.21&	88.75&	4.13	&99.05	&1.25&	99.38&	22.86&	94.63	&73.38\\
  Energy w/ OE& 2.71&	99.34	&34.82	&90.05	&3.27	&99.18	&2.54&	99.23	&30.16&	94.76	&72.76\\
  WOODS&	0.17	&99.80	&21.87	&93.73&	0.48	&99.61	&1.24	&99.54	&9.95&	95.97	&73.91\\
 \rowcolor[HTML]{EFEFEF} \model (Ours)	&\textbf{0.02}	&\textbf{99.98}&	\textbf{1.27}	&\textbf{99.62}&	\textbf{0.04}&	\textbf{99.96}	&\textbf{0.01}	&\textbf{99.99}	&\textbf{5.64}&	\textbf{99.16}	&73.77\\
  & \multicolumn{10}{c}{$\pi=0.9$} \\
 OE &  0.84	&99.36	&19.78	&96.29&	1.64	&99.57&	0.51&	99.75&	12.74	&94.95&	72.02\\
  Energy w/ OE&0.97	&99.64&	17.52&	96.53&	1.36&	99.73&	0.94&	99.59	&14.01&	95.73	&73.62\\
  WOODS&	0.05	&99.98	&11.34	&95.83&	0.07	&\textbf{99.99}&	0.03&	\textbf{99.99}&	6.72&	98.73&	73.86\\
 \rowcolor[HTML]{EFEFEF} \model (Ours)	&\textbf{0.03}&	\textbf{99.99}&	\textbf{2.79}	&\textbf{99.89}&	\textbf{0.05}&	\textbf{99.99}	&\textbf{0.01}	&\textbf{99.99}&	\textbf{5.88}&	\textbf{99.53}&	74.01\\	
   \hline
\end{tabular}}}
\end{table}
}

{\color{black}\section{Comparison with Weakly Supervised OOD Detection Baselines}
We have additionally compared with the two related works (TSL~\citep{he2023topological} and STEP~\citep{zhou2021step}). To ensure a fair comparison, we strictly follow the experimental setting in TSL, and rerun \model under the identical setup. The comparison on \textsc{Cifar-100} is shown as follows. 
\begin{table}[!h]
  \centering
  \small
  \caption{\small \textcolor{black}{Comparison with relevant baselines on \textsc{Cifar-100}. {Bold} numbers are superior results.  }}
    \scalebox{0.88}{
    {\color{black}\begin{tabular}{ccccc}
    \toprule
    \multirow{3}[4]{*}{Methods} & \multicolumn{4}{c}{OOD Datasets}                                                   \\
    \cmidrule{2-5}
          & \multicolumn{2}{c}{\textsc{Lsun-C}} & \multicolumn{2}{c}{\textsc{Lsun-Resize}}  \\
\cmidrule{2-5}     
& FPR95 & AUROC & FPR95 & AUROC  \\
    \midrule
  STEP & \textbf{0.00}	&99.99	&9.81&	97.87\\
  TSL & \textbf{0.00}	&\textbf{100.00}	&1.76&	99.57\\
\rowcolor[HTML]{EFEFEF} \model (Ours)&\textbf{0.00}	&99.99	&\textbf{0.58}&	\textbf{99.95} \\
   \hline
\end{tabular}}}
\end{table}

}
{\color{black}\section{Additional Results on Different Backbones }
We have additionally tried ResNet-18 and ResNet-34 as the network architectures—which are among the most used in OOD detection literature. The comparison with the baselines on \textsc{Cifar-100} is shown in the following tables, where \model outperforms all the baselines across different architectures. These additional results support the effectiveness of our approach.

\begin{table}[!h]
  \centering
  \small
  \caption{\small \textcolor{black}{OOD detection performance on \textsc{Cifar-100} as ID. All methods are trained on ResNet-18 for 100 epochs. For each dataset, we create corresponding wild mixture distribution
$\mathbb{P}_\text{wild} = (1 - \pi) \mathbb{P}_\text{in} + \pi \mathbb{P}_\text{out}$ for training and test on the corresponding OOD dataset. {Bold} numbers highlight the best results.}}
    \scalebox{0.67}{
  {\color{black}  \begin{tabular}{cccccccccccccc}
    \toprule
    \multirow{3}[4]{*}{Methods} & \multicolumn{12}{c}{OOD Datasets}                                                             & \multirow{3}[4]{*}{ID ACC} \\
    \cmidrule{2-13}
          & \multicolumn{2}{c}{\textsc{Svhn}} & \multicolumn{2}{c}{\textsc{Places365}} & \multicolumn{2}{c}{\textsc{Lsun-C}} & \multicolumn{2}{c}{\textsc{Lsun-Resize}} & \multicolumn{2}{c}{\textsc{Textures}} & \multicolumn{2}{c}{Average} &  \\
\cmidrule{2-13}          & FPR95 & AUROC & FPR95 & AUROC & FPR95 & AUROC & FPR95 & AUROC & FPR95 & AUROC & FPR95 & AUROC &  \\
  \hline
     \multicolumn{14}{c}{With $\mathbb{P}_{\text{in}}$ only} \\
    MSP   &81.32 & 77.74&  83.06 & 74.47&70.11&83.51 & 82.46& 75.73&  85.11&73.36 &80.41 & 76.96& 78.67\\
ODIN&  40.94 & 93.29& 87.71&71.46&28.72&94.51&79.61&82.13&83.63&72.37& 64.12&82.75 &78.67
 \\
Mahalanobis &22.44 &95.67& 92.66 &61.39& 68.90& 86.30& 23.07& 94.20 &62.39& 79.39&53.89&83.39&78.67
\\
Energy &  81.74&84.56&  82.23&76.68&34.78&93.93&73.57&82.99& 85.87&74.94&71.64& 82.62&78.67\\
KNN& 83.62&72.76& 82.09&80.03&65.96&	84.82 & 71.05	&81.24&76.88&	77.90& 75.92& 79.35& 78.67\\
ReAct &70.81&88.24&81.33&76.49&39.99&92.51&54.47&89.56&59.15&87.96&61.15&86.95&78.67 \\
 DICE&   54.65& 88.84& 79.58&  77.26& 0.93& 99.74& 49.40& 91.04&65.04& 76.42& 49.92&86.66 &78.67\\
CSI& 49.98 &89.57& 82.87& 75.64 &76.39& 80.38 &74.21 &83.34& 58.23& 81.04 &68.33 &81.99& 74.23\\
KNN+ & 43.21 &90.21& 84.62& 74.21& 50.12& 82.48& 76.92& 80.81& 63.21& 84.91& 63.61& 82.52& 77.03\\
\hline
 \multicolumn{14}{c}{With $\mathbb{P}_{\text{in}}$ and $\mathbb{P}_{\text{wild}}$ } \\
     OE&3.29  &97.93 & 62.90 & 80.23 & 7.07 & 95.93 & 4.06 & \textbf{97.98} & 33.27 & 90.03 & 22.12 &92.42 & 74.89\\
  Energy (w/ OE) &3.12& 94.27& 59.38& 82.19& 9.12& 91.23 &7.28 &95.39& 43.92& 90.11&24.56&90.64& 77.92\\
WOODS&3.92& 96.92 &33.92& 86.29& 5.19& 94.23 &\textbf{2.95}& 96.23 &11.95 &94.65 &11.59&93.66 & 77.54\\

\rowcolor[HTML]{EFEFEF} \model  & \textbf{2.29}& \textbf{97.96} &\textbf{6.29} &\textbf{96.66} &\textbf{3.92}& \textbf{97.81}& 4.87 &97.10& \textbf{8.28} &\textbf{95.95} &\textbf{5.13}&\textbf{97.10}&77.71\\
   \hline
\end{tabular}}}
\end{table}

\begin{table}[!h]
  \centering
  \small
  \caption{\small \textcolor{black}{OOD detection performance on \textsc{Cifar-100} as ID. All methods are trained on ResNet-34 for 100 epochs. For each dataset, we create corresponding wild mixture distribution
$\mathbb{P}_\text{wild} = (1 - \pi) \mathbb{P}_\text{in} + \pi \mathbb{P}_\text{out}$ for training and test on the corresponding OOD dataset. {Bold} numbers highlight the best results.}}
    \scalebox{0.67}{
  {\color{black}  \begin{tabular}{cccccccccccccc}
    \toprule
    \multirow{3}[4]{*}{Methods} & \multicolumn{12}{c}{OOD Datasets}                                                             & \multirow{3}[4]{*}{ID ACC} \\
    \cmidrule{2-13}
          & \multicolumn{2}{c}{\textsc{Svhn}} & \multicolumn{2}{c}{\textsc{Places365}} & \multicolumn{2}{c}{\textsc{Lsun-C}} & \multicolumn{2}{c}{\textsc{Lsun-Resize}} & \multicolumn{2}{c}{\textsc{Textures}} & \multicolumn{2}{c}{Average} &  \\
\cmidrule{2-13}          & FPR95 & AUROC & FPR95 & AUROC & FPR95 & AUROC & FPR95 & AUROC & FPR95 & AUROC & FPR95 & AUROC &  \\
  \hline
     \multicolumn{14}{c}{With $\mathbb{P}_{\text{in}}$ only} \\
    MSP   &78.89 & 79.80&  84.38&  74.21&  83.47&  75.28&  84.61 & 74.51&  86.51 & 72.53 & 83.12&  75.27&79.04\\
ODIN& 70.16& 84.88 &82.16& 75.19 &76.36 &80.10& 79.54& 79.16 &85.28 &75.23& 78.70& 79.11&79.04
 \\
Mahalanobis & 87.09& 80.62 &84.63 &73.89 &84.15& 79.43 &83.18 &78.83 &61.72 &84.87& 80.15 &79.53&79.04
\\
Energy & 66.91 &85.25 &81.41 &76.37 &59.77 &86.69 &66.52 &84.49 &79.01& 79.96& 70.72& 82.55& 79.04\\
KNN& 81.12 &73.65  & 79.62& 78.21 & 63.29& 85.56 & 73.92 & 79.77& 73.29 & 80.35 & 74.25 & 79.51 & 79.04\\
ReAct & 82.85 & 70.12 &81.75 & 76.25& 80.70 & 83.03 &67.40 & 83.28 &74.60& 81.61 & 77.46 & 78.86 &79.04 \\
 DICE& 83.55 & 72.49  & 85.05 & 75.92 & 94.05 & 73.59 &75.20 & 80.90 &79.80 & 77.83 &83.53 & 76.15 & 79.04  \\
CSI& 44.53 &92.65& 79.08& 76.27 &75.58& 83.78 &76.62 &84.98& 61.61& 86.47 &67.48 &84.83&77.89\\
KNN+ & 39.23 &92.78& 80.74& 77.58& 48.99& 89.30& 74.99& 82.69& 57.15& 88.35& 60.22& 86.14&78.32\\
\hline
 \multicolumn{14}{c}{With $\mathbb{P}_{\text{in}}$ and $\mathbb{P}_{\text{wild}}$ } \\
     OE& 2.11  &98.23 & 60.12 & 83.22 & 6.08 & 96.34 & 3.94 & 98.13 & 30.00 & 92.27 & 20.45 & 93.64&  75.72\\
  Energy (w/ OE) &1.94& 95.03& 68.84& 85.94& 7.66& 92.04 &6.86 &97.63& 40.82& 93.07& 25.22&92.74 & 78.75\\
WOODS& 2.08& 97.33 &25.37& 88.93& 4.26& 97.74 &1.05& 97.30 &8.85 &96.86 & 8.32&  95.63& 78.97\\

\rowcolor[HTML]{EFEFEF} \model  &\textbf{0.98} &\textbf{99.94 }&\textbf{2.98} &\textbf{99.08}& \textbf{0.07} &\textbf{99.94}& \textbf{0.03}& \textbf{99.96}& \textbf{4.01} &\textbf{98.83}& \textbf{ 1.61}& \textbf{99.55}& 78.01\\
   \hline
\end{tabular}}}
\end{table}

}

  \section{Broader Impact}
  \label{sec:broader}
 Our project aims to improve the reliability and safety of modern machine learning models.  From the theoretical perspective, our analysis can facilitate and deepen the understanding of   the effect of unlabeled wild data for OOD detection. In Appendix~\ref{sec:verification_discrepancy}, we properly verify the necessary conditions and the value of our error bound using real-world datasets. Hence, we believe our theoretical framework has a broad utility and significance.

 From the practical side,  our study can lead to direct benefits and societal impacts, particularly when the wild data is abundant in the models' operating environment, such as in safety-critical applications
i.e., autonomous driving and healthcare data analysis.  Our study does not involve any human subjects or violation of legal compliance. We do not anticipate any potentially harmful consequences to our work. Through our study and releasing our code, we hope to raise stronger research and societal awareness towards the problem of exploring unlabeled wild data for out-of-distribution detection in real-world settings.

\end{document}